\documentclass[twoside]{article}

\usepackage[accepted]{aistats2022_arxiv}

\usepackage{amsmath,amsthm,amsfonts}
\usepackage{amsmath,bm,xcolor,url}
\usepackage{xcolor}
\usepackage{amssymb,latexsym,bm}
\usepackage{array}%array and tabular environments
\usepackage{mathrsfs}
\catcode`~=11 \def\UrlSpecials{\do\~{\kern -.15em\lower .7ex\hbox{~}\kern .04em}} \catcode`~=13 

\allowdisplaybreaks[1]

\newcommand{\calA}{\mathcal{A}}

\newcommand{\calE}{\mathcal{E}}
\newcommand{\calF}{\mathcal{F}}

\newcommand{\calL}{\mathcal{L}}

\newcommand{\calN}{\mathcal{N}}

\newcommand{\calS}{\mathcal{S}}

\newcommand{\calX}{\mathcal{X}}

\newcommand{\bbN}{\mathbb{N}}

\newcommand{\bbP}{\mathbb{P}}

\newcommand{\bbR}{\mathbb{R}}

\DeclareMathAlphabet{\mathbsf}{OT1}{cmss}{bx}{n}
\DeclareMathAlphabet{\mathssf}{OT1}{cmss}{m}{sl}% slanted sans serif

\DeclareSymbolFont{bsfletters}{OT1}{cmss}{bx}{n}  
\DeclareSymbolFont{ssfletters}{OT1}{cmss}{m}{n}
\DeclareMathSymbol{\bsfGamma}{0}{bsfletters}{'000}
\DeclareMathSymbol{\ssfGamma}{0}{ssfletters}{'000}
\DeclareMathSymbol{\bsfDelta}{0}{bsfletters}{'001}
\DeclareMathSymbol{\ssfDelta}{0}{ssfletters}{'001}
\DeclareMathSymbol{\bsfTheta}{0}{bsfletters}{'002}
\DeclareMathSymbol{\ssfTheta}{0}{ssfletters}{'002}
\DeclareMathSymbol{\bsfLambda}{0}{bsfletters}{'003}
\DeclareMathSymbol{\ssfLambda}{0}{ssfletters}{'003}
\DeclareMathSymbol{\bsfXi}{0}{bsfletters}{'004}
\DeclareMathSymbol{\ssfXi}{0}{ssfletters}{'004}
\DeclareMathSymbol{\bsfPi}{0}{bsfletters}{'005}
\DeclareMathSymbol{\ssfPi}{0}{ssfletters}{'005}
\DeclareMathSymbol{\bsfSigma}{0}{bsfletters}{'006}
\DeclareMathSymbol{\ssfSigma}{0}{ssfletters}{'006}
\DeclareMathSymbol{\bsfUpsilon}{0}{bsfletters}{'007}
\DeclareMathSymbol{\ssfUpsilon}{0}{ssfletters}{'007}
\DeclareMathSymbol{\bsfPhi}{0}{bsfletters}{'010}
\DeclareMathSymbol{\ssfPhi}{0}{ssfletters}{'010}
\DeclareMathSymbol{\bsfPsi}{0}{bsfletters}{'011}
\DeclareMathSymbol{\ssfPsi}{0}{ssfletters}{'011}
\DeclareMathSymbol{\bsfOmega}{0}{bsfletters}{'012}
\DeclareMathSymbol{\ssfOmega}{0}{ssfletters}{'012}

\DeclareMathOperator*{\argmax}{arg\,max}
\DeclareMathOperator*{\argmin}{arg\,min}

\newtheorem{theorem}{Theorem} 
\newtheorem{lemma}{Lemma}

\newtheorem{proposition}{Proposition}
\newtheorem{corollary}{Corollary}
\newtheorem{definition}{Definition}

\newtheorem{assumption}{Assumption}
\newtheorem{fact}{Fact}

\theoremstyle{definition}

\newcommand{\qednew}{\nobreak \ifvmode \relax \else
      \ifdim\lastskip<1.5em \hskip-\lastskip
      \hskip1.5em plus0em minus0.5em \fi \nobreak
      \vrule height0.75em width0.5em depth0.25em\fi}

\newcommand{\Exp}{\mathbb{E}}

\newcommand{\muhat}{\widehat{\mu}}
\newcommand{\R}{\mathbb{R}}

\newcommand{\safe}{\calS}
\newcommand{\ihat}{\widehat{i}}
\newcommand{\ist}{i^*}

\newcommand{\thetahat}{\widehat{\theta}}

\newcommand{\asafelow}{\underline{a}_s}
\newcommand{\asafehat}{\widehat{a}_s}

\newcommand{\aunbar}{\bar{a}_u}
\newcommand{\aunhat}{\widehat{a}_u}
\newcommand{\alowhat}{\underline{a}}

\newcommand{\ghat}{\widehat{g}}
\newcommand{\fhat}{\widehat{f}}
\newcommand{\Efun}{\calE_{\mathrm{fun}}}
\newcommand{\epssafe}{\epsilon_{\mathrm{safe}}}
\newcommand{\unsafe}{\mathrm{unsafe}}

\newcommand{\ellbar}{\bar{\ell}}

\newcommand{\mbar}{\bar{m}}

\newcommand{\nbar}{\bar{n}}
\newcommand{\mhat}{\widehat{m}}
\newcommand{\nhat}{\widehat{n}}

\newcommand{\tbar}{\bar{t}}
\newcommand{\ellhatun}{\widehat{\ell}_{\unsafe}}
\newcommand{\kl}{\mathrm{KL}}

\newcommand{\gtilinv}{\widetilde{g}^{-1}}
\newcommand{\gtil}{\widetilde{g}}

\newcommand{\ellsafe}{\Phi}
\newcommand{\ellsolvea}{\Psi^1}
\newcommand{\ellsolveb}{\Psi^2}
\newcommand{\ellsolvec}{\Psi^3}

\newcommand{\cOtil}{\widetilde{\mathcal{O}}}
\newcommand{\case}{\mathsf{c}}

\newcommand{\alglinear}{\textsc{SafeBAI-Linear}\xspace}
\newcommand{\algmono}{\textsc{SafeBAI-Monotonic}\xspace}

\usepackage{url}            % simple URL typesetting
\usepackage{booktabs}       % professional-quality tables
\usepackage{amsfonts}       % blackboard math symbols
\usepackage{dsfont}
\usepackage{nicefrac}       % compact symbols for 1/2, etc.
\usepackage{amsthm}
\usepackage[hypertexnames=false]{hyperref}

\usepackage{natbib}
\usepackage{cleveref}
\usepackage[noend]{algpseudocode}
\usepackage{algorithm}
\usepackage{xspace}
\usepackage{graphicx}
\usepackage{wrapfig}
\usepackage{enumitem}
\usepackage[skip=3pt]{caption}

\usepackage{etoolbox}
\makeatletter
\patchcmd{\ALG@step}{\addtocounter{ALG@line}{1}}{\refstepcounter{ALG@line}}{}{}
\newcommand{\ALG@lineautorefname}{Line}
\makeatother

\allowdisplaybreaks

\begin{document}

\twocolumn[

\aistatstitle{Best Arm Identification with Safety Constraints}

\aistatsauthor{ Zhenlin Wang \And Andrew Wagenmaker \And  Kevin Jamieson }

\aistatsaddress{ National University of Singapore \\ \texttt{wang\_zhenlin@u.nus.edu} \And  University of Washington \\  \texttt{ajwagen@cs.washington.edu}  \And University of Washington \\ \texttt{jamieson@cs.washington.edu} } ]

\begin{abstract}
%!TEX root = ../BAI_Constraint.tex

The best arm identification problem in the multi-armed bandit setting is an excellent model of many real-world decision-making problems, yet it fails to capture the fact that in the real-world, safety constraints often must be met while learning. In this work we study the question of best-arm identification in safety-critical settings, where the goal of the agent is to find the best \emph{safe} option out of many, while exploring in a way that guarantees certain, initially unknown safety constraints are met. We first analyze this problem in the setting where the reward and safety constraint takes a linear structure, and show nearly matching upper and lower bounds. We then analyze a much more general version of the problem where we only assume the reward and safety constraint can be modeled by monotonic functions, and propose an algorithm in this setting which is guaranteed to learn safely. We conclude with experimental results demonstrating the effectiveness of our approaches in scenarios such as safely identifying the best drug out of many in order to treat an illness.

\end{abstract}

%!TEX root = ../BAI_Constraint.tex

\section{INTRODUCTION}
Consider a dosing trial where a scientist is trying to determine, out of $d$ different drugs, which is most effective at treating a particular illness. For each drug, the scientist can run an experiment where they administer a particular dose of a drug to a patient and observe the effectiveness. After repeating this process on multiple patients for each drug, the scientist must then give a recommendation as to which drug is most effective. Critically, in the experimental process, the safety of the participants must be ensured. 
While we may know a low, baseline safe dosage for each drug, the drug may only become effective at higher doses, some of which may be above a certain dose threshold in which the treatment begins to become unsafe, and could cause negative side effects.
The goal of the experimenter is not only to determine which drug is most effective, but to also guarantee that only safe dosage levels are applied to the patients in the experiments. 

We can cast this problem as a \emph{best-arm identification problem with safety constraints}. In the classical multi-armed bandit (MAB) setting, at each round, the agent may choose from one of $d$ options and obtains a noisy realization of the option's \emph{reward}. The goal of best-arm identification (BAI) is to determine which of the $d$ options has the largest mean reward using the minimum number of samples. This problem has been extensively studied and is an effective model of many real-world problems. However, the standard MAB setting is unable to incorporate either different ``doses'' or safety constraints.

In this work we set out to address this question and answer how an agent ought to sample in a multi-armed bandit model where the goal is to \emph{safely} identify the best option. At each timestep, the learner must choose both the ``dosing level'' for the arm pulled, and ensure that this dosing level is ``safe''. Critically, we assume that the range of safe dosing levels is unknown, and this must be learned as well. 

We first study this problem in the setting where the relationship between the reward obtained---the effectiveness---and the dosing level, as well as the safety of a dosing level, are \emph{linear}. We propose a near-optimal action elimination-style algorithm, \alglinear, which successively refines its estimates of the safe dosing levels to ensure safety throughout execution, and gradually increases the dose applied to each option until it is able to determine the best option. We next consider the much more general setting where we only assume that the effectiveness and the safety are monotonic functions of the dosing level---a more realistic assumption for modeling a drug response than the linear assumption. We propose an algorithm in this setting, \algmono, which we show explores efficiently---guaranteeing only safe dose levels are applied while determining the optimal drug.

\section{RELATED WORK}

\paragraph{Best-Arm Identification in Multi-Armed Bandits.}
The BAI problem in multi-armed bandits has been extensively studied for decades, beginning with the works of \cite{bechhofer1958sequential} and \cite{paulson1964sequential}. A variety of different approaches have been proposed, including action-elimination style algorithms \citep{even2002pac,karnin2013almost} and upper confidence bound-style algorithms \citep{bubeck2009pure,kalyanakrishnan2012pac,jamieson2014lil}. Furthermore, lower bounds have been proposed \citep{mannor2004sample,kaufmann2016complexity} which have been shown to be near-tight. The linear response setting we consider is somewhat related to the BAI problem in the linear bandit setting \citep{soare2014best} and the monotonic function setting we study is related to the continuous-armed bandit setting, also known as zeroth-order or derivative-free optimization \citep{flaxman2005online,jamieson2012query,valko2013stochastic,rios2013derivative}.

Perhaps the most comparable work to ours is \cite{sui2015safe, sui2018stagewise}, which considers the problem of finding the value $x \in D$ such that $f(x)$ is maximized, for some $f$ and $D$, where at each time step they observe $f(x_t) + \eta_t$, and must guarantee that $g(x_t) \ge h$ during exploration. In essence, this is a constrained BAI problem, but these works do not provide a tight upper bound (indeed, their algorithms are unverifiable in style---they provide no stopping condition guaranteeing the optimal value has been found), or an information-theoretic lower bound. Furthermore, their algorithms use a potentially wasteful exploration strategy which could overexplore suboptimal arms. Our work improves on all these facets---our algorithm is verifiable, providing a stopping criteria with an optimality guarantee, utilizes a much more efficient exploration scheme, and we prove an upper bound with nearly matching lower bound.
 To our knowledge, this is the only existing work which studies a problem akin to BAI where safety constraints are present.

\paragraph{Regret Minimization in Bandits with Safety Constraints.}
A related line of work is that of regret minimization in bandits with safety constraints. The majority of work in this setting considers the more general linear bandits problem where now the agent at each timestep must choose a vector $x_t \in \calX \subseteq \R^d$ and observes $y_t = \theta^\top x_t + \eta_t$. A variety of different formulations of the safety constraint have been proposed, but in general they require that $\mu^\top x_t \le \gamma$ either almost surely or in expectation. Various algorithmic approaches have been applied, such as Thompson Sampling \citep{moradipari2020stage,moradipari2021safe}, as well as optimistic UCB-style approaches \citep{kazerouni2016conservative,amani2019linear,pacchiano2021stochastic}. The key difference between our problem and the existing work in this setting is that we are interested in best-arm identification, while existing works tend to focus on regret minimization, where the goal is to achieve large online reward, not simply to identify the best arm. In addition, we seek to obtain instance-dependent (gap-dependent) results while existing works only target minimax (worst-case) bounds.

\paragraph{Dose-Finding and Thresholding Bandits.}
Dose-finding is a long-standing problem in the biomedical sciences where the goal is to determine the optimal dose of a drug to give a patient \citep{thall1998strategy,thall2004dose,rogatko2005new,musuamba2017advanced,riviere2018phase}. Much work has been done on developing statistically justified procedures to address this. Recently, the bandits community has begun to approach this problem from the perspective of structured multi-armed bandits---a setting which has become known as the \emph{thresholding bandit} problem. \cite{chen2014combinatorial} consider a general version of the thresholding bandit problem, while several follow-up works \citep{locatelli2016optimal,garivier2017thresholding,cheshire2020influence,aziz2021multi} consider the particular application of this setting to the problem of identifying safe dosing levels in Phase 1 clinical trials. In contrast to our setting, this setting does not consider safety constraints---at any time, any dose level may be tested with no penalty---and, in addition, does not consider the identification of the best drug out of many, rather it can be seen as simply determining the largest safe dose level for a particular drug. In a sense, our setup then extends this dose-finding problem to its multi-dimensional analogue.

%!TEX root = ../BAI_Constraint.tex

\section{PRELIMINARIES}

\paragraph{Notation.}
Throughout, $[n] = \{ 1,2,3,\ldots,n\}$. We will use $\cOtil( \cdot )$ to hide logarithmic terms and absolute constants. We let $\bbR^d_+ = \{ x \in \bbR^d  :  x_i \ge 0, \forall i \in [d] \}$.

\paragraph{Problem Setting: Linear.}
In the linear response case, we are interested in the setting where at every time step $t$, the learner chooses a coordinate $i_t \in [d]$ and a value $a_t \in [a_{0,i_t}, M_{i_t}]$, and observes
\begin{align*}
y_t = a_t \theta_{i_t} + \eta_t, \quad z_t = a_t \mu_{i_t} + w_t
\end{align*}
for $\eta_t$ and $w_t$ both $\sigma^2$-sub-Gaussian.
We assume that $\theta \in \R_+^d$ and $\mu \in \R_+^d$ are initially unknown and, furthermore, that the learner must always choose, with high probability, $a_t$ and $i_t$ satisfying $a_t \mu_{i_t} \le \gamma$, for some known $\gamma$. To make this tractable, we assume that the learner is given $a_{0,i} > 0$ and $M_i > 0$, $a_{0,i} \le M_i$, and that $a_{0,i} \mu_i \le \gamma$ for all $i$. 

The goal of the learner is to identify the optimal \emph{coordinate}; that is, to find $\ist \in [d]$ defined as:
\begin{align*}
\ist := \argmax_i \max_{a \in [a_{0,i},M_i]} a \theta_i \quad \text{s.t.} \quad a \mu_i \le \gamma .
\end{align*}
In other words, we want to find the coordinate that has the largest safe value. Here we will only be concerned with identifying this coordinate, not determining what that largest safe value is (once the optimal coordinate is identified, it can be repeatedly played until the largest safe value is identified). 
For simplicity we assume that there is a unique optimal coordinate.

We will let $\safe_i$ denote the set of all safe values for coordinate $i$: $\safe_i = \{ a \in [a_{0,i},M_i] \ : \ a_{0,i} \mu_i \le \gamma \}$.  
We will define the gap for coordinate $i$ as:
\begin{align*}
\Delta_i := \min \left \{ \frac{\gamma \theta_{\ist}}{\mu_{\ist}}, \theta_{\ist} M_{\ist} \right \} - \min \left \{ \frac{\gamma \theta_i}{\mu_i}, \theta_i M_i \right \}.
\end{align*} 
Note that $\gamma \theta_i / \mu_i = \max_{a \in \safe_i} a \theta_i$ is the largest safe effect for coordinate $i$, in the case when $M_i \not \in \safe_i$, and $\theta_i M_i$ is the largest safe effect when $M_i \in \safe_i$. Our definition of the gap then corresponds to the difference in maximum safe value for the optimal coordinate and the maximum safe value for coordinate $i$.

\paragraph{Problem Setting: Monotonic.}
While the linear setting is useful for precisely quantifying the complexity of learning, often in real-world settings the actual structure is nonlinear. To address this, we introduce the following more general version of the problem, where now we assume our observations take the form
\begin{align*}
y_t = f_{i_t}(a_t) + \eta_t, \quad z_t = g_{i_t}(a_t) + w_t
\end{align*}
for some functions $\{ (f_i,g_i) \}_{i \in [d]}$. We assume that $f_i$ and $g_i$ are defined over all $\bbR$ and, to simplify the analysis, allow the learner to play any values $a \in \bbR$ (provided they are safe). To guarantee safety while learning, we must make additional assumptions on the structure of $f_i$ and $g_i$. Henceforth, we will assume that $f_i$ and $g_i$ satisfy the following.

\begin{assumption}\label{asm:smooth_fun}
For all $i\in [d]$, $g_i(\cdot)$ is 1-Lipschitz and strictly monotonically increasing. Furthermore, for all $i \in [d]$ and $a \in \bbR$, $f_i(a) \in [0,1]$ and $f_i(\cdot)$ is nondecreasing. 
\end{assumption}

Note that this assumption implies that $g_i$ is invertible. We assume that $f_i$ and $g_i$ are initially unknown to the learner, but that the learner is told that they are 1-Lipschitz. In this setting, our safety constraint is:
\begin{align*}
g_{i_t}(a_t) \le \gamma
\end{align*}
and we assume that for each $i$, the learner is given an initial value $a_{0,i}$ such that $g_i(a_{0,i}) \le \gamma$. To guarantee that we find the best arm, we need a slightly stronger query model which allows arms to be queried that are marginally above this threshold. To this end, we introduce a value $\epssafe$, and allow our learner to query points $a_t$ which satisfy:
\begin{align*}
g_{i_t}(a_t) \le \gamma + \epssafe.
\end{align*}
However, we are still interested in finding the best value satisfying $g_i(a_t) \le \gamma$, and define the best coordinate as (note that $f_i(g_i^{-1}(\gamma))$ is the maximum achievable safe value for coordinate $i$):
\begin{align*}
\ist := \argmax_i f_i(g_i^{-1}(\gamma))
\end{align*}
Similar to the linear case, our goal is only to identify the best coordinate, not the value of $f_{\ist}(g_{\ist}^{-1}(\gamma))$. We can define a notion of the \emph{gap} as the difference between the maximum safe reward achievable by the best arm and the maximum safe reward achievable by arm $i$: 
\begin{align*}
\Delta_i := f_{\ist}(g_{\ist}^{-1}(\gamma)) - f_i(g_i^{-1}(\gamma)) .
\end{align*}

\paragraph{Algorithm Classes.}
Formally, we will define ``safe'' algorithms in the following way.

\begin{definition}[$\delta$-Safe Algorithm]
We say that an algorithm $\calA$ is $\delta$-safe if, with probability $1-\delta$, $\calA$ only pulls arms $a_t$ satisfying $g_{i_t}(a_t) \le \gamma$.
\end{definition}

In the linear case, the above condition $g_{i_t}(a_t) \le \gamma$ is equivalent to $a_t \mu_{i_t} \le \gamma$. 

\begin{definition}[$\delta$-PAC Algorithm]
We say that an algorithm $\calA$ is $\delta$-PAC if it outputs an arm $\ihat$ such that $\bbP[\ihat = \ist] \ge 1-\delta$.
\end{definition}

\begin{definition}[$\delta$-PAC Safe Algorithm]
We say that an algorithm $\calA$ is a $\delta$-PAC Safe algorithm if $\calA$ is both $\delta$-Safe and $\delta$-PAC.
\end{definition}

Our goal will be to obtain an algorithm that is $\delta$-PAC Safe with the minimum possible sample complexity.

\section{LOWER BOUND}
We first present a lower bound on the complexity of safe BAI. For simplicity, we assume here that $\sigma^2 = 1$ and that $\eta_t \sim \calN(0,\sigma^2), w_t \sim \calN(0,\sigma^2)$.
We will consider a slightly different class of algorithms to prove our lower bound. Rather than learning which arms are safe, we give the algorithm a set of arms it can pull throughout execution.

\begin{definition}[$(\delta,\{ \calX_i \}_{i \in [d]})$-PAC Algorithm]
We say an algorithm is $(\delta,\{ \calX_i \}_{i \in [d]})$-PAC if it is $\delta$-PAC and with probability 1 only chooses values $a_t \in \calX_{i_t}$.
\end{definition}

A $(\delta,\{ \calX_i \}_{i \in [d]})$-PAC algorithm is then given a set of values, $\calX_i$, by an oracle that it can pull for each arm, and is only allowed to pull these arms throughout execution (critically, it is not told if these values are safe---it is only told that it is allowed to pull them). Note that, if $\calX_i = \safe_i$, a $(\delta,\{ \calX_i \}_{i \in [d]})$-PAC learner is a strictly more powerful learner than a $\delta$-PAC Safe learner---it is able to query \emph{any} safe arm at \emph{any} time, while a $\delta$-PAC Safe learner may only query arms it has verified are safe. Thus, the task of learning the optimal safe coordinate for a $(\delta,\{ \safe_i \}_{i \in [d]})$-PAC learner is easier than for a $\delta$-PAC Safe learner.

\begin{theorem}\label{thm:linear_lb}
Fix an instance $\theta \in \bbR_+^d$ and $\mu \in \bbR_+^d$, $\gamma > 0$, and $M_i = \infty$ for all $i$, and let $\safe_i(\mu,\gamma)$ denote the safe values for coordinate $i$ on this instance. Let $\tau$ denote the stopping time for any $(\delta,\{ \safe_i(\mu,\gamma) \}_{i \in [d]})$-PAC algorithm. Then on this instance we will have that
\begin{align*}
\Exp_{\theta,\mu}[\tau] \ge \frac{2}{3} \log \frac{1}{2.4 \delta} \cdot \sum_{i \neq \ist} \frac{ 1 + \theta_{\ist}^2/\mu_{\ist}^2 + \theta_i^2/\mu_i^2}{\Delta_i^2} .
\end{align*}
\end{theorem}

\Cref{thm:linear_lb} states that the lower bound in the safe BAI  setting scales as the familiar ``sum over inverse gaps squared'' lower bound of the standard multi-armed bandit setting. The key difference here is the $\theta_{\ist}^2/\mu_{\ist}^2 + \theta_i^2/\mu_i^2 $ term. For each $i$, we can break up the cost associated with coordinate $i$ into terms $\frac{1}{\Delta_i^2}$ and $ \frac{\theta_{\ist}^2/\mu_{\ist}^2 + \theta_i^2/\mu_i^2 }{\Delta_i^2}$. The first term is due to showing that coordinate $i$ is suboptimal and is present in the standard multi-armed bandit lower bound. The second term, in contrast, is not present in the standard multi-armed bandit lower bound and arises in this setting as the cost of learning where the safety threshold is for coordinate $i$.

\section{LINEAR RESPONSE}

\algrenewcommand\algorithmicindent{1.3em}
\begin{algorithm}[h]
\begin{algorithmic}[1]
	\State \textbf{input:} Confidence $\delta$, noise variance $\sigma^2$, safety tolerance $\gamma$, value bounds $\{ (a_{0,i},M_i) \}_{i \in [d]}$
	\State \textbf{initialize} $\asafehat^0(i) \leftarrow a_{0,i} , \aunhat^0(i) \leftarrow M_i, i \in [d]$, $\calX_0 \leftarrow [d]$, $\ell \leftarrow 1$
	\While{$| \calX_{\ell-1} | > 1$}
		\State $\epsilon_\ell \leftarrow 2^{-\ell}$
		\For{$i \in \calX_{\ell-1}$}
			\State $N_\ell \leftarrow \lceil 2 \sigma^2 \log \frac{8d \ell^2}{\delta} \cdot \epsilon_\ell^{-2} \rceil$
			\State Pull $\asafehat^{\ell-1}(i)$ $N_\ell$ times, observe: \\
		\hspace{3.5em} $y_t = \asafehat^{\ell-1}(i) \theta_i + \eta_t, z_t = \asafehat^{\ell-1}(i) \mu_i + w_t$	
			\State $\thetahat_{i,\ell} \leftarrow \frac{1}{\asafehat^{\ell-1}(i) N_\ell} \sum_{t=1}^{N_\ell} y_t$, 
			\State $\muhat_{i,\ell} \leftarrow \frac{1}{\asafehat^{\ell-1}(i) N_\ell} \sum_{t=1}^{N_\ell} z_t$ 
			\If{$\asafehat^{\ell-1}(i) < M_i$}
				\State $\asafehat^\ell(i) \leftarrow \min \{ \max \{ \frac{\gamma}{\muhat_{i, \ell} + \epsilon_\ell/\asafehat^{\ell-1}(i)}, a_{0,i} \}, M_i \}$
				\State $\aunhat^\ell(i) \leftarrow \min \{  \frac{\gamma}{\muhat_{i, \ell} - \epsilon_\ell/\asafehat^{\ell-1}(i)}, M_i \}$
			\Else
				\State $\asafehat^\ell(i) \leftarrow M_i, \aunhat^\ell(i) \leftarrow M_i$
			\EndIf
		\EndFor
		\State $\calX_\ell \leftarrow  \{ i \in \calX_{\ell - 1} \ : \ \aunhat^\ell(i) (\thetahat_{i, \ell} + \tfrac{\epsilon_\ell}{\asafehat^{\ell-1}(i)}) \ge $ 
		\Statex \hspace{5em} $\max_j \asafehat^\ell(j) (\thetahat_{j, \ell} - \tfrac{\epsilon_\ell}{\asafehat^{\ell-1}(j)}) \}$
		\State $\ell \leftarrow \ell + 1$
	\EndWhile
	\State \textbf{return} $\calX_{\ell-1}$
\end{algorithmic}
\caption{Safe Best-Arm Identification for Linear Functions (\alglinear)}
\label{alg:constrained_bai}
\end{algorithm}

Given this lower bound, we next propose an algorithm, \alglinear, in the linear response case. \alglinear proceeds in epochs. At every epoch it maintains an estimate of the largest value it can guarantee is safe, $\asafehat^\ell(i)$, and the smallest value it can guarantee is unsafe, $\aunhat^\ell(i)$, for each active coordinate. It then uses these estimates to construct a lower bound on the maximum safe value of coordinate $i$, $\asafehat^\ell(i) (\thetahat_{i, \ell} - \tfrac{\epsilon_\ell}{\asafehat^{\ell-1}(i)})$, and an upper bound, $\aunhat^\ell(i) (\thetahat_{i, \ell} + \tfrac{\epsilon_\ell}{\asafehat^{\ell-1}(i)})$, and eliminates coordinates that are provably suboptimal. By only pulling the provably safe values, $\asafehat^\ell(i)$, and halving the tolerance at every epoch, \alglinear is able to safely refine its estimates of the problem parameters. Furthermore, by playing the largest verifiably safe value, it is able to effectively reduce the signal-to-noise ratio, guaranteeing efficient, safe convergence to the optimal coordinate.

\subsection{Sample Complexity}
Towards presenting the sample complexity of \Cref{alg:constrained_bai}, we define the function:
$$\xi_a(x) := 2^{a \sqrt{\log_2 \max \{ x, 2 \}}}.$$
For $a > 0$, $\xi_a(x)$ grows sub-polynomially yet super-logarithmically in $x$. The following result gives a quantification of its growth. 
\begin{proposition}\label{prop:exp_sqrt_fun}
For $a > 0$, we can bound $\xi_a(x) \le \min_z \max \{ x, 2\}^z + 2^{a^2/z}$ and $\xi_a(x) \ge 1$. 
\end{proposition}

We are now ready to give our sample complexity result.

\begin{theorem}\label{thm:linear_complexity2_simp}
Recall that $\safe_i$ denotes the safe values for coordinate $i$, and $M_i$ denotes the maximum playable value for coordinate $i$. Let $C_\gamma := \xi_{\sqrt{32}}(2\gamma)$ and define the following:
\begin{enumerate}[leftmargin=*]
\item \textbf{Case 1 ($M_i \not\in \safe_i$ and $M_{\ist} \not\in \safe_{\ist}$)}: 
\begin{align*}
N_{1,i}&  :=   \frac{1 + \theta_{\ist}^2/\mu_{\ist}^2 + \theta_i^2/\mu_i^2}{\Delta_i^2}     \\
& \qquad + C_\gamma \max \Big \{ \xi_4(\tfrac{1}{a_{0,i} \mu_i}) , \xi_4(\tfrac{1}{a_{0,\ist} \mu_{\ist}})  \Big \}  .
\end{align*}
\item \textbf{Case 2 ($M_i \in \safe_i$ and $M_{\ist} \not\in \safe_{\ist}$)}:
\begin{align*}
N_{2,i} & :=  \frac{1 + M_i^2 \theta_i^2 /\gamma^2}{\Delta_i^2}  + \frac{ M_i^2 \mu_i^2/\gamma^2}{ (\gamma - M_i \mu_i)^2}  \\
& + C_\gamma \max \Big \{ \xi_4(\tfrac{1}{a_{0,i} \mu_i})  \xi_4 ( \tfrac{M_i \mu_i}{\gamma - M_i \mu_i} ), \xi_4(\tfrac{1}{a_{0,\ist} \mu_{\ist}}) \Big \} . 
\end{align*}
\item \textbf{Case 3  ($M_i \not\in \safe_i$ and $M_{\ist} \in \safe_{\ist}$)}:
\begin{align*}
N_{3,i} & :=  \frac{1 + M_{\ist}^2 \theta_{\ist}^2 /\gamma^2}{\Delta_i^2}  +  \frac{M_{\ist}^2 \mu_{\ist}^2/\gamma^2}{(\gamma - M_{\ist} \mu_{\ist})^2} \\
& + C_\gamma \max \Big \{ \xi_4(\tfrac{1}{a_{0,\ist} \mu_{\ist}}) \xi_4 ( \tfrac{M_{\ist} \mu_{\ist}} {\gamma - M_{\ist} \mu_{\ist}} ), \xi_4(\tfrac{1}{a_{0,i} \mu_{i}}) \Big \} .
\end{align*}
\item \textbf{Case 4 ($M_i \in \safe_i$ and $M_{\ist} \in \safe_{\ist}$)}:
\begin{align*}
N_{4,i}  & := \frac{1}{\Delta_i^2} + \frac{ M_{\ist}^2 \mu_{\ist}^2/\gamma^2}{ (\gamma - M_{\ist} \mu_{\ist})^2} + \frac{ M_{i}^2 \mu_{i}^2/\gamma^2}{(\gamma - M_{i} \mu_{i})^2}\\
& + C_\gamma \max_{j \in \{ i,\ist \} }   \xi_4(\tfrac{1}{a_{0,j} \mu_j})  \xi_4 ( \tfrac{M_{j} \mu_{j}}{\gamma - M_{j} \mu_{j}} ).  
\end{align*}
\end{enumerate}
Let $\case(i) \in \{1,2,3,4\}$ denote the case coordinate $i$ falls in. Then, with probability at least $1-\delta$, \Cref{alg:constrained_bai} will output $\ist$, only pull safe arms, and terminate after collecting at most
\begin{align*}
\cOtil \bigg ( \log \tfrac{d}{\delta} \cdot \sum_{i \neq \ist} N_{\case(i),i} + \log \tfrac{d}{\delta} \cdot \tfrac{d}{\gamma^8}  \bigg )
\end{align*}
samples.
\end{theorem}

We present the full version of \Cref{thm:linear_complexity2_simp} as \Cref{thm:linear_complexity2} in the appendix. While we only state the result for \emph{best}-coordinate identification, \alglinear could also be applied to obtain an $\epsilon$-good coordinate identification-style guarantee. When $M_i$ is unsafe for all $i$ (for example, if the maximum possible value $M_i$ is unbounded and $\mu_i >0$), we obtain the following. 

\begin{corollary}\label{cor:linear_simplified}
Assume that $M_i$ is unsafe for all $i$. Then, with probability at least $1-\delta$, \Cref{alg:constrained_bai} will output $\ist$, only pull safe arms, and terminate after collecting at most
\begin{align*}
\cOtil \bigg (& \log \tfrac{d}{\delta} \cdot \sum_{i \neq \ist} \frac{1 + \theta_{\ist}^2/\mu_{\ist}^2 + \theta_i^2/\mu_i^2}{\Delta_i^2}  \\
&  +  \log \tfrac{d}{\delta} \cdot \sum_{i \neq \ist} \max \Big \{ \tfrac{1}{\gamma^8}, C_\gamma \xi_4(\tfrac{1}{a_{0,i} \mu_i}) , C_\gamma \xi_4(\tfrac{1}{a_{0,\ist} \mu_{\ist}})  \Big \}  \bigg )
\end{align*}
samples.
\end{corollary}

We note that the sample complexity stated in \Cref{cor:linear_simplified} exactly matches the lower bound given in \Cref{thm:linear_lb}, up to constants and the lower order term scaling inversely in $a_{0,i}$ and $\gamma$, despite the fact that \Cref{thm:linear_lb} was proved for a more powerful class of learners. It follows that \alglinear achieves the near-optimal sample complexity for the problem, and does so while guaranteeing only safe arms are pulled.

\subsection{Proof Sketch}
At every epoch $\ell$, for all coordinates $i \in \calX_{\ell-1}$ we have not yet shown are suboptimal, we collect  $N_\ell = \lceil 2 \sigma^2 \log \frac{8d \ell^2}{\delta} \cdot \epsilon_\ell^{-2} \rceil$ samples at value $\asafehat^{\ell-1}(i)$. Standard concentration then gives that, with high probability,
\begin{align*}
\asafehat^{\ell-1}(i) |  \mu_i -  \muhat_{i,\ell} | \le \epsilon_\ell, \quad \asafehat^{\ell-1}(i) |  \theta_i -  \thetahat_{i,\ell} | \le \epsilon_\ell .
\end{align*}
Note then that
\begin{align*}
\frac{\gamma}{\muhat_{i, \ell} + \epsilon_\ell/\asafehat^{\ell-1}(i)} \le \frac{\gamma}{\mu_{i} -  \epsilon_\ell/\asafehat^{\ell-1}(i) + \epsilon_\ell/\asafehat^{\ell-1}(i)} = \frac{\gamma}{\mu_i}
\end{align*}
so it follows that $\gamma/(\muhat_{i, \ell} + \epsilon_\ell/\asafehat^{\ell-1}(i))$ is safe. As this is what we set $\asafehat^{\ell}(i)$ to (ignoring the range $[a_{0,i},M_i]$), it follows that $\asafehat^{\ell}(i)$ is always safe, so we only ever play safe values. A similar argument shows that $\aunhat^{\ell}(i)$ is always unsafe. Given the accuracy of our estimate of $\theta_i$, it follows that $\aunhat^\ell(i) (\thetahat_{i, \ell} + \epsilon_\ell/\asafehat^{\ell-1}(i))$ is an upper bound on the maximum safe value of coordinate $i$, while $\asafehat^\ell(i) (\thetahat_{i, \ell} - \epsilon_\ell/\asafehat^{\ell-1}(i))$ is a lower bound. As we only eliminate coordinates that have upper bounds less than the largest lower bound, it follows that we only eliminate suboptimal coordinates.

To bound the sample complexity, we first obtain a deterministic lower bound on $\asafehat^{\ell}(i)$ by solving a recursion (the dependence on $\xi_a(x)$ results from solving this recursion), and then, using this, obtain deterministic upper and lower bounds on $\aunhat^\ell(i) (\thetahat_{i, \ell} + \epsilon_\ell/\asafehat^{\ell-1}(i))$ and $\asafehat^\ell(i) (\thetahat_{i, \ell} - \epsilon_\ell/\asafehat^{\ell-1}(i))$. Noting that these are separated at most by the true gap, we show that if $\asafehat^{\ell}(i)$ is close enough to the maximum safe value, in order to eliminate arm $i$ it suffices to collect roughly $1/\Delta_i^2$ samples. Full details of the proof are given in \Cref{sec:linear_proofs}.

\section{MONOTONIC RESPONSE}
We turn now to our second setting, where we assume that our response is a monotonic function. 
In order to show that a coordinate is suboptimal, you must obtain an upper bound on the maximum safe function value. Without assuming more than monotonicity and smoothness, we cannot guarantee such an upper bound if we only allow the learner to sample points
$a$ such that $g_i(a) \le \gamma$---points that are safe. To address this, we allow the learner to sample points $a$ such that $g_i(a) \le \gamma + \epssafe$, where $\epssafe > 0$ is a parameter that may be specified as desired. However, the goal is still to determine the coordinate $\ist := \argmax_i f_i(g_i^{-1}(\gamma))$
that achieves the largest safe value.

\subsection{Algorithm Description}

\setlength{\textfloatsep}{14pt}
\begin{algorithm}[t!]
\begin{algorithmic}[1]
	\State \textbf{input:} Confidence $\delta$, safety gap $\epssafe$
	\State \textbf{initialize:} $\asafehat^{0,0}(i) \leftarrow a_{0,i} , \aunhat^{0,0}(i) \leftarrow a_{0,i}, \unsafe(i) \leftarrow 0$, $\calX_0 \leftarrow [d]$, $\ell \leftarrow 1$, $t \leftarrow 1$, $\epsilon_\ell \leftarrow 2^{-\ell}$, $n_i \leftarrow 1$, $m_i \leftarrow 1$
	\While{$| \calX_{\ell - 1} | > 1$}
		\State $N_{\ell,t} \leftarrow \lceil 2 \sigma^2 \log \frac{8 t^2}{\delta} \cdot \epsilon_\ell^{-2} \rceil$
		\For{$ i \in \calX_{\ell - 1} $}
			\State $\asafehat^{0,\ell}(i) \leftarrow \asafehat^{n_i-1,\ell-1}(i)$, $n_i \leftarrow 1$
			\State $\aunhat^{0,\ell}(i) \leftarrow \aunhat^{m_i-1,\ell-1}(i)$, $m_i \leftarrow 1$
			\State $\texttt{Estimate}_i(\asafehat^{0,\ell}(i),N_{\ell,t}$), $t \leftarrow t + 1$ \label{line:pull_enough}
			\While{$  \gamma - \ghat_i(\asafehat^{n_i-1,\ell}(i)) > 2 \epsilon_\ell$}
			\Statex \hspace{3.5em} {\color{blue} // increase safe value to $g_i^{-1}(\gamma)$}
			\State $\asafehat^{n_i,\ell}(i) \leftarrow  \gamma  + \asafehat^{n_i-1,\ell}(i) $ \label{line:asafehat_increment}
			\Statex \hspace{7em} $- \ghat_i(\asafehat^{n_i-1,\ell}(i)) - \epsilon_\ell $
			\State $\texttt{Estimate}_i(\asafehat^{n_i,\ell}(i),N_{\ell,t})$, $ t \leftarrow t +1$
			\State $n_i \leftarrow n_i + 1$
			\EndWhile
			\If{$\unsafe(i) = 0$}
			\Statex \hspace{3.5em} {\color{blue} // increase unsafe value to $g_i^{-1}(\gamma + \epssafe)$}
			\State $\texttt{Estimate}_i(\aunhat^{0,\ell}(i),N_{\ell,t})$, $ t \leftarrow t +1$
			\While{$ \gamma + \epssafe - \ghat_i(\aunhat^{m_i - 1,\ell}(i)) > 2 \epsilon_\ell$}\label{line:while_aun_unsafe0}
				\State $\aunhat^{m_i,\ell}(i) \leftarrow  \gamma + \epssafe + \aunhat^{m_i-1,\ell}(i)$  \label{line:aunhat_increment}
				\Statex \hspace{7em} $- \ghat_i(\aunhat^{m_i-1,\ell}(i)) - \epsilon_\ell$
				\State $\texttt{Estimate}_i(\aunhat^{m_i,\ell}(i),N_{\ell,t})$, $ t \leftarrow t +1$
				\If{$\ghat_i(\aunhat^{m_i,\ell}(i)) - \epsilon_\ell \ge \gamma$ }
					\State $\unsafe(i) \leftarrow 1$, \textbf{break}
				\EndIf
				\State $m_i \leftarrow m_i + 1$
				\EndWhile
			\Else
				\Statex \hspace{3.5em} {\color{blue} // decrease unsafe value to $g_i^{-1}(\gamma)$}
				\State $\aunhat^{1,\ell}(i) \leftarrow \aunhat^{0,\ell}(i)/2 + \asafehat^{n_i-1,\ell}(i)/2$, $m_i \leftarrow 2$
				\State $\texttt{Estimate}_i(\aunhat^{1,\ell}(i),N_{\ell,t})$, $ t \leftarrow t +1$
				\While{$\ghat_i(\aunhat^{m_i-1,\ell}(i)) - \epsilon_{\ell} < \gamma$}\label{line:while_binary_search}
					\State $\aunhat^{m_i,\ell}(i) \leftarrow \aunhat^{0,\ell}(i)/2 + \aunhat^{m_i-1,\ell}(i)/2$					
					\If{$\aunhat^{0,\ell}(i) - \aunhat^{m_i,\ell}(i) \le \epsilon_{\ell}$}\label{line:if_a_close}
						\State $\aunhat^{m_i,\ell}(i) \leftarrow \aunhat^{0,\ell}(i)$, \textbf{break}
					\EndIf
					\State $\texttt{Estimate}_i(\aunhat^{m_i,\ell}(i),N_{\ell,t})$, $ t \leftarrow t +1$
					\State  $m_i \leftarrow m_i + 1$
				\EndWhile
			\EndIf
		\EndFor
		\Statex \hspace{0.8em} { \color{blue} // eliminate suboptimal coordinate}
		\State $\calX_\ell \leftarrow \calX_{\ell-1} \backslash \{ i \in \calX_{\ell-1} \ : \ \unsafe(i) = 1$, 
		\Statex \hspace{2em} $\fhat_i(\aunhat^{m_i-1,\ell}(i)) + 2\epsilon_\ell \le \max_{j} \fhat_j(\asafehat^{n_j-1,\ell}(j))   \}$
		\State $\ell \leftarrow \ell + 1$
	\EndWhile
	\State \textbf{return} $\calX_{\ell-1}$
\end{algorithmic}
\caption{Safe Best-Arm Identification for Monotonic Functions (\algmono)}
\label{alg:constrained_bai_monotonic2}
\end{algorithm}

\begin{algorithm}[h]
\begin{algorithmic}[1]
\Function{${\normalfont \texttt{Estimate}}_i(a,N)$}{}
	\For{$t = 1,\ldots,N$}
	\State Pull coordinate $i$ at $a$, observe \\
	\hspace{4em} $y_t = f_i(a) + \eta_t, z_t = g_i(a) + w_t $
	\EndFor
	\State $\fhat_i(a) \leftarrow N^{-1} \sum_{t=1}^N y_t$, $\ghat_i(a) \leftarrow N^{-1} \sum_{t=1}^N z_t$
\EndFunction
\end{algorithmic}
\end{algorithm}

We propose a modification of \alglinear that takes into account the Lipschitz, monotonic nature of the function to guarantee safe learning, \algmono. Similar to \alglinear, \algmono proceeds in epochs, maintaining estimates of the largest verifiably safe values for each coordinate, and refining the tolerance at each epoch to learn better estimates of the function values. However, due to the structural differences present in the monotonic setting, \algmono deviates from \alglinear in several important ways.

\setlength{\textfloatsep}{20.0pt plus 2.0pt minus 4.0pt}
\textbf{Safe Value Updates.} \algmono updates the estimates of the safe values, $\asafehat^{n_i,\ell}(i)$, as
\begin{align*}
\asafehat^{n_i,\ell}(i) \leftarrow  \gamma  + \asafehat^{n_i-1,\ell}(i) - \ghat_i(\asafehat^{n_i-1,\ell}(i)) - \epsilon_\ell .
\end{align*}
At epoch $\ell$, we collect enough samples to guarantee that with high probability $|\ghat_i(\asafehat^{n_i-1,\ell}(i)) - g_i(\asafehat^{n_i-1,\ell}(i)) | \le \epsilon_\ell$. Thus, using that $g_i$ is 1-Lipschitz, 
\begin{align*}
g_i(\asafehat^{n_i,\ell}(i)) & \le g_i(\asafehat^{n_i-1,\ell}(i)) + \gamma - \ghat_i(\asafehat^{n_i-1,\ell}(i)) - \epsilon_\ell \\
& \le g_i(\asafehat^{n_i-1,\ell}(i)) + \gamma - g_i(\asafehat^{n_i-1,\ell}(i)) \\
& = \gamma
\end{align*}
so it follows that $\asafehat^{n_i,\ell}(i)$ is also safe.

\paragraph{Epoch Update Schedule.} Note that we can increase the value of $\asafehat^{n_i,\ell}(i)$ by at most a factor of roughly $\gamma - g_i(\asafehat^{n_i,\ell}(i))$ at every iteration of the update to $\asafehat^{n_i,\ell}(i)$. Assume that $\gamma - g_i(\asafehat^{n_i,\ell}(i)) \gg \epsilon_\ell$.
In this regime it follows the dominant term in the update
\begin{align*}
\asafehat^{n_i+1,\ell}(i) \leftarrow  \gamma  + \asafehat^{n_i,\ell}(i) - \ghat_i(\asafehat^{n_i,\ell}(i)) - \epsilon_\ell 
\end{align*}
is not $\epsilon_\ell$---our update increment is well above the ``noise floor''. 
As such, increasing $\ell$ will not help us, we should instead keep incrementing $\asafehat^{n_i,\ell}(i)$ until we arrive at a value where $\gamma - g_i(\asafehat^{n_i,\ell}(i)) \approx \epsilon_\ell$. At this point, we are close enough to the noise floor that decreasing $\epsilon_\ell$ will help us further increase $\asafehat^{n_i,\ell}(i)$. This is precisely the procedure \algmono implements, only increasing $\ell$ once the updates in each coordinate have reached the noise floor.

\paragraph{Unsafe Value Updates.} As noted, to guarantee a coordinate is suboptimal in the monotonic setting, we need to find a sample point $a$ such that $g_i(a) \ge \gamma$ and $g_i(a) \le \gamma + \epssafe$. Once we have found this value, we have an upper bound on the maximum safe value for coordinate $i$, yet this value may need to be refined, for example if the gap for coordinate $i$ is small. As such, once we find $a$ with $g_i(a) \ge \gamma$, we must decrease the $a$ value to get as close as possible to the threshold $g_i(a) = \gamma$. \algmono implements this by updating the ``unsafe'' estimate, $\aunhat^{m_i,\ell}(i)$, in two stages. In the first stage, while $\unsafe(i) = 0$, it applies an update analogous to the safe update, but which instead guarantees it will stay below the threshold $\gamma + \epssafe$. Once it can guarantee it has crossed the safety threshold, $g_i(\aunhat^{m_i,\ell}(i)) \ge \gamma$, it sets $\unsafe(i) = 1$ and decreases $\aunhat^{m_i,\ell}(i)$ while ensuring $g_i(\aunhat^{m_i,\ell}(i)) \ge \gamma$ via a binary search procedure.

The elimination criteria of \algmono is similar to that of \alglinear but differs in that it only eliminates a coordinate after it has crossed the safety threshold, allowing it to obtain an upper bound on that coordinate's value.

\subsection{Sample Complexity}
Before stating the sample complexity of \algmono, we make an additional assumption. 

\begin{assumption}\label{asm:inverse}
For all $x \in [\gamma, \gamma + \epssafe]$ and each $i \in [d]$, there exists some $a \in \bbR$ such $g_i(a) = x$. As such, the inverse $g_i^{-1}(x)$ is well defined for all $x \in [\gamma, \gamma + \epssafe]$.
\end{assumption}
This assumption is primarily for technical reasons and allows us to simplify the results somewhat. 

Let $\widetilde{x}_i := \inf_x x \text{ s.t. } \exists a, g_i(a) = x$ and define $\gtilinv_i$ as:
\begin{align*}
\gtilinv_i(x) = \left \{ \begin{matrix} a_{0,i} & x \le \widetilde{x}_i \\
g_i^{-1}(x) &  \widetilde{x}_i < x \le  \gamma + \epssafe  \\
g_i^{-1}(\gamma + \epssafe) & x > \gamma + \epssafe
\end{matrix} \right . .
\end{align*}
In words, $\gtilinv_i$ extends the inverse of $g_i^{-1}$ to outside its range. Now define the following:
\begin{align*}
\nbar_{i,\ell} & :=  \left \{ \begin{matrix} \frac{ \gtil_i^{-1}(\gamma -  \epsilon_\ell) - \gtil_i^{-1}(\gamma - 6 \epsilon_\ell)}{ \epsilon_\ell}& \ell \ge 2 \\
2(\gtil_i^{-1}(\gamma -  \tfrac{1}{2}) - a_{0,i}) & \ell = 1 \end{matrix} \right . , \\
\mbar_{i,\ell} & :=   \left \{ \begin{matrix} \frac{ \gtil_i^{-1}(\gamma + \epssafe -  \epsilon_\ell) - \gtil_i^{-1}(\gamma + \epssafe - 6 \epsilon_\ell)}{\epsilon_\ell}& \ell \ge 2 \\
 2(\gtil_i^{-1}(\gamma + \epssafe -  \tfrac{1}{2}) - a_{0,i}) & \ell = 1 \end{matrix} \right .   , \\
\aunbar^\ell(i) & :=   \sum_{s = 1}^{\ell} \frac{g_i^{-1}(\min \{ \gamma + 2 \epsilon_s, \gamma + \epssafe \})}{2^{\ell  - s + 1}} \\
 & \qquad \qquad + \left ( \ell + \frac{4 g_i^{-1}(\gamma + \epssafe)}{ \epssafe } \right ) 2^{-\ell}  , \\ 
 \ellbar(i) & := \argmin_{\ell \in \bbN} \ell \text{ s.t. } f_i(\aunbar^\ell(i)) + 4 \epsilon_\ell \le f_{\ist}(\gtil_{\ist}^{-1}(\gamma - 3 \epsilon_\ell))
 \end{align*}
and let $\ellbar(\ist) := \max_{i \neq \ist} \ellbar(i)$. We the have the following result.

\begin{theorem}\label{thm:monotonic_complexity}
Under Assumption \ref{asm:smooth_fun} and \ref{asm:inverse}, with probability $1-\delta$, \Cref{alg:constrained_bai_monotonic2} will terminate and output $\ist$ after collecting at most
\begin{align*}
\cOtil \bigg ( & \log \tfrac{1}{\delta} \cdot \sum_{i = 1, i \neq \ist}^d  \sum_{\ell = 1}^{\ellbar(i) } ( \mbar_{i,\ell} + \nbar_{i,\ell} + \ell + 2 )   2^{2 \ell}   \bigg )
\end{align*}
samples, and will only pull safe arms during execution. 
\end{theorem}

While this is a closed-form and deterministic expression, it is difficult to interpret in general. To obtain a more interpretable expression, we make the following additional assumption.
\begin{assumption}\label{asm:cont_deriv}
For all $i \in [d]$, $g_i$ is differentiable on $\bbR$, and $g_i^{-1}(x)$ is well-defined for all $x \in [\gamma - 3/2, \gamma + \epssafe]$. Furthermore,
$g_i'(a) \ge L^{-1}$
for all $a \in [g_i^{-1}(\gamma - 3/2), g_i^{-1}(\gamma + \epssafe)]$ and some $L > 0$. 
\end{assumption}

We do not assume that the learner knows either the value of $L$ or that $g_i$ satisfies Assumption \ref{asm:cont_deriv}. Under this assumption, we obtain the following corollary.

 \begin{corollary}\label{cor:mon_simplified_complexity}
Assume that Assumption \ref{asm:cont_deriv} holds. Then, with probability $1-\delta$, \Cref{alg:constrained_bai_monotonic2} will terminate and output $\ist$ after collecting at most
\begin{align*}
\cOtil \Bigg ( & \log \tfrac{1}{\delta} \cdot \sum_{i \neq \ist} \frac{1 + L^3 + (1+L) g_i^{-1}(\gamma)^2/\epssafe^2}{\Delta_i^2}   \\
& +   \sum_{i=1}^d \Big ( g_i^{-1}(\gamma + \epssafe - \tfrac{1}{2}) + g_i^{-1}(\gamma - \tfrac{1}{2}) - 2a_{0,i} \Big )  \Bigg )
\end{align*}
samples, and will only pull safe arms during execution. 
\end{corollary} 

\Cref{cor:mon_simplified_complexity} provides the familiar ``sum over inverse gap squared'' complexity, with the additional polynomial dependence on $L$, as well as the terms $g_i^{-1}(\gamma + \epssafe - \frac{1}{2}) + g_i^{-1}(\gamma - \frac{1}{2}) - 2a_{0,i}$ and $g_i^{-1}(\gamma)^2/\epssafe^2$. The former term, $g_i^{-1}(\gamma + \epssafe - \frac{1}{2}) + g_i^{-1}(\gamma - \frac{1}{2}) - 2a_{0,i}$, results from the initial phase of learning needed to guarantee that we are in the ``neighborhood'' of the safety boundary, $g_i^{-1}(\gamma)$, and increases as $a_{0,i}$ starts farther from the safety boundary. The latter term results from the complexity needed to learn a value above the safety threshold---if $\epssafe$ is small, we must explore more conservatively, which will increase our complexity.

 \newcommand{\safelts}{\textsc{Safe-LTS}\xspace}
 \newcommand{\safeopt}{\textsc{SafeOpt}\xspace}
\section{EXPERIMENTAL RESULTS} \label{sec:experiment}
In this section, we present experimental results on our algorithms and compare them with two existing models: the safe linear Thompson Sampling (\textsc{Safe-LTS}) algorithm of \cite{moradipari2021safe}, and the \textsc{SafeOpt} algorithm of \cite{sui2015safe}. We note that the \textsc{StageOpt} algorithm proposed in \cite{sui2018stagewise} relies on similar design principles to \textsc{SafeOpt}, which will cause its failure modes on the instances we present to be similar to that of \textsc{SafeOpt}. As such, we only plot results for \safeopt. 

\subsection{Algorithm Setup}
We demonstrate the superior performance of our algorithm in both a linear function model and a drug response model. We must point out that in both \cite{moradipari2021safe} and \cite{sui2015safe}, the problem setups do not aim to identify a best arm with high confidence--- \cite{moradipari2021safe} seeks to minimize regret, and \cite{sui2015safe} is unverifiable, their algorithm does not provide a stopping condition.
For effective comparison with our algorithms, we must choose proper stopping criteria for these models. For \safelts, we adopt a similar stopping criteria as \alglinear, with the confidence interval expression following \safelts's definition. For \textsc{SafeOpt}, we discretize the continuous values each arm can take into 50 evenly spaced values. We stop when \safeopt suggests all possibly optimal safe values have been found and there is an arm whose reward lower confidence bound is above all the reward upper confidence bounds of the remaining arms. In all our experiments, every algorithm always found the best arm, and none pulled any unsafe arms. More details on the experimental setup can be found in \Cref{app:experiment}.

\subsection{Linear Response Model}
In our first experiment, we consider a setup with $d\in \{5,10,20\}$ arms. We choose $\theta_1 = 1, \theta_2 = 0.9, \theta_{i>2} = 1$ and $\mu_1 =1,\mu_2=1.5, \mu_{i>2} = 5$, and set the safety threshold $\gamma$ to 1. With this setting, the minimum gap is $\Delta = 0.4$, and the remaining gaps are all $0.8$. The sample observations are perturbed by Gaussian noise with mean 0 and variance $\sigma^2 = 0.5$. 
The average simple regret values at $d=10$ are also computed to illustrate how fast the choice of arms in each pull is improving. We define the simple regret as $r_t = \theta_{\ist} a_{i^*} - \theta_{\widehat{i}_t} \asafehat^t(\widehat{i}_t)$, for $\widehat{i}_t = \ \argmax_i \thetahat_{i,t} \asafehat^t(i)$,
where $\thetahat_{i,t}$ is each algorithm's estimate of $\theta_i$ at time $t$, and $\asafehat^t(i)$ is the largest verifiably safe value for arm $i$ at time $t$. In this setting, we run \safeopt with a linear kernel. All data points are the average of 10 trials.

From \Cref{fig:linear_pulls}, we observe that the number of arm pulls scales linearly with $d$, and \alglinear outperforms other two algorithms significantly. 
Next, as seen from \Cref{fig:linear_regrets}, \alglinear shows the fastest decrease in simple regret as the number of arm pulls increases, and attains near zero regret after the best arm is identified. 
On the other hand, both \safelts and \textsc{SafeOpt} show a slower rate of decrease in simple regret. Intuitively the worse performances of the other two algorithms may be attributed to unnecessary pulls wasted to find the largest safe value for suboptimal coordinates. In contrast, \alglinear can quickly identify and remove the suboptimal arms before their largest safe values is reached, preventing overexploration.

\begin{figure}[t!] 
    \centering
    \includegraphics[width=0.38\textwidth]{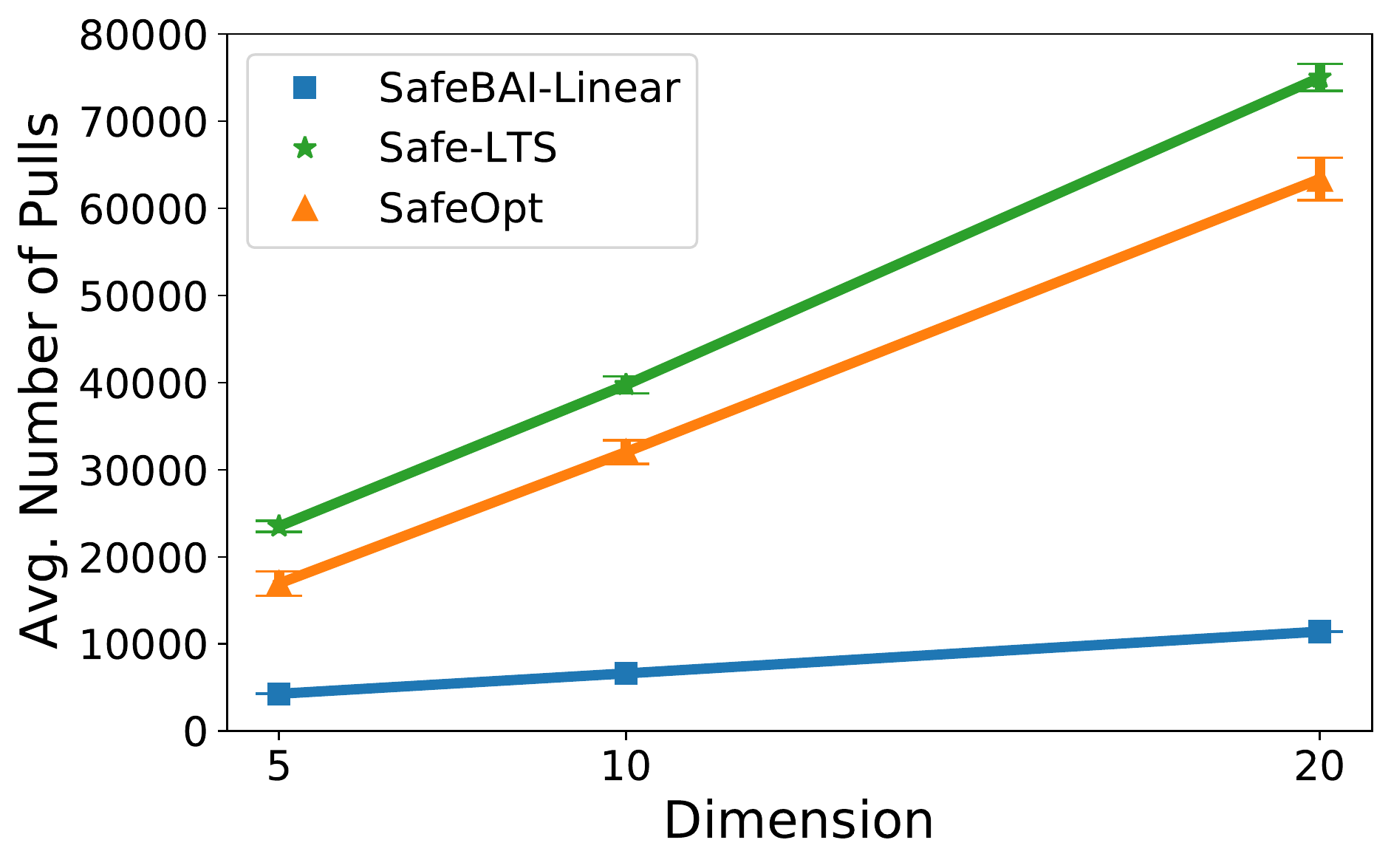} 
    \caption{Total arm pulls to termination vs. dimension in the linear model. Error bars give 1 standard deviation.}\label{fig:linear_pulls}
    \vspace*{-1.5ex}
\end{figure}

\begin{figure}[t!] 
    \centering
    \includegraphics[width=0.38\textwidth]{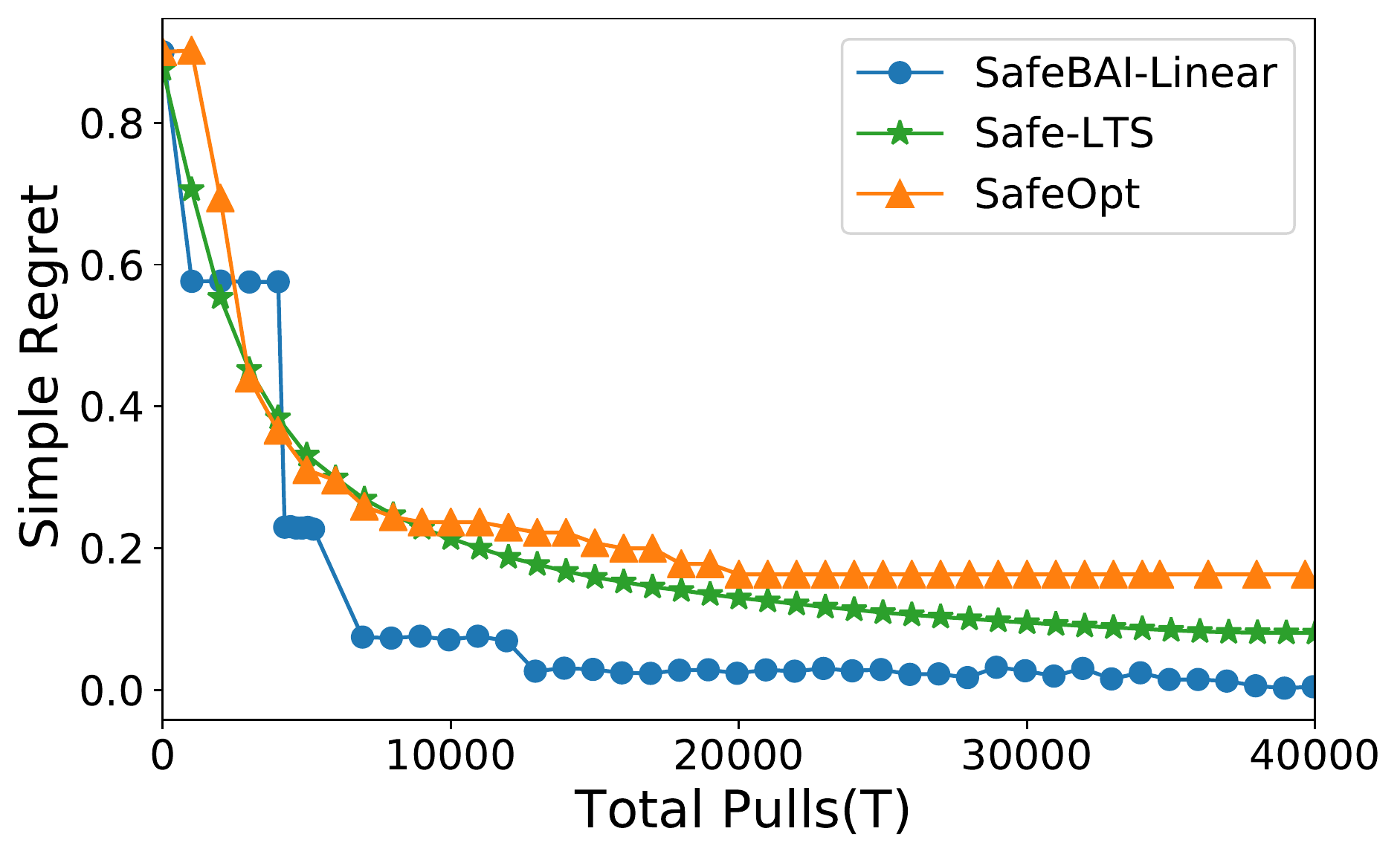} 
    \caption{Simple regret in linear model with $d=10$. }\label{fig:linear_regrets}
    \vspace*{-1.5ex}
\end{figure}

\begin{figure}[t!] 
    \centering
    \includegraphics[width=0.38\textwidth]{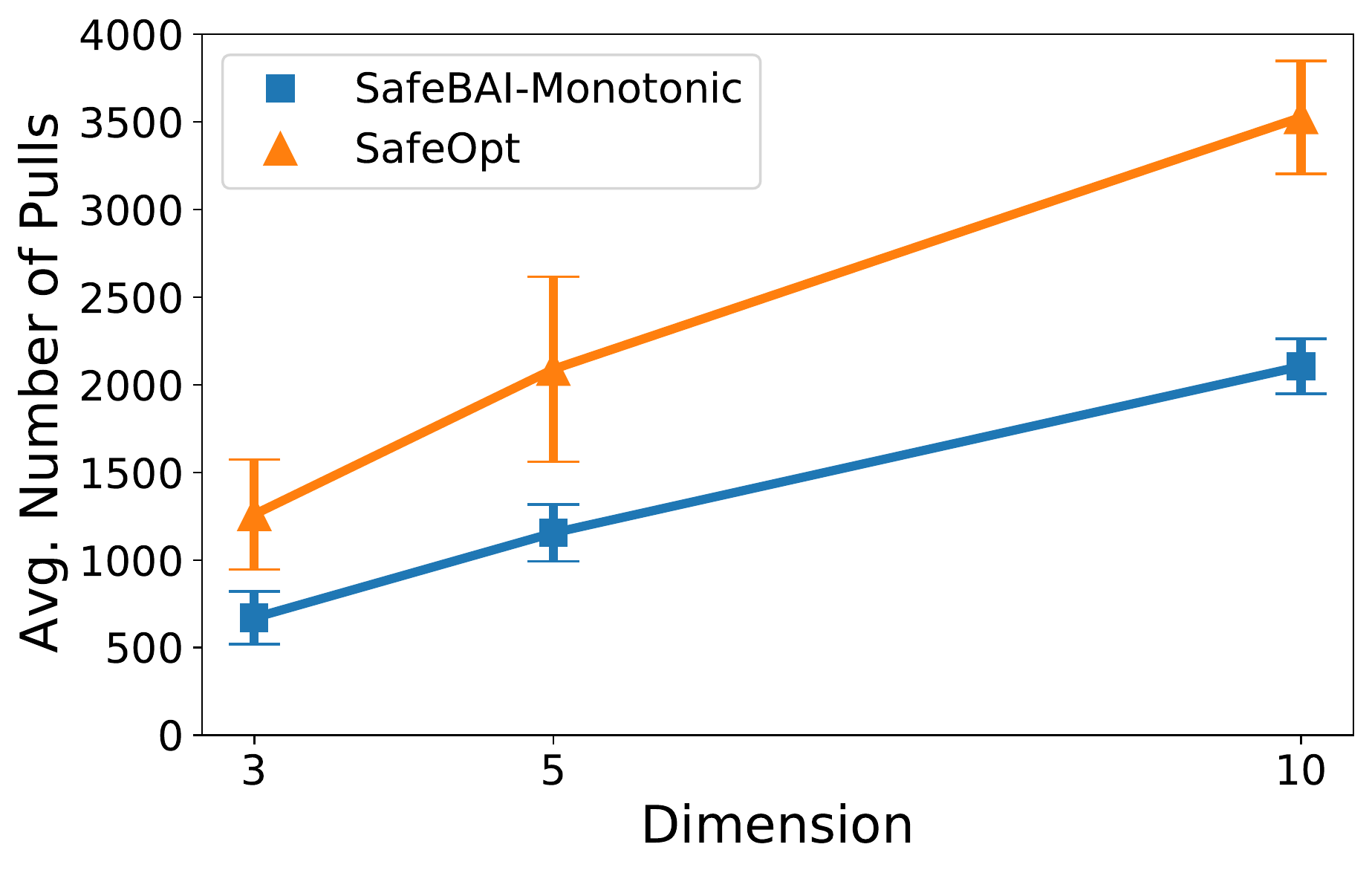} 
    \caption{Total arm pulls to termination vs. dimension in drug response model. Error bars give 1 standard deviation. }\label{fig:mono_pulls}
    \vspace*{-1.5ex}
\end{figure}

\begin{figure}[t!] 
    \centering
    \includegraphics[width=0.38\textwidth]{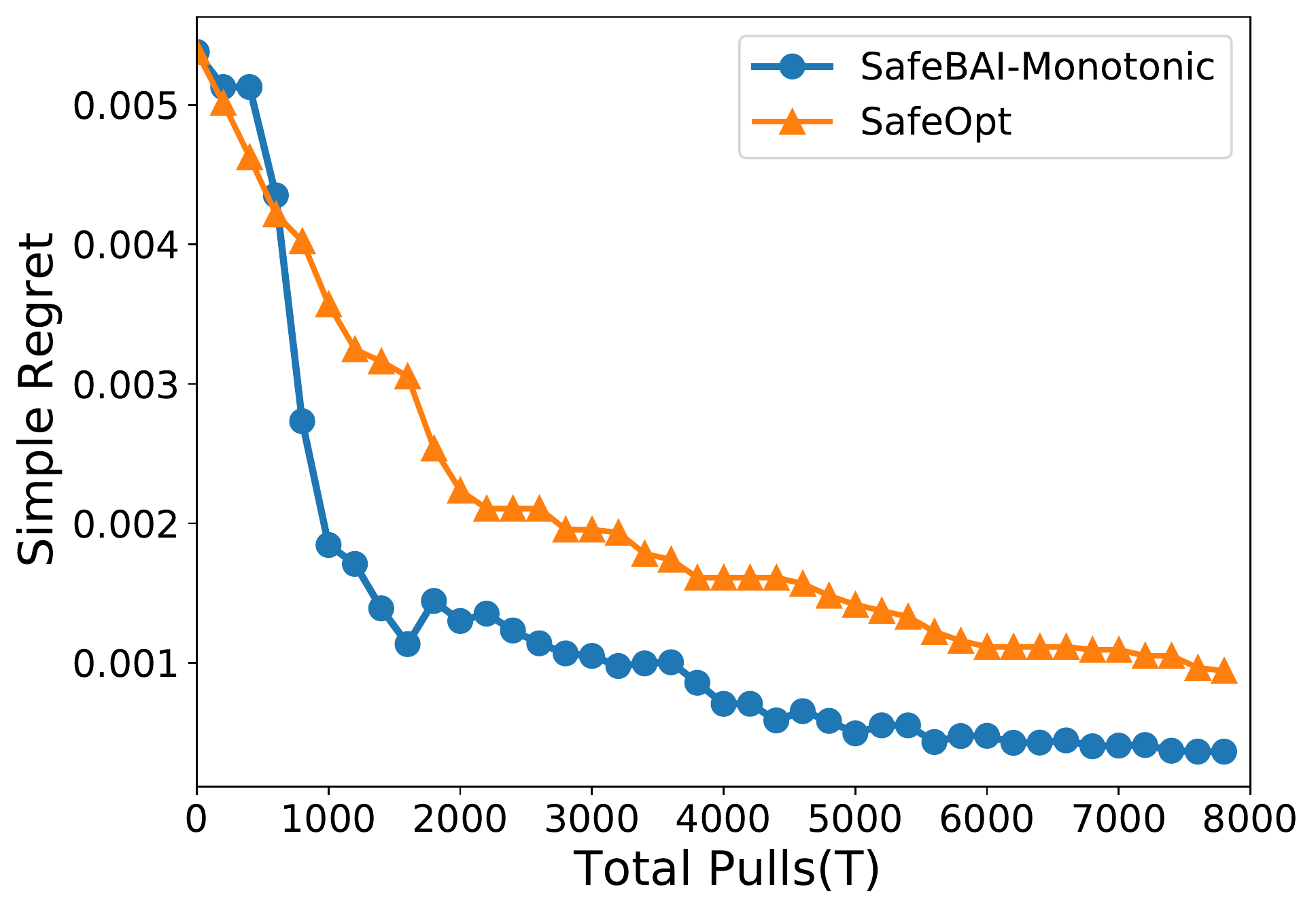} 
    \caption{Simple regret in drug response model with $d = 20$. }\label{fig:mono_regrets}
    \vspace*{-1.5ex}
\end{figure}

\subsection{Application to Best Drug Identification}
In the second experiment, we investigate the performance of \algmono in comparison to \textsc{SafeOpt} on a nonlinear drug response model. In the dose-finding literature, it is often assumed that higher dosage leads to stronger drug response and efficacy, while also increasing toxicity \citep{mandrekar2007adaptive,yuan2009bayesian,cai2014bayesian}. A common assumption (see e.g. \cite{thall1998strategy} and aforementioned works) has been to rely on a logistic function to model both the efficacy and toxicity: for the $i$th drug and dose $a$, the efficacy is modeled as $f_i(a) = \frac{1}{1+e^{-\theta_i\cdot a}}$ and toxicity as $g_i(a) = \frac{1}{1+e^{-\mu_i \cdot a}}$. Note that this choice of $f_i$ and $g_i$ is monotonic and fits within our setting.
Applying this model, we can effectively consider a drug selection setup where our goal is to identify the drug with highest utility among a set of candidate drugs, while ensuring none of the dosages tested lead to toxic response, i.e. all drug tests satisfy a safety constraint. 
We design drug sets with a number of $d \in \{3, 5, 10, 20\}$ drugs, efficacy model parameters $\theta_1 = 0.01, \theta_{i>1} = 10$, toxicity model parameters $\mu_i = 1$. This gives a minimum gap $\Delta \approx 0.497$. The safety threshold $\gamma$ is set to 0.3. For \safeopt, we use an RBF kernel. The sample observations are perturbed by Gaussian noise with mean 0 and variance $\sigma^2 = 0.1$. Similar to the linear case, we will also plot the simple regret values, where here we define the simple regret as $r_t = f_{\ist} (a_{i^*}) - f_{\widehat{i}_t}(\asafehat^t(\widehat{i}_t))$, for $\widehat{i}_t = \ \argmax_i \fhat_{i}(\asafehat^t(i)).$
We run \safeopt with an RBF kernel in this setting. All data points are the average of 20 trials.

The result in \Cref{fig:mono_pulls} suggests \algmono is able to much more effectively identify the best safe drug than \safeopt. The total number of drug evaluations required for \algmono is roughly half of \safeopt in all instances. Furthermore, the expensive posterior update in \textsc{SafeOpt} makes it slow to run with samples $> 5000$, while \algmono can be efficiently run when many more samples are required, allowing it to easily scale to trials with more drugs. \Cref{fig:mono_regrets} illustrates that \algmono also performs much better in terms of the simple regret. Not only is it able to identify and verify the optimal coordinate more quickly, even before verifying it has found the optimal coordinate, its sampling strategy allows it to obtain better intermediate estimates of the best coordinate.

\subsection*{Acknowledgements} 
The work of AW is supported by an NSF GFRP Fellowship DGE-1762114. The work of KJ is supported in part by grants
NSF CCF 2007036 and NSF TRIPODS 2023166.

\bibliographystyle{icml2021}
\bibliography{references}

\newpage
\onecolumn
\appendix
%!TEX root = ../BAI_Constraint.tex

\section{Linear Response Functions Proof}\label{sec:linear_proofs}

\subsection{Correctness of \alglinear}
\begin{lemma} \label{lem:hoeff_bound}
    For any $i \in [d]$ and $\ell \in \bbN^{+}$, define the events
    \begin{equation}
        \calE_{i, \ell}(\theta) = \left \{|\thetahat_{i, \ell} - \theta_i| \leq \frac{2^{-\ell}}{\asafehat^{\ell-1}(i)} \right \}, \quad \calE_{i, \ell}(\mu) = \left \{|\muhat_{i, \ell} - \mu_i| \leq \frac{2^{-\ell}}{\asafehat^{\ell-1}(i)}\right \}
    \end{equation}
    and let 
    \begin{align*}
    \calE := \left ( \bigcap_{i=1}^d \bigcap_{\ell = 1}^{\infty} \calE_{i, \ell}(\theta) \right ) \cap  \left ( \bigcap_{i=1}^d \bigcap_{\ell = 1}^{\infty} \calE_{i, \ell}(\mu) \right ).
    \end{align*}
   Then it holds that $\bbP(\calE) \ge 1 - \delta$. 
\end{lemma}

\begin{proof} [Proof of \Cref{lem:hoeff_bound}]
    Given that at $\ell$-th iteration, $N_\ell = \lceil2 \sigma^2 \log \frac{8d \ell^2}{\delta} \cdot \epsilon_{\ell}^{-2} \rceil$  and 
    $$\asafehat^{\ell-1}(i) \thetahat_{i,\ell} = \frac{1}{N_\ell}\sum_{t= 1}^{n_i} ( \asafehat^{\ell-1}(i)\theta_i + \eta_t) = \asafehat^{\ell-1}(i)\theta_i + \frac{1}{n_i} \sum_{t=1}^{N_\ell} \eta_{t}.$$
    By Hoeffding's inequality, $|\asafehat^{\ell-1}(i)\theta_i - \asafehat^{\ell-1}(i)\thetahat_{i, \ell}| \leq  \sqrt{\frac{2 \sigma^2 \log(\frac{8d \ell^2}{\delta})}{N_\ell}} = 2^{-\ell}$ with probability $1-\frac{\delta}{4d \ell^2}$. Thus we have that $\bbP(\calE^c_{i, \ell}(\theta)) = \bbP(\calE^c_{i, \ell}(\mu)) \leq \frac{\delta}{4d \ell^2}$. Taking union bound over all coordinates and rounds we have:
    \begin{equation}
        \bbP(\bigcup^{d}_{i=1}\bigcup^{\infty}_{\ell=1}\calE^c_{i, \ell}(\mu)) \leq \sum^{d}_{i=1}\sum^{\infty}_{\ell=1}  \frac{\delta}{4d \ell^2} = d \cdot \frac{\delta}{4d}  \cdot \frac{\pi^2}{6} \leq \delta/2.
    \end{equation}
    Using the same inequality on the union bound we can obtain the inequality result for $\theta$ as well, and union bounding over both gives the result.
\end{proof}

\begin{lemma}\label{lem:asafehat_safe}
On the event $\calE$, for all $\ell$ and $i \in \calX_{\ell-1}$, we will have that $\asafehat^\ell(i) \in \safe_i$. 
\end{lemma}
\begin{proof}
On the event $\calE$, $|\muhat_{i,\ell}- \mu_i| \leq \frac{2^{-\ell}}{\asafehat^{\ell-1}(i)}, |\thetahat_{i,\ell} - \theta_i| \leq \frac{2^{-\ell}}{\asafehat^{\ell-1}(i)}$. Then using our choice of $\asafehat^\ell(i)$, it follows that $\asafehat^\ell(i) \mu_i \le \asafehat^\ell(i) (\muhat_{i, \ell} + \tfrac{\epsilon_\ell}{ \asafehat^{\ell-1}(i)}) = \gamma $. Thus, by construction, $\asafehat^\ell(i)$ will be safe. 
\end{proof}

\begin{lemma}\label{lem:asafe_val}
On the event $\calE$, coordinate $i$ is suboptimal if the event  
    \begin{equation}\label{eq:linear_arm_subopt}
        \{\exists j \in [n], \ell > 0: \aunhat^\ell(i) (\thetahat_{i,\ell} + \frac{2^{-\ell}}{\asafehat^{\ell-1}(i)}) < \asafehat^\ell(j) (\thetahat_{j,\ell} - \frac{2^{-\ell}}{\asafehat^{\ell-1}(j)})\}
    \end{equation}
    holds, so $\ist \in \calX_{\ell}$ for all $\ell$.
\end{lemma}
\begin{proof}[Proof of Lemma \ref{lem:asafe_val}]
    Fix coordinate $i$ and its largest safe value $a_s(i)$. By \Cref{lem:asafehat_safe}, we will have that $\asafehat^\ell(i)$ is safe. Further, applying our choice of $\aunhat^\ell(i)$, we obtain
    $$\asafehat^\ell(i) = M_i \wedge \gamma / (\muhat_{i, \ell} + \tfrac{\epsilon_\ell}{ \asafehat^{\ell-1}(i)})  \leq a_s(i) = M_i \wedge \gamma / \mu_i \leq  M_i \wedge \gamma / (\muhat_{i, \ell} - \tfrac{\epsilon_\ell}{ \asafehat^{\ell-1}(i)}) =  \aunhat^\ell(i).$$ Applying this set of inequalities $\asafehat^\ell(i) \leq a_s(i) \leq \aunhat^\ell(i)$, it follows that
    \begin{align*}
        &\aunhat^\ell(i) (\thetahat_{i,\ell} + \frac{2^{-\ell}}{\asafehat^{\ell-1}(i)}) < \asafehat^\ell(j) (\thetahat_{j,\ell} - \frac{2^{-\ell}}{\asafehat^{\ell-1}(j)}) \\
        & \implies a_s(i) (\thetahat_{i,\ell} + \frac{2^{-\ell}}{\asafehat^{\ell-1}(i)}) < \asafehat^\ell(j) (\thetahat_{j,\ell} - \frac{2^{-\ell}}{\asafehat^{\ell-1}(j)}) \\
        & \implies a_s(i)\theta_i < a_s(j)\theta_j.
    \end{align*}
    The last inequality suggests $i$ being suboptimal by definition. That $\ist \in \calX_{\ell}$ for all $\ell$ follows directly since \eqref{eq:linear_arm_subopt} is identical to the condition in \Cref{alg:constrained_bai} to remove a coordinate, and so, since \eqref{eq:linear_arm_subopt} implies $i$ is suboptimal, it follows that $\ist$ can never be removed.
\end{proof}

\subsection{Coordinate Elimination Criteria}
\begin{lemma} \label{lem:l_ineq}

    Assume $\asafehat^{\ell}(i)< M_i$ and  $\asafehat^{\ell}(i^*)< M_{i^*}$ and that $\calE$ holds. Let $\alowhat(i,\ell)$ be a lower bound for $\asafehat^\ell(i)$, then any suboptimal coordinate $i$ will be eliminated after the $\ell$-th iteration where $\ell$ satisfies:
    \begin{equation}
        \frac{\theta_i \alowhat(i,\ell-1) + 2^{1-\ell}}{\mu_i \alowhat(i,\ell-1) - 2^{1-\ell}} \leq \frac{\theta_{i^*} \alowhat(i^*,\ell-1) - 2^{1-\ell}}{\mu_{i^*} \alowhat(i^*,\ell-1) + 2^{1-\ell}}.
    \end{equation}
\end{lemma}

\begin{proof}[Proof of Lemma \ref{lem:l_ineq}]
    Since on $\calE$, $\ist \in \calX_\ell$ for all $\ell$ as shown in \Cref{lem:asafe_val}, a sufficient condition to eliminate coordinate $i$ is that $\ist$ has a larger estimated maximum safe value, and we can therefore reduce the elimination criteria to simply a comparison between $i$ and $\ist$. First we transform this inequality via:
    \begin{equation}\label{eq:elim_suff}
        \frac{\theta_i \alowhat(i,\ell-1) + 2^{1-\ell}}{\mu_i \alowhat(i,\ell-1) - 2^{1-\ell}} \leq \frac{\theta_{i^*} \alowhat(i^*,\ell-1) - 2^{1-\ell}}{\mu_{i^*} \alowhat(i^*,\ell-1) + 2^{1-\ell}}
        \implies 
        \frac{\gamma(\theta_i + \frac{2^{1-\ell}}{\alowhat(i,\ell-1)})}{\mu_i - \frac{2^{1-\ell}}{\alowhat(i, \ell-1)}} \leq \frac{\gamma (\theta_{i^*} - \frac{2^{1-\ell}}{\alowhat(i^*,\ell-1)})}{\mu_{i^*} + \frac{2^{1-\ell}}{\alowhat(i^*, \ell-1)}}
    \end{equation}
    Note that, by their definition and on $\calE$, $\asafehat^\ell(i^*) \geq \frac{\gamma}{\mu_{i^*} + \frac{2^{1-\ell}}{\alowhat(i^*, \ell-1)}}$, $\aunhat^\ell(i) \leq \frac{\gamma}{\mu_i - \frac{2^{1-\ell}}{\alowhat(i, \ell-1)}}$. Using this and that $\alowhat(i,\ell) \leq \asafehat^\ell(i)$ and $|\thetahat_{i,\ell} - \theta_i| \leq \frac{2^{-\ell}}{\asafehat^{\ell-1}(i)}$, we have that \eqref{eq:elim_suff} implies:
    \begin{align*}
        & \aunhat^\ell(i) (\theta_i + \frac{2^{1-\ell}}{\alowhat(i,\ell-1)}) < \asafehat^\ell(i^*) (\theta_{i^*} - \frac{2^{1-\ell}}{\alowhat(i^*,\ell-1)}) \\
        & \implies \aunhat^\ell(i) (\theta_i + 2\frac{2^{-\ell}}{\asafehat^{\ell-1}(i)}) < \asafehat^\ell(i^*) (\theta_{i^*} - 2\frac{2^{-\ell}}{\asafehat^{\ell-1}(i^*)}) \\
        & \implies \aunhat^\ell(i) (\thetahat_{i,\ell} + \frac{2^{-\ell}}{\asafehat^{\ell-1}(i)}) < \asafehat^\ell(i^*) (\thetahat_{i^*,\ell} - \frac{2^{-\ell}}{\asafehat^{\ell-1}(i^*)})
    \end{align*}
    From the last inequality and Lemma \ref{lem:asafe_val}, we have that $i$ stops getting sampled at such $\ell$.
\end{proof}

\begin{lemma} \label{lem:sat_ineq_i}
    On $\calE$, for any saturated coordinate $i$ (i.e. $a_s(i) = M_i$), when $i^*$ is not saturated, we will eliminate $i$ after the $\ell$-th iteration when $\ell$ satisfies:
    \begin{equation}
        M_{i}\theta_{i} + 2^{1-\ell} < \gamma \cdot\frac{\theta_{i^*}\alowhat({i^*},\ell-1) - 2^{1-\ell}}{\mu_{i^*}\alowhat({i^*},\ell-1) + 2^{1-\ell}}
    \end{equation}
\end{lemma}

\begin{proof}[Proof of Lemma \ref{lem:sat_ineq_i}]
    Since on $\calE$, $\ist \in \calX_\ell$ for all $\ell$ as shown in \Cref{lem:asafe_val}, a sufficient condition to eliminate coordinate $i$ is that $\ist$ has a larger estimated maximum safe value, and we can therefore reduce the elimination criteria to simply a comparison between $i$ and $\ist$. First we transform this inequality via:
    \begin{equation}\label{eq:elim_cond2}
        M_{i}\theta_{i} + 2^{1-\ell} < \gamma \cdot\frac{\theta_{i^*}\alowhat({i^*},\ell-1) - 2^{1-\ell}}{\mu_{i^*}\alowhat({i^*},\ell-1) + 2^{1-\ell}}
        \implies 
        M_{i}(\theta_{i} + \frac{2^{1-\ell}}{M_{i}}) < \frac{\gamma \cdot(\theta_{i^*} - \frac{2^{1-\ell}}{\alowhat({i^*},\ell-1)})}{\mu_{i^*}+\frac{2^{1-\ell}}{\alowhat({i^*},\ell-1)}} 
    \end{equation}
    Note that $\asafehat^\ell(i) = M_{i} = \asafehat^{\ell-1}(i)$, $\asafehat^\ell(i^*) \geq \frac{\gamma}{\mu_{i^*} + \frac{2^{1-\ell}}{\alowhat(i^*, \ell-1)}}$. Using $\alowhat(i,\ell) \leq \asafehat^\ell(i)$ and $|\thetahat_{i,\ell} - \theta_i| \leq \frac{2^{-\ell}}{\asafehat^{\ell-1}(i)}$, we have that \eqref{eq:elim_cond2} implies:
    \begin{align*}
        & M_{i}(\theta_{i} + \frac{2^{1-\ell}}{M_{i}}) < \asafehat^\ell(i^*) (\theta_{i^*} - 2\frac{2^{-\ell}}{\asafehat^{\ell-1}(i^*)}) \\
        & \implies \aunhat^\ell(i) (\theta_i + 2\frac{2^{-\ell}}{\asafehat^{\ell-1}(i)}) < \asafehat^\ell(i^*) (\theta_{i^*} - 2\frac{2^{-\ell}}{\asafehat^{\ell-1}(i^*)}) \\
        & \implies \aunhat^\ell(i) (\thetahat_{i,\ell} + \frac{2^{-\ell}}{\asafehat^{\ell-1}(i)}) < \asafehat^\ell(i^*) (\thetahat_{i^*,\ell} - \frac{2^{-\ell}}{\asafehat^{\ell-1}(i^*)}).
    \end{align*}
    From the last inequality and Lemma \ref{lem:asafe_val}, we have that $i$ stops getting sampled at such $\ell$.
\end{proof}

\begin{lemma} \label{lem:sat_ineq_i_star}
On $\calE$, when $i^*$ is a saturated coordinate (i.e. $a_s(i^*) = M_{i^*}$), then for any unsaturated suboptimal coordinate $i$, we will eliminate $i$ after the $\ell$-th iteration where $\ell$ satisfies:
    \begin{equation}
        \gamma \cdot\frac{\theta_i\alowhat(i,\ell-1) + 2^{1-\ell}}{\mu_i\alowhat(i,\ell-1) - 2^{1-\ell}} < M_{i^*}\theta_{i^*} - 2^{1-\ell}
    \end{equation}
\end{lemma}

\begin{proof}[Proof of Lemma \ref{lem:sat_ineq_i_star}]
    Since on $\calE$, $\ist \in \calX_\ell$ for all $\ell$ as shown in \Cref{lem:asafe_val}, a sufficient condition to eliminate coordinate $i$ is that $\ist$ has a larger estimated maximum safe value, and we can therefore reduce the elimination criteria to simply a comparison between $i$ and $\ist$. First we transform this inequality via:
    \begin{equation}\label{eq:elim_cond3}
        \gamma \cdot\frac{\theta_i\alowhat(i,\ell-1) + 2^{1-\ell}}{\mu_i\alowhat(i,\ell-1) - 2^{1-\ell}} < M_{i^*}\theta_{i^*} - 2^{1-\ell}
        \implies 
        \frac{\gamma \cdot(\theta_i + \frac{2^{1-\ell}}{\alowhat(i,\ell-1)})}{\mu_i-\frac{2^{1-\ell}}{\alowhat(i,\ell-1)}} < M_{i^*}(\theta_{i^*} - \frac{2^{1-\ell}}{M_{i^*}})
    \end{equation}
    Note that $\asafehat^\ell(i^*) = M_{i^*} = \asafehat^{\ell-1}(i^*)$, $\aunhat^\ell(i) \leq \frac{\gamma}{\mu_i - \frac{2^{1-\ell}}{\alowhat(i, \ell-1)}}$. Using $\alowhat(i,\ell) \leq \asafehat^\ell(i)$ and $|\thetahat_{i,\ell} - \theta_i| \leq \frac{2^{-\ell}}{\asafehat^{\ell-1}(i)}$, we have \eqref{eq:elim_cond3} implies:
    \begin{align*}
        & \aunhat^\ell(i) (\theta_i + \frac{2^{1-\ell}}{\alowhat(i,\ell-1)}) < M_{i^*}(\theta_{i^*} - \frac{2^{1-\ell}}{M_{i^*}}) \\
        & \implies \aunhat^\ell(i) (\theta_i + 2\frac{2^{-\ell}}{\asafehat^{\ell-1}(i)}) < \asafehat^\ell(i^*) (\theta_{i^*} - 2\frac{2^{-\ell}}{\asafehat^{\ell-1}(i^*)}) \\
        & \implies \aunhat^\ell(i) (\thetahat_{i,\ell} + \frac{2^{-\ell}}{\asafehat^{\ell-1}(i)}) < \asafehat^\ell(i^*) (\thetahat_{i^*,\ell} - \frac{2^{-\ell}}{\asafehat^{\ell-1}(i^*)})
    \end{align*}
    From the last inequality and Lemma \ref{lem:asafe_val}, we have that $i$ stops getting sampled at such $\ell$.
\end{proof}

\subsection{Lower Bounding Safe Value Estimate}
\begin{lemma} \label{lem:asafe_lb}
Define
\begin{align*}
\asafelow^{\ell}(i) :=  \frac{\gamma^{\ell}2^{\ell (\ell+1)/2}}{\mu_i\sum^{\ell}_{k=1}2^{\ell-k}\gamma^{k-1}2^{k(k+1)/2} + \frac{2^{\ell}}{a_{0,i}}} .
\end{align*}
    Let $\ell_0 := \argmin_\ell \ell \text{ s.t. } \asafelow^\ell(i) \ge M_i$. Then, on $\calE$:
    \begin{equation*}
         \asafehat^{\ell}(i) \geq \left \{ \begin{matrix} \min \{ \max \{ \asafelow^{\ell}(i), \ a_{0,i}  \}, M_i \} & \ell < \ell_0 \\
         M_i & \ell \ge \ell_0 \end{matrix} \right . . 
    \end{equation*}
\end{lemma}

\begin{proof}[Proof of Lemma \ref{lem:asafe_lb}]
    Recall our choice of $\asafehat^\ell(i) = \min \{ \max \{ \gamma / (\muhat_{i,\ell} + \frac{2^{-\ell}}{\asafehat^{\ell-1}(i)}), a_{0,i} \}, M_i \}$. Let $\ell_0 := \argmin_\ell \ell \text{ s.t. } \asafelow^\ell(i) \ge M_i$. We will first show that for $\ell < \ell_0$, we can lower bound $\asafehat^{\ell}(i) \ge \max \{ \asafelow^{\ell}(i), \ a_{0,i}  \}$. We prove this lemma via induction.

    In the base case, since on $\calE$, $|\muhat_{i,\ell} - \mu_i| \leq \frac{2^{-\ell}}{\asafehat^{\ell-1}(i)}$, we have 
    $$\asafehat^{1}(i) =  \max\{ \frac{\gamma}{\muhat_{i,1} + \frac{2^{-1}}{a_{0,i}}}, a_{0,i} \} \geq \max\{\frac{\gamma^1}{\mu_i + 2\frac{2^{-1}}{a_{0,i}}} , a_{0,i} \} = \max\{\frac{\gamma^1 \cdot 2}{2\mu_i + \frac{2}{a_{0,i}}} , a_{0,i} \} = \max \{ \asafelow^\ell(i), a_{0,i} \},$$ 
    hence the base case holds. Next, suppose the inequality holds for $\ell-1$, i.e. $\asafehat^{\ell-1}(i) \geq \max \{ \asafelow^{\ell-1}(i), a_{0,i} \}$, then we obtain, on $\calE$,
    \begin{align*}
        \frac{\gamma}{\muhat_{i,\ell} + \frac{2^{-\ell}}{\asafehat^{\ell-1}(i)}} &\geq  \frac{\gamma}{\mu_i + 2\cdot  \frac{2^{-\ell}}{\asafehat^{\ell-1}(i)}} \\
            &= \frac{\gamma}{\mu_i + 2\cdot  \frac{2^{-\ell}}{\max \{ \gamma / (\muhat_{i,\ell-1} + \frac{2^{1-\ell}}{\asafehat^{\ell-2}(i)}), a_{0,i}\}}} \\
            &\geq \frac{\gamma}{\mu_i + 2\cdot  \frac{2^{-\ell}}{\gamma / (\muhat_{i,\ell-1} + \frac{2^{1-\ell}}{\asafehat^{\ell-2}(i)})}} \\
            &\geq \frac{\gamma}{\mu_i + 2 \cdot 2^{-\ell} \frac{\mu_i\sum^{\ell-1}_{k=1}2^{\ell-1-k}\gamma^{k-1}2^{\sum^{k}_{s=1}s} + \frac{2^{\ell-1}}{a_{0,i}}}{\gamma^{\ell-1}2^{\sum^{\ell-1}_{k=1}k}}}  \\
            &= \frac{\gamma \cdot \gamma^{\ell-1}2^{\sum^{\ell-1}_{k=1}k}}{\mu_i \cdot \gamma^{\ell-1}2^{\sum^{\ell-1}_{k=1}k} + 2 \cdot 2^{-\ell} (\mu_i\sum^{\ell-1}_{k=1}2^{\ell-1-k}\gamma^{k-1}2^{\sum^{k}_{s=1}s} + \frac{2^{\ell-1}}{a_{0,i}})} \\
            &= \frac{\gamma^{\ell}2^{\sum^{\ell-1}_{k=1}k}2^{\ell}}{\mu_i \cdot \gamma^{\ell-1}2^{\sum^{\ell-1}_{k=1}k}2^{\ell} + 2\mu_i\sum^{\ell-1}_{k=0}2^{\ell-1-k}\gamma^{k-1}2^{\sum^{k}_{s=1}s} + \frac{2^{\ell}}{a_{0,i}}} \\
            &= \frac{\gamma^{\ell}2^{\sum^{\ell}_{k=1}k}}{\mu_i\sum^{\ell}_{k=1}2^{\ell-k}\gamma^{k-1}2^{\sum^{k}_{s=1}s} + \frac{2^{\ell}}{a_{0,i}}}.
    \end{align*}
    Since $\sum^{k}_{s=1}s = s(s+1)/2$, it follows then that
    \begin{align*}
    \asafehat^\ell(i) & =  \min \bigg \{ \max\bigg\{\frac{\gamma}{\muhat_{i,\ell} + \frac{2^{1-\ell}}{\asafehat^{\ell-1}(i)}}, a_{0,i} \bigg\}, M_i \bigg \} \ge \min \{ \max \{ \asafelow^\ell(i), a_{0,i} \}, M_i \} .
    \end{align*}
    Since $\ell < \ell_0$, by definition $\asafelow^\ell(i) < M_i$, which implies that $\max \{ \asafelow^\ell(i), a_{0,i} \} \le M_i$ (since $a_{0,i} \le M_i$ by assumption), so $\asafehat^\ell(i) \ge  \max \{ \asafelow^\ell(i), a_{0,i} \}$, proving the inductive step. We conclude that for any $\ell < \ell_0$, we can lower bound $\asafehat^\ell(i) \ge \max \{ \asafelow^\ell(i), a_{0,i} \} =\min \{ \max \{ \asafelow^\ell(i), a_{0,i} \}, M_i \}$. 
    
Now take $\ell = \ell_0$. In this case we can use that $\asafehat^{\ell_0-1}(i) \ge  \max \{ \asafelow^{\ell_0-1}(i), a_{0,i} \}$ and repeat the above calculation to get that
\begin{align*}
\frac{\gamma}{\muhat_{i,\ell_0} + \frac{2^{1-\ell_0}}{\asafehat^{\ell_0-1}(i)}} \ge \asafelow^{\ell_0}(i) .
\end{align*}
Since by definition $\asafelow^{\ell_0}(i) \ge M_i$, we conclude
\begin{align*}
\asafehat^{\ell_0}(i) & =  \min \bigg \{ \max\bigg\{\frac{\gamma}{\muhat_{i,\ell_0} + \frac{2^{1-\ell_0}}{\asafehat^{\ell_0-1}(i)}}, a_{0,i} \bigg\}, M_i \bigg \} \ge \min \{ \max \{ \asafelow^{\ell_0}(i), a_{0,i} \}, M_i \} \ge \min \{ \max \{ M_i, a_{0,i} \}, M_i \} = M_i .
\end{align*}
Since we set $\asafehat^{\ell}(i)$ to $M_i$ for all epochs after the first epoch for which $\asafehat^{\ell}(i) = M_i$, it follows that $\asafehat^{\ell}(i) = M_i$ for $\ell \ge \ell_0$. 
\end{proof}

\begin{lemma}\label{lem:asafe_linear_lb2}
We will have that
\begin{align*}
\asafelow^\ell(i) := \frac{\gamma^{\ell}2^{\ell (\ell+1)/2}}{\mu_i\sum^{\ell}_{k=1}2^{\ell-k}\gamma^{k-1}2^{k(k+1)/2} + \frac{2^{\ell}}{a_{0,i}}} \ge \frac{\gamma}{(1 + \alpha) \mu_i}
\end{align*}
as long as
\begin{align*}
\ell \ge \ellsafe_i(\alpha) :=  \max \left \{8,  \sqrt{8 \log_2 \frac{2\gamma}{\alpha}}, \sqrt{4 \log_2 \frac{2\gamma}{a_{0,i} \mu_i \alpha} }, 4 \log_2 \frac{2^{1/2}}{\gamma},    \log_2 \frac{4}{\gamma \alpha} + 2 \log_2 \left (   \log_2 \frac{4}{\gamma \alpha} \right ) \right \} .
\end{align*}
\end{lemma}
\begin{proof}
We first want to find a lower bound on $\ell$ that will guarantee
\begin{align}\label{eq:ell_asafe_lb1}
2^\ell/a_{0,i} \le \frac{\alpha}{2} \cdot \mu_i \gamma^{\ell-1}2^{\ell (\ell+1)/2}
\end{align}
since in this case we can lower bound
\begin{align*}
\frac{\gamma^{\ell}2^{\ell (\ell+1)/2}}{\mu_i\sum^{\ell}_{k=1}2^{\ell-k}\gamma^{k-1}2^{k(k+1)/2} + \frac{2^{\ell}}{a_{0,i}}} & \ge \frac{\gamma^{\ell}2^{\ell (\ell+1)/2}}{\mu_i\sum^{\ell}_{k=1}2^{\ell-k}\gamma^{k-1}2^{k(k+1)/2} +  \frac{\alpha}{2} \cdot \mu_i \gamma^{\ell-1}2^{\ell (\ell+1)/2}} \\
& = \frac{\gamma}{\mu_i} \cdot \frac{1}{ \sum^{\ell}_{k=1} \frac{2^{\ell-k}\gamma^{k-1}2^{k(k+1)/2}}{\gamma^{\ell-1}2^{\ell (\ell+1)/2}} + \alpha/2} .
\end{align*}
To show \eqref{eq:ell_asafe_lb1}, we can rearrange it as
\begin{align*}
( 2^{1/2} / \gamma)^\ell \le \frac{a_{0,i} \mu_i \alpha}{2 \gamma} 2^{\ell^2/2}.
\end{align*}
Note that $( 2^{1/2} / \gamma)^\ell \le 2^{\ell^2/4}$ for
\begin{align*}
\ell \ge 4 \log_2 (2^{1/2}/\gamma).
\end{align*}
So for $\ell$ satisfying this, we will have that $( 2^{1/2} / \gamma)^\ell \le \frac{a_{0,i} \mu_i \alpha}{2 \gamma} 2^{\ell^2/2}$ as long as
\begin{align*}
2^{\ell^2/4} \le \frac{a_{0,i} \mu_i \alpha}{2 \gamma} 2^{\ell^2/2} \implies \frac{2\gamma}{a_{0,i} \mu_i \alpha} \le 2^{\ell^2/4} \implies \sqrt{4 \log_2 \frac{2\gamma}{a_{0,i} \mu_i \alpha} } \le \ell
\end{align*}
so a sufficient condition to meet \eqref{eq:ell_asafe_lb1} is 
\begin{align*}
\ell \ge \max \left \{  \sqrt{4 \log_2 \frac{2\gamma}{a_{0,i} \mu_i \alpha} }, 4 \log_2 \frac{2^{1/2}}{\gamma} \right \} . 
\end{align*}

Our goal now is to find $\ell$ such that
\begin{align*}
1 + \frac{\alpha}{2} \ge \sum^{\ell}_{k=1} \frac{2^{\ell-k}\gamma^{k-1}2^{k(k+1)/2}}{\gamma^{\ell-1}2^{\ell (\ell+1)/2}} = 1 + \sum^{\ell-1}_{k=1} \frac{2^{\ell-k}\gamma^{k-1}2^{k(k+1)/2}}{\gamma^{\ell-1}2^{\ell (\ell+1)/2}}.
\end{align*}
If we can find such an $\ell$, the desired conclusion will follow. 
Note that
    \begin{equation} \label{eq:first_last_sum_ineq2}
        2^{\ell-k}\gamma^{k-1}2^{k(k+1)/2} \leq \max\{2^{\ell},2 \gamma^{\ell-2}2^{\frac{(\ell-1)\ell}{2}}\}, \quad \forall k \leq \ell-1.
    \end{equation}
To see this, consider 
\begin{align*}
\log_2 ( 2^{\ell-k}\gamma^{k-1}2^{k(k+1)/2} ) = \ell - k + (k-1) \log_2 \gamma + k^2/2 + k/2.
\end{align*}
Taking the derivative of this with respect to $k$ we get
\begin{align*}
-1/2 + \log_2 \gamma + k . 
\end{align*}
As this is negative for $k < 1/2 - \log_2 \gamma$ and otherwise positive, it follows that the maximum of $2^{\ell-k}\gamma^{k-1}2^{k(k+1)/2} $ over the range $k \in [1,\ell-1]$ must either occur at $k = 1$, or $k = \ell - 1$.

It follows that
\begin{align*}
\sum^{\ell-1}_{k=1} \frac{2^{\ell-k}\gamma^{k-1}2^{k(k+1)/2}}{\gamma^{\ell-1}2^{\ell (\ell+1)/2}} \le \frac{(\ell - 1)  \max\{2^{\ell},2 \gamma^{\ell-2}2^{\frac{(\ell-1)\ell}{2}}\}}{\gamma^{\ell-1}2^{\ell (\ell+1)/2}} = (\ell - 1) \max \{ \gamma^{1-\ell} 2^{ \ell/2 - \ell^2/2}, \gamma^{-1} 2^{1-\ell} \}.
\end{align*}
Now
\begin{align*}
(\ell - 1) \gamma^{-1} 2^{1-\ell} \le \alpha / 2 \iff \log_2(\ell - 1) - \log_2 \frac{\gamma \alpha}{4} \le \ell .
\end{align*}
Assume that $ - \log_2 \frac{\gamma \alpha}{4}  > 0$, as otherwise we are trivially done. We claim that the above inequality is satisfied as long as
\begin{align*}
\ell \ge  \max \{ - \log_2 \frac{\gamma \alpha}{4}, 3 \} + 2 \log \left (  \max \{ - \log_2 \frac{\gamma \alpha}{4}, 3 \}  \right ).
\end{align*}
To see this, note that with $\ell =  \max \{ - \log_2 \frac{\gamma \alpha}{4}, 3 \} + 2 \log \left (  \max \{ - \log_2 \frac{\gamma \alpha}{4}, 3 \} \right ) $, we have
\begin{align*}
\log_2(\ell - 1) - \log_2 \frac{\gamma \alpha}{4} & = \log_2 \left (  \max \{ - \log_2 \frac{\gamma \alpha}{4}, 3 \} + 2 \log \left (  \max \{ - \log_2 \frac{\gamma \alpha}{4}, 3 \} \right ) - 1\right ) - \log_2 \frac{\gamma \alpha}{4} \\
& \le \log_2 \left ( 3  \max \{ - \log_2 \frac{\gamma \alpha}{4}, 3 \} \right ) - \log_2 \frac{\gamma \alpha}{4} \\
& \le  \max \{ - \log_2 \frac{\gamma \alpha}{4}, 3 \} + 2 \log \left (  \max \{ - \log_2 \frac{\gamma \alpha}{4}, 3 \}  \right ) \\
& = \ell
\end{align*}
where the last inequality follows since
\begin{align*}
\log_2 \left ( 3  \max \{ - \log_2 \frac{\gamma \alpha}{4}, 3 \} \right ) & = \log_2 3 + \log_2  \left (   \max \{ - \log_2 \frac{\gamma \alpha}{4}, 3 \} \right )  \\
& \le 2 \log_2  \left (   \max \{ - \log_2 \frac{\gamma \alpha}{4}, 3 \} \right ) .
\end{align*}
To guarantee that
\begin{align*}
(\ell - 1) \gamma^{1-\ell} 2^{\ell/2 - \ell^2/2} \le \alpha/2 \iff \log_2 (\ell - 1) + \ell \log_2 (2^{1/2}/\gamma) - \log_2 \frac{\alpha}{2\gamma} \le \ell^2/2 
\end{align*}
it suffices to take
\begin{align*}
\ell \ge \max \{ 8, \sqrt{-8 \log_2 \frac{\alpha}{2\gamma}}, 4 \log_2 \frac{2^{1/2}}{\gamma} \} .
\end{align*}
Putting all of this together, we have shown that
\begin{align*}
\frac{\gamma^{\ell}2^{\ell (\ell+1)/2}}{\mu_i\sum^{\ell}_{k=1}2^{\ell-k}\gamma^{k-1}2^{k(k+1)/2} + \frac{2^{\ell}}{a_{0,i}}} \ge \frac{\gamma}{(1+\alpha)\mu_i}
\end{align*}
as long as
\begin{align*}
\ell \ge \max \left \{8,  \sqrt{-8 \log \frac{\alpha}{2\gamma}}, \sqrt{4 \log_2 \frac{2\gamma}{a_{0,i} \mu_i \alpha} }, 4 \log_2 \frac{2^{1/2}}{\gamma},  \max \{ - \log_2 \frac{\gamma \alpha}{4}, 3 \} + 2 \log \left (  \max \{ - \log_2 \frac{\gamma \alpha}{4}, 3 \}  \right ) \right \} .
\end{align*}
Note that if $\max \{ - \log_2 \frac{\gamma \alpha}{4}, 3 \} = 3$, then the condition 
\begin{align*}
\ell \ge  \max \{ - \log_2 \frac{\gamma \alpha}{4}, 3 \} + 2 \log \left (  \max \{ - \log_2 \frac{\gamma \alpha}{4}, 3 \}  \right )
\end{align*}
is implied by $\ell \ge 8$. The stated conclusion follows.
\end{proof}

\subsection{Sample Complexity of \alglinear}
We will define the following:
\begin{align*}
\ellsolvea_i(x) & := \log_2 \frac{4(2 + \frac{\theta_{i^*}}{\mu_{i^*}} +\frac{\theta_{i}}{\mu_{i}})}{x} \\
\ellsolveb_i(x) & := \max \left \{ \log_2 \frac{4(2 + M_i \theta_i/\gamma)}{x}, \log_2 \frac{4}{\gamma}  \right \} \\
\ellsolvec_i(x) & := \log_2 \frac{4}{x} \\
\alpha_i & := \frac{\gamma}{\mu_i M_i} - 1
\end{align*}

\begin{lemma}\label{lem:i_termination_bound}
On $\calE$, when running \Cref{alg:constrained_bai}:
\begin{enumerate}
\item If $M_i \not\in \safe_i$ and $M_{\ist} \not\in \safe_{\ist}$, then we will have eliminated coordinate $i$ once
\begin{align*}
\ell \ge \max \Big \{ \ellsolvea_i(\Delta_i), \ellsafe_i(1), \ellsafe_{\ist}(1) \Big \} .
\end{align*}
\item If $M_i \in \safe_i$ and $M_{\ist} \not\in \safe_{\ist}$, then we will have eliminated coordinate $i$ once
\begin{align*}
\ell \ge \max \Big \{ \ellsolveb_i(\Delta_i), \ellsafe_i(\alpha_i), \ellsafe_{\ist}(1) \Big \} \quad \text{or} \quad \ell \ge \max \left \{ \ellsolvea_i \left (  \frac{\gamma \theta_{i^*}}{\mu_{i^*}} -\frac{\gamma \theta_{i}}{\mu_{i}} \right ), \ellsafe_i(1), \ellsafe_{\ist}(1) \right \} .
\end{align*}
\item If $M_i \not\in \safe_i$ and $M_{\ist} \in \safe_{\ist}$, then we will have eliminated coordinate $i$ once
\begin{align*}
\ell \ge \max \Big \{ \ellsolveb_{\ist}(\Delta_i), \ellsafe_i(1), \ellsafe_{\ist}(\alpha_{\ist}) \Big \} \quad \text{or} \quad \ell \ge \max \left \{ \ellsolvea_i \left (  \frac{\gamma \theta_{i^*}}{\mu_{i^*}} -\frac{\gamma \theta_{i}}{\mu_{i}} \right ), \ellsafe_i(1), \ellsafe_{\ist}(1) \right \} .
\end{align*}
\item If $M_i \in \safe_i$ and $M_{\ist} \in \safe_{\ist}$, then we will have eliminated coordinate $i$ once any of the following conditions is met
\begin{align*}
& \ell \ge \max \Big \{ \ellsolvec_i(\Delta_i), \ellsafe_i(\alpha_i), \ellsafe_{\ist}(\alpha_{\ist}) \Big \}, \quad \ell \ge \max \left \{ \ellsolvea_i \left (  \frac{\gamma \theta_{i^*}}{\mu_{i^*}} -\frac{\gamma \theta_{i}}{\mu_{i}} \right ), \ellsafe_i(1), \ellsafe_{\ist}(1) \right \} \\
& \ell \ge \max \left \{ \ellsolveb_i \left (  \frac{\gamma \theta_{\ist}}{\mu_{\ist}} - M_i \theta_i \right ), \ellsafe_i(\alpha_i), \ellsafe_{\ist}(1) \right \} , \quad \ell \ge \max \left \{ \ellsolveb_{\ist} \left ( M_{\ist} \theta_{\ist} - \frac{\gamma \theta_i}{\mu_i} \right ), \ellsafe_i(1), \ellsafe_{\ist}(\alpha_{\ist}) \right \} .
\end{align*}
\end{enumerate}
\end{lemma}
\begin{proof}
We prove each case individually. 

\paragraph{Case 1 ($M_i \not\in \safe_i$ and $M_{\ist} \not\in \safe_{\ist}$).}
By \Cref{lem:asafe_lb}, on $\calE$, 
\begin{align*}
\asafehat^{\ell}(i) \geq \max \bigg\{ \frac{\gamma^{\ell}2^{\ell(\ell+1)/2}}{\mu_i\sum^{\ell}_{k=1}2^{\ell-k}\gamma^{k-1}2^{k(k+1)/2} + \frac{2^{\ell}}{a_{0,i}}}, \ a_{0,i} \bigg \} 
\end{align*}
since by \Cref{lem:asafehat_safe} we know that $\asafehat^\ell(i) \in \safe_i$, and $M_i \not\in \safe_i$ by assumption, so we are in the first case given in \Cref{lem:asafe_lb}. By \Cref{lem:asafe_linear_lb2}, we can therefore lower bound
\begin{align*}
\asafehat^{\ell}(i) \ge \frac{\gamma^{\ell}2^{\ell (\ell+1)/2}}{\mu_i\sum^{\ell}_{k=1}2^{\ell-k}\gamma^{k-1}2^{k(k+1)/2} + \frac{2^{\ell}}{a_{0,i}}} \ge \frac{\gamma}{2 \mu_i}
\end{align*}
as long as $\ell \ge \ellsafe_i(1)$, 
and similarly
\begin{align*}
\asafehat^{\ell}(\ist) \ge \frac{\gamma}{2 \mu_{\ist}} 
\end{align*}
as long as $\ell \ge \ellsafe_{\ist}(1)$.

By \Cref{lem:l_ineq}, we will then have eliminated coordinate $i$ once
\begin{align*}
 \frac{\theta_i \frac{\gamma}{2 \mu_i} + 2^{1-\ell}}{\mu_i \frac{\gamma}{2 \mu_i} - 2^{1-\ell}} \leq \frac{\theta_{i^*} \frac{\gamma}{2 \mu_{\ist}}  - 2^{1-\ell}}{\mu_{i^*} \frac{\gamma}{2 \mu_{\ist}}  + 2^{1-\ell}} & \iff 2\cdot2^{1-\ell}(2 + \frac{\theta_{i^*}}{\mu_{i^*}} +\frac{\theta_{i}}{\mu_{i}}) \leq \gamma (\frac{\theta_{i^*}}{\mu_{i^*}} -\frac{\theta_{i}}{\mu_{i}}) \\
& \iff 2^{-\ell} \leq \frac{\Delta_i}{4(2 + \frac{\theta_{i^*}}{\mu_{i^*}} +\frac{\theta_{i}}{\mu_{i}})} 
\end{align*}
where the last expression follows since we have assumed $M_i \not\in \safe_i$ and $M_{\ist} \not\in \safe_{\ist}$, so $\Delta_i = \gamma (\frac{\theta_{i^*}}{\mu_{i^*}} -\frac{\theta_{i}}{\mu_{i}})$.

\paragraph{Case 2 ($M_i \in \safe_i$ and $M_{\ist} \not\in \safe_{\ist}$).}
Note that $M_i = \frac{\gamma}{(1+\alpha_i) \mu_i}$ for $\alpha_i = \frac{\gamma}{\mu_i M_i} - 1$. Since $M_i \in \safe_i$ by assumption, we will have $\alpha_i \ge 0$. By \Cref{lem:asafe_lb} and \Cref{lem:asafe_linear_lb2}, we will have that $\asafehat^\ell(i) = M_i$ once $\ell \ge \ellsafe_i(\alpha_i)$
Also assume that $\ell \ge \ellsafe_{\ist}(1)$ so that $\asafehat^\ell(\ist) \ge \gamma / (2 \mu_{\ist})$. We can then apply \Cref{lem:sat_ineq_i} to get that we will eliminate coordinate $i$ once:
\begin{align*}
M_{i}\theta_{i} + 2^{1-\ell} < \gamma \cdot\frac{\theta_{i^*}\frac{\gamma}{2\mu_{\ist}} - 2^{1-\ell}}{\mu_{i^*}\frac{\gamma}{2\mu_{\ist}}  + 2^{1-\ell}} & \iff 2^{2-\ell} ( \frac{3}{2} + \frac{M_i \theta_i}{\gamma} + \frac{2^{1-\ell}}{\gamma} ) <  (\gamma \frac{\theta_{\ist}}{\mu_{\ist}} - M_i \theta_i) = \Delta_i
\end{align*}
A sufficient condition to meet this is that
\begin{align*}
\ell \ge \max \left \{ \log_2 \frac{4(2 + M_i \theta_i/\gamma)}{\Delta_i}, \log_2 \frac{4}{\gamma}  \right \} .
\end{align*}
While we are guaranteed to terminate for this choice of $\ell$, it is possible that we will also terminate and eliminate coordinate $i$ for an $\ell$ value where we have not yet guaranteed that $M_i$ is safe. In that case we can revert to the first case, with the only difference being that $\gamma ( \frac{\theta_{i^*}}{\mu_{i^*}} -\frac{\theta_{i}}{\mu_{i}}) \neq \Delta_i$, so we replace the $\Delta_i$ in the complexity derived in the first case by $\gamma ( \frac{\theta_{i^*}}{\mu_{i^*}} -\frac{\theta_{i}}{\mu_{i}} )$. Nonetheless, our sufficient condition above still works.

\paragraph{Case 3 ($M_i \not\in \safe_i$ and $M_{\ist} \in \safe_{\ist}$).} 
This case is nearly identical to the previous case. To guarantee that $\asafehat^\ell(\ist) = M_{\ist}$, we can take $\ell \ge \ellsafe_{\ist}(\alpha_{\ist})$,
where $\alpha_{\ist}$ is defined as in the previous case. Assume also that $\ell \ge \ellsafe_i(1)$ so that we can guarantee $\asafehat^\ell(i) \ge \gamma/(2 \mu_i)$. By \Cref{lem:sat_ineq_i_star} we will then eliminate coordinate $i$ once
\begin{align*}
        \gamma \cdot\frac{\theta_i \frac{\gamma}{2\mu_i} + 2^{1-\ell}}{\mu_i \frac{\gamma}{2\mu_i} - 2^{1-\ell}} < M_{i^*}\theta_{i^*} - 2^{1-\ell} \iff 2^{2 - \ell} ( \frac{3}{2} + \frac{M_{\ist} \theta_{\ist}}{\gamma} - \frac{2^{1 - \ell}}{\gamma}) < M_{\ist} \theta_{\ist} - \gamma \frac{\theta_i}{\mu_i} = \Delta_i
\end{align*}
and a sufficient condition to meet this is
\begin{align*}
\ell \ge \log_2 \frac{4(3/2 + M_{\ist} \theta_{\ist}/\gamma)}{\Delta_i} .
\end{align*}
As in the previous case, we may reach the termination criteria before identifying that $M_{\ist}$ is safe, so we can also derive a sufficient condition analogous to case 1. This time, we have $\Delta_i < \gamma ( \frac{\theta_{i^*}}{\mu_{i^*}} -\frac{\theta_{i}}{\mu_{i}} )$, so the case 1 condition holds naturally without replacing $\Delta_i$.

\paragraph{Case 4 ($M_i \in \safe_i$ and $M_{\ist} \in \safe_{\ist}$).}
In this case, if $\ell$ is large enough that $\ell \ge \ellsafe_{\ist}(\alpha_{\ist})$ and $\ell \ge \ellsafe_i(\alpha_i)$ hold, on $\calE$, a sufficient condition to eliminate $i$ is that
\begin{align*}
M_i \theta_i + 2^{1-\ell} < M_{\ist} \theta_{\ist} - 2^{1-\ell} \iff 2^{-\ell} < 4 (M_{\ist} \theta_{\ist} - M_i \theta_i) = \Delta_i .
\end{align*}
The result then follows from this and by noting that we might terminate earlier than this by meeting any of the conditions in Case 1-3.
\end{proof}

\begin{lemma}\label{thm:linear_upper_bound}
With probability at least $1-\delta$, \Cref{alg:constrained_bai} will output $\ist$, only pull safe arms, and terminate after collecting at most
\begin{align*}
8 \sigma^2 \sum_{i \neq \ist} \log \frac{8 d \calL_i^2}{\delta} \cdot 2^{2 \calL_i} + 2 \sum_{i \neq \ist} \calL_i 
\end{align*}
samples, where $\calL_i$ denotes the minimum feasible $\ell$ value given in \Cref{lem:i_termination_bound} for the case coordinate $i$ falls in.
\end{lemma}
\begin{proof}[Proof of \Cref{thm:linear_upper_bound}]
\Cref{lem:hoeff_bound} implies that $\bbP[\calE] \ge 1 - \delta$. We will assume that $\calE$ holds for the remainder of the proof. 

By \Cref{lem:asafe_val}, on $\calE$, $\ist \in \calX_\ell$ for all $\ell$. Since \Cref{alg:constrained_bai} will only terminate when $| \calX_{\ell} | = 1$, we will return $\ist$. By \Cref{lem:asafehat_safe}, we will only pull safe arms on $\calE$.

We can then bound the number of samples taken by arm $i$ using \Cref{lem:i_termination_bound}. Every $i$ with $\ist$ will fall into one of the four cases given in \Cref{lem:i_termination_bound}. We let $\calL_i$ denote the minimum sufficient $\ell$ value in the relevant case for coordinate $i$. Then we can bound the sample complexity as (since we will pull $\ist$ as many times as the next best arm)
\begin{align*}
2 \sum_{i \neq \ist} \sum_{\ell = 1}^{\calL_i} (2 \sigma^2 \log \frac{8 d \ell^2}{\delta} 2^{2 \ell} + 1) & \le 8 \sigma^2 \sum_{i \neq \ist} \log \frac{8 d \calL_i^2}{\delta} \cdot 2^{2 \calL_i} + 2 \sum_{i \neq \ist} \calL_i .
\end{align*} 
\end{proof}

\begin{lemma}\label{lem:ellsafe_simplified}
We can upper bound
\begin{align*}
2^{2 \ellsafe_i(\alpha)} \le \cOtil \left ( \max \left \{ 2^{16}, \xi_{\sqrt{32}}(2\gamma) \xi_{\sqrt{32}}(1/\alpha),  \xi_4(\frac{1}{a_{0,i} \mu_i}) \xi_4(2\gamma) \xi_4(1/\alpha), \frac{16}{\gamma^8}, \frac{16}{\gamma^2 \alpha^2} \right \} \right ). 
\end{align*}
\end{lemma}
\begin{proof}
Recall that
\begin{align*}
\ellsafe_i(\alpha) =  \max \left \{8,  \sqrt{8 \log_2 \frac{2\gamma}{\alpha}}, \sqrt{4 \log_2 \frac{2\gamma}{a_{0,i} \mu_i \alpha} }, 4 \log_2 \frac{2^{1/2}}{\gamma},    \log_2 \frac{4}{\gamma \alpha} + 2 \log_2 \left (   \log_2 \frac{4}{\gamma \alpha} \right ) \right \} .
\end{align*}
Using that $\sqrt{a+b} \le \sqrt{a} + \sqrt{b}$, we can then upper bound
\begin{align*}
& 2^{2 \sqrt{8 \log_2 \frac{2\gamma}{\alpha}}} \le 2^{2 \sqrt{8 \log_2 2\gamma}} 2^{2 \sqrt{8 \log_2 1/\alpha}} \le \xi_{\sqrt{32}}(2\gamma) \xi_{\sqrt{32}}(1/\alpha) \\
& 2^{2 \sqrt{4 \log_2 \frac{2\gamma}{a_{0,i} \mu_i \alpha} }} \le 2^{4 \sqrt{\log_2 \frac{1}{a_{0,i} \mu_i}}} 2^{4 \sqrt{\log_2 (2\gamma)}} 2^{4 \sqrt{\log_2(1/\alpha)}} \le \xi_4(\frac{1}{a_{0,i} \mu_i}) \xi_4(2\gamma) \xi_4(1/\alpha) \\
& 2^{2 \cdot 4 \log_2 \frac{2^{3/2}}{\gamma}} = \frac{16}{\gamma^8} \\
& 2^{2 ( \log_2 \frac{4}{\gamma \alpha} + 2 \log_2 \left (   \log_2 \frac{4}{\gamma \alpha} \right ) )} \le \cOtil ( \frac{16}{\gamma^2 \alpha^2} ) .
\end{align*}
The result follows from these.
\end{proof}

\begin{theorem}[Full version of \Cref{thm:linear_complexity2_simp}]\label{thm:linear_complexity2}
We define the following values:
\begin{enumerate}
\item Case 1 ($M_i \not\in \safe_i$ and $M_{\ist} \not\in \safe_{\ist}$): 
\begin{align*}
N_{1,i} :=   \frac{1 + \theta_{\ist}^2/\mu_{\ist}^2 + \theta_i^2/\mu_i^2}{\Delta_i^2} +  \max \left \{ \frac{1}{\gamma^8}, \xi_{\sqrt{32}}(2\gamma), \xi_4(\frac{1}{a_{0,i} \mu_i}) \xi_4(2\gamma), \xi_4(\frac{1}{a_{0,\ist} \mu_{\ist}}) \xi_4(2\gamma) \right \}  .
\end{align*}
\item Case 2 ($M_i \in \safe_i$ and $M_{\ist} \not\in \safe_{\ist}$):
\begin{align*}
N_{2,i} & :=  \frac{1 + M_i^2 \theta_i^2 /\gamma^2}{\Delta_i^2} + \max \bigg \{  \frac{1}{\gamma^8}, \xi_{\sqrt{32}}(2\gamma), \xi_{\sqrt{32}}(2\gamma) \xi_{\sqrt{32}} ( \frac{M_i \mu_i}{\gamma - M_i \mu_i} ), \\
& \qquad  \xi_4(\frac{1}{a_{0,i} \mu_i}) \xi_4(2\gamma) \xi_4 ( \frac{M_i \mu_i}{\gamma - M_i \mu_i} ), \xi_4(\frac{1}{a_{0,\ist} \mu_{\ist}}) \xi_4(2\gamma), \frac{ M_i^2 \mu_i^2}{\gamma^2 (\gamma - M_i \mu_i)^2} \bigg \} .
\end{align*}
\item Case 3  ($M_i \not\in \safe_i$ and $M_{\ist} \in \safe_{\ist}$):
\begin{align*}
N_{3,i} & :=  \frac{1 + M_{\ist}^2 \theta_{\ist}^2 /\gamma^2}{\Delta_i^2} + \max \bigg \{ \frac{1}{\gamma^8}, \xi_{\sqrt{32}}(2\gamma), \xi_{\sqrt{32}}(2\gamma) \xi_{\sqrt{32}} ( \frac{M_{\ist} \mu_{\ist}}{\gamma - M_{\ist} \mu_{\ist}} ), \\
& \qquad  \xi_4(\frac{1}{a_{0,\ist} \mu_{\ist}}) \xi_4(2\gamma) \xi_4 ( \frac{M_{\ist} \mu_{\ist}}{\gamma - M_{\ist} \mu_{\ist}} ), \xi_4(\frac{1}{a_{0,i} \mu_{i}}) \xi_4(2\gamma), \frac{16 M_{\ist}^2 \mu_{\ist}^2}{\gamma^2 (\gamma - M_{\ist} \mu_{\ist})^2} \bigg \}  .
\end{align*}
\item Case 4 ($M_i \in \safe_i$ and $M_{\ist} \in \safe_{\ist}$):
\begin{align*}
N_{4,i} & := \frac{1}{\Delta_i^2} + \max \bigg \{ \frac{1}{\gamma^8}, \xi_{\sqrt{32}}(2\gamma) \xi_{\sqrt{32}} ( \frac{M_{i} \mu_{i}}{\gamma - M_{i} \mu_{i}} ), \xi_{\sqrt{32}}(2\gamma) \xi_{\sqrt{32}} ( \frac{M_{\ist} \mu_{\ist}}{\gamma - M_{\ist} \mu_{\ist}} ), \\
& \qquad \qquad \xi_4(\frac{1}{a_{0,i} \mu_i}) \xi_4(2\gamma) \xi_4 ( \frac{M_{i} \mu_{i}}{\gamma - M_{i} \mu_{i}} ), \xi_4(\frac{1}{a_{0,\ist} \mu_{\ist}}) \xi_4(2\gamma) \xi_4 ( \frac{M_{\ist} \mu_{\ist}}{\gamma - M_{\ist} \mu_{\ist}} ), \\
& \qquad \qquad \frac{ M_{\ist}^2 \mu_{\ist}^2}{\gamma^2 (\gamma - M_{\ist} \mu_{\ist})^2}, \frac{ M_{i}^2 \mu_{i}^2}{\gamma^2 (\gamma - M_{i} \mu_{i})^2}  \bigg \}  .
\end{align*}
\end{enumerate}
Let $\case(i)$ denote the case coordinate $i$ falls in. Then, with probability at least $1-\delta$, \Cref{alg:constrained_bai} will output $\ist$, only pull safe arms, and terminate after collecting at most
\begin{align*}
\cOtil \bigg ( \log \frac{d}{\delta} \cdot \sum_{i \neq \ist} N_{\case(i),i} \bigg )
\end{align*}
samples.
\end{theorem}
\begin{proof}
We apply \Cref{lem:ellsafe_simplified} to simplify each of the cases. While for cases 2-4 there exist multiple termination criteria, as we are concerned with interpretability, we consider only the first. 

\paragraph{Case 1 ($M_i \not\in \safe_i$ and $M_{\ist} \not\in \safe_{\ist}$).}
Here we need $\ell \ge \max \{ \ellsolvea_i(\Delta_i), \ellsafe_i(1), \ellsafe_{\ist}(a) \}$. By \Cref{lem:ellsafe_simplified}, and noting that $\xi_a(1) = 1$, we can upper bound
\begin{align*}
& \max \{2^{2\ellsolvea_i(\Delta_i)}, 2^{2 \ellsafe_i(1)}, 2^{2 \ellsafe_{\ist}(1)} \} \\
& \qquad \le \cOtil \left ( \max \left \{ \frac{1 + \theta_{\ist}^2/\mu_{\ist}^2 + \theta_i^2/\mu_i^2}{\Delta_i^2}, 2^{16}, \frac{16}{\gamma^8}, \xi_{\sqrt{32}}(2\gamma), \xi_4(\frac{1}{a_{0,i} \mu_i}) \xi_4(2\gamma), \xi_4(\frac{1}{a_{0,\ist} \mu_{\ist}}) \xi_4(2\gamma) \right \} \right ) .
\end{align*}

\paragraph{Case 2 ($M_i \in \safe_i$ and $M_{\ist} \not\in \safe_{\ist}$).}
We now need $\ell \ge \max \Big \{ \ellsolveb_i(\Delta_i), \ellsafe_i(\alpha_i), \ellsafe_{\ist}(1) \Big \}$. Applying the same simplifications, we can bound
\begin{align*}
& \max \{2^{2\ellsolveb_i(\Delta_i)}, 2^{2 \ellsafe_i(\alpha_i)}, 2^{2 \ellsafe_{\ist}(1)} \} \\
& \qquad \le \cOtil \bigg ( \max \bigg \{ \frac{1 + M_i^2 \theta_i^2 /\gamma^2}{\Delta_i^2}, 2^{16}, \frac{16}{\gamma^8}, \xi_{\sqrt{32}}(2\gamma), \xi_{\sqrt{32}}(2\gamma) \xi_{\sqrt{32}} ( \frac{M_i \mu_i}{\gamma - M_i \mu_i} ), \\
& \qquad \qquad \xi_4(\frac{1}{a_{0,i} \mu_i}) \xi_4(2\gamma) \xi_4 ( \frac{M_i \mu_i}{\gamma - M_i \mu_i} ), \xi_4(\frac{1}{a_{0,\ist} \mu_{\ist}}) \xi_4(2\gamma), \frac{16 M_i^2 \mu_i^2}{\gamma^2 (\gamma - M_i \mu_i)^2} \bigg \} \bigg ) .
\end{align*}

\paragraph{Case 3 ($M_i \not\in \safe_i$ and $M_{\ist} \in \safe_{\ist}$).}
We now need $\ell \ge \max \{ \ellsolveb_{\ist}(\Delta_i), \ellsafe_i(1), \ellsafe_{\ist}(\alpha_{\ist}) \}$. We can bound
\begin{align*}
& \max \{2^{2\ellsolveb_{\ist}(\Delta_i)}, 2^{2 \ellsafe_i(1)}, 2^{2 \ellsafe_{\ist}(\alpha_{\ist})} \} \\
& \qquad \le \cOtil \bigg ( \max \bigg \{ \frac{1 + M_{\ist}^2 \theta_{\ist}^2 /\gamma^2}{\Delta_i^2}, 2^{16}, \frac{16}{\gamma^8}, \xi_{\sqrt{32}}(2\gamma), \xi_{\sqrt{32}}(2\gamma) \xi_{\sqrt{32}} ( \frac{M_{\ist} \mu_{\ist}}{\gamma - M_{\ist} \mu_{\ist}} ), \\
& \qquad \qquad \xi_4(\frac{1}{a_{0,\ist} \mu_{\ist}}) \xi_4(2\gamma) \xi_4 ( \frac{M_{\ist} \mu_{\ist}}{\gamma - M_{\ist} \mu_{\ist}} ), \xi_4(\frac{1}{a_{0,i} \mu_{i}}) \xi_4(2\gamma), \frac{16 M_{\ist}^2 \mu_{\ist}^2}{\gamma^2 (\gamma - M_{\ist} \mu_{\ist})^2} \bigg \} \bigg ) .
\end{align*}

\paragraph{Case 4 ($M_i \in \safe_i$ and $M_{\ist} \in \safe_{\ist}$).}
In the final case, we need $\ell \ge \max \Big \{ \ellsolvec_i(\Delta_i), \ellsafe_i(\alpha_i), \ellsafe_{\ist}(\alpha_{\ist}) \Big \}$. We can then upper bound
\begin{align*}
& \max \{2^{2\ellsolvec_{\ist}(\Delta_i)}, 2^{2 \ellsafe_i(\alpha_i)}, 2^{2 \ellsafe_{\ist}(\alpha_{\ist})} \} \\
& \qquad \le \cOtil \bigg ( \max \bigg \{ \frac{1 }{\Delta_i^2}, 2^{16}, \frac{16}{\gamma^8}, \xi_{\sqrt{32}}(2\gamma) \xi_{\sqrt{32}} ( \frac{M_{i} \mu_{i}}{\gamma - M_{i} \mu_{i}} ), \xi_{\sqrt{32}}(2\gamma) \xi_{\sqrt{32}} ( \frac{M_{\ist} \mu_{\ist}}{\gamma - M_{\ist} \mu_{\ist}} ), \\
& \qquad \qquad \xi_4(\frac{1}{a_{0,i} \mu_i}) \xi_4(2\gamma) \xi_4 ( \frac{M_{i} \mu_{i}}{\gamma - M_{i} \mu_{i}} ), \xi_4(\frac{1}{a_{0,\ist} \mu_{\ist}}) \xi_4(2\gamma) \xi_4 ( \frac{M_{\ist} \mu_{\ist}}{\gamma - M_{\ist} \mu_{\ist}} ), \\
& \qquad \qquad \frac{16 M_{\ist}^2 \mu_{\ist}^2}{\gamma^2 (\gamma - M_{\ist} \mu_{\ist})^2}, \frac{16 M_{i}^2 \mu_{i}^2}{\gamma^2 (\gamma - M_{i} \mu_{i})^2}  \bigg \} \bigg ) .
\end{align*}

\paragraph{Obtaining a final complexity.}
To obtain the final result, we apply \Cref{thm:linear_upper_bound} with the upper bounds given here. For the simplified case presented in the main text, we note that $\xi_{a}(\gamma)$ can be thought of as an absolute constant for $a$ and $\gamma$ not too large.
\end{proof}

\newcommand{\xtil}{\widetilde{x}}
\begin{proof}[Proof of \Cref{prop:exp_sqrt_fun}]
Let $\xtil := \max \{ 2 , x \}$. Note that, for $z > 0$,
\begin{align*}
2^{a \sqrt{\log_2 \xtil}}  \le \xtil^z \iff a \sqrt{ \log_2 \xtil} \le z \log_2 \xtil \iff \frac{a}{z} \le \sqrt{ \log_2 \xtil} \iff 2^{a^2/z^2} \le \xtil .
\end{align*}
As $2^{a \sqrt{\log_2 \xtil}} $ is increasing in $\xtil$, we then have
\begin{align*}
2^{a \sqrt{\log_2 \xtil}} \le \left \{ \begin{matrix} 2^{a \sqrt{\log_2 \xtil}} & \xtil < 2^{a^2/z^2} \\
\xtil^z & 2^{a^2/z^2} \le \xtil  \end{matrix} \right . \le  \left \{ \begin{matrix} 2^{a \sqrt{\log_2 2^{a^2/z^2}}} & \xtil < 2^{a^2/z^2} \\
\xtil^z & 2^{a^2/z^2} \le \xtil  \end{matrix} \right . \le \xtil^z + 2^{a \sqrt{\log_2 2^{a^2/z^2}}} = \xtil^z + 2^{a^2/z} . 
\end{align*}
As this holds for all $z > 0$, we can take the minimum over $z$. 
\end{proof}

%!TEX root = ../BAI_Constraint.tex

\section{Lower Bound Proof}
Throughout this section we assume that $\sigma^2 = 1$.

Let $\nu$ denote a distribution with
\begin{align*}
\nu_i = \calN \left ( \begin{bmatrix} a_i \theta_i \\ a_i \mu_i \end{bmatrix}, I \right ) 
\end{align*}
for some $\theta_i > 0, \mu_i > 0$, and let
\begin{align*}
\ist = \argmax_i \theta_i / \mu_i .
\end{align*}
We assume that $\ist$ is the unique maximum of $\theta_i/\mu_i$.

\begin{lemma}\label{lem:lb_change_measure}
Let $\tau$ denote the stopping time for any $\delta$-PAC algorithm. Then,
\begin{align*}
\Exp_{\nu}[\tau] \ge \sum_{i \neq \ist}  \frac{ \frac{2}{3 a_i^2} (\frac{1}{\mu_i^2} + (\frac{\theta_{\ist}}{\mu_i \mu_{\ist}})^2 + \frac{\theta_i^2}{\mu_i^4})}{( \frac{\theta_{\ist}}{\mu_{\ist}} - \frac{\theta_i}{\mu_i}  )^2} \cdot \log \frac{1}{2.4\delta} .
\end{align*}
\end{lemma}
\begin{proof}
We rely on the change-of-measure technique of \cite{kaufmann2016complexity}. Note that while \cite{kaufmann2016complexity} assumes the distributions are 1-dimensional, their proof goes through identically for 2-dimensional distributions. In particular, their Lemma 1 holds as
\begin{align*}
\sum_{i=1}^d \Exp_{\nu}[N_i(\tau)] \kl(\nu_i, \nu_i') \ge \sup_{\calE \in \calF_\tau} d (\bbP_\nu(\calE), \bbP_{\nu'}(\calE)) .
\end{align*}
where here $\nu_i,\nu_i'$ are 2-dimensional distributions, $N_i(\tau)$ is the number of pulls of arm $i$ up to stopping time $\tau$, and $\calF_\tau$ is the filtration up to $\tau$. 

To prove the result, for every coordinate $i$ we construct three alternate instances---one which involves perturbing $\theta_i$, and two which involve perturbing $\mu_i$. We then apply the above inequality with each alternate instance to show a lower bound on the number of times we must pull arm $i$. 

\paragraph{Alternate Instance 1.} Fix $i \neq \ist$ and consider the alternate instance $\nu'$ where $\nu_j' = \nu_j$ for $j \neq i$, and 
\begin{align*}
\nu_i = \calN \left ( \begin{bmatrix} a_i \theta_i' \\ a_i \mu_i \end{bmatrix}, I \right ) 
\end{align*}
for $\theta_i' = \mu_i \theta_{\ist}/\mu_{\ist} + \alpha, \alpha > 0$. Let $\calE = \{ \ihat = \ist \}$ (where here $\ist$ denotes the optimal arm on $\nu$). Note that on $\nu'$, we will have that $\theta'_i/\mu_i = \theta_{\ist}/\mu_{\ist} + \alpha/\mu_i > \theta_{\ist}/\mu_{\ist}$. It follows that arm $i$ is the optimal arm under instance $\nu'$, so for any $\delta$-PAC algorithm, $\bbP_\nu[\calE] \ge 1 - \delta, \bbP_{\nu'}[\calE] \le \delta$. In addition, we have that
\begin{align*}
\kl(\nu_i,\nu_i') = \frac{a_i^2}{2} (\theta_i - \theta_i')^2 = \frac{a_i^2}{2} (\theta_i - \mu_i \theta_{\ist} / \mu_{\ist} - \alpha)^2 = \frac{a_i^2 \mu_i^2}{2} (\theta_i / \mu_i - \theta_{\ist}/\mu_{\ist} - \alpha / \mu_i)^2 .
\end{align*}
Lower bounding $d(\Pr_\nu(\calE), \Pr_{\nu'}(\calE)) \ge \log \frac{1}{2.4\delta}$ as in \cite{kaufmann2016complexity} and letting $\alpha \rightarrow 0$, we conclude that
\begin{align*}
\Exp_{\nu}[N_i(\tau)]  \ge \frac{2}{a_i^2 \mu_i^2 ( \frac{\theta_i}{\mu_i} - \frac{\theta_{\ist}}{\mu_{\ist}} )^2} \cdot \log \frac{1}{2.4 \delta} .
\end{align*}

\paragraph{Alternate Instance 2.} Next, consider the alternate instance $\nu'$ where $\mu_i' = \mu_{\ist} \theta_i / \theta_{\ist} - \alpha$, and leaving all other parameters the same as $\nu$. We now have that 
\begin{align*}
\frac{\theta_i}{\mu_i'} = \frac{\theta_i}{\frac{\mu_{\ist} \theta_i}{\theta_{\ist}} - \alpha } = \frac{1}{\frac{\mu_{\ist}}{\theta_{\ist}} - \frac{\alpha}{\theta_i}} > \frac{\theta_{\ist}}{\mu_{\ist}}
\end{align*}
where the last inequality holds for small enough $\alpha$. Thus, $i$ is the optimal coordinate for $\nu'$. Using that
\begin{align*}
\kl(\nu_i,\nu_i') = \frac{a_i^2}{2} (\mu_i - \mu_i')^2 = \frac{a_i^2}{2} ( \mu_i - \frac{\mu_{\ist} \theta_i}{\theta_{\ist}} + \alpha)^2 = \frac{a_i^2}{2} ( \frac{\mu_i \mu_{\ist}}{\theta_{\ist}} )^2 ( \frac{\theta_{\ist}}{\mu_{\ist}} - \frac{\theta_i}{\mu_i} + \frac{\theta_{\ist} \alpha}{\mu_{i} \mu_{\ist}} )^2 .
\end{align*} 
Letting $\alpha$ tend to 0 and using the same event $\calE$, we conclude that
\begin{align*}
\Exp_{\nu}[N_i(\tau)] \ge \frac{2}{a_i^2  ( \frac{\mu_i \mu_{\ist}}{\theta_{\ist}} )^2 ( \frac{\theta_{\ist}}{\mu_{\ist}} - \frac{\theta_i}{\mu_i}  )^2} \cdot \log \frac{1}{2.4 \delta} .
\end{align*}

\paragraph{Alternate Instance 3.} We next consider the alternate instance $\nu'$ where $\mu_i' = \mu_i - \frac{\mu_i^2}{\theta_i} ( \frac{\theta_{\ist}}{\mu_{\ist}} - \frac{\theta_i}{\mu_i} )$ and all other parameters the same as $\nu$. Now,
\begin{align*}
\frac{\mu_i'}{\theta_i} = \frac{\mu_i}{\theta_i} - \frac{\mu_i^2}{\theta_i^2} ( \frac{\theta_{\ist}}{\mu_{\ist}} - \frac{\theta_i}{\mu_i} ) = \frac{\mu_{\ist}}{\theta_{\ist}} + ( \frac{\mu_i}{\theta_i} - \frac{\mu_{\ist}}{\theta_{\ist}})  - \frac{\mu_i^2}{\theta_i^2} ( \frac{\theta_{\ist}}{\mu_{\ist}} - \frac{\theta_i}{\mu_i} ) .
\end{align*} 
Note that
\begin{align*}
 \frac{\mu_i}{\theta_i} - \frac{\mu_{\ist}}{\theta_{\ist}} = \frac{\mu_i \mu_{\ist}}{\theta_i \theta_{\ist}} ( \frac{\theta_{\ist}}{\mu_{\ist}} - \frac{\theta_i}{\mu_i} ) < \frac{\mu_i^2}{\theta_i^2} ( \frac{\theta_{\ist}}{\mu_{\ist}} - \frac{\theta_i}{\mu_i} ) 
\end{align*}
where the inequality follows since $\ist$ is the unique maximizer of $\theta_i/\mu_i$, so $ ( \frac{\theta_{\ist}}{\mu_{\ist}} - \frac{\theta_i}{\mu_i} ) > 0$ and $\mu_i / \theta_i > \mu_{\ist}/\theta_{\ist} $ . Thus,
\begin{align*}
\frac{\mu_{\ist}}{\theta_{\ist}} + ( \frac{\mu_i}{\theta_i} - \frac{\mu_{\ist}}{\theta_{\ist}})  - \frac{\mu_i^2}{\theta_i^2} ( \frac{\theta_{\ist}}{\mu_{\ist}} - \frac{\theta_i}{\mu_i} ) < \frac{\mu_{\ist}}{\theta_{\ist}} 
\end{align*}
so $\theta_i / \mu_i' > \mu_{\ist}/\theta_{\ist}$, and $i$ is therefore the optimal coordinate on instance $\nu'$. Note also that
\begin{align*}
\kl(\nu_i,\nu_i') = \frac{a_i^2}{2} ( \mu_i - \mu_i')^2 = \frac{a_i^2}{2} (\frac{\mu_i^2}{\theta_i} ( \frac{\theta_{\ist}}{\mu_{\ist}} - \frac{\theta_i}{\mu_i} ))^2 =  \frac{a_i^2 \mu_i^4}{2 \theta_i^2}  ( \frac{\theta_{\ist}}{\mu_{\ist}} - \frac{\theta_i}{\mu_i} )^2 .
\end{align*}
Using the same event $\calE$ as above, we conclude that 
\begin{align*}
\Exp_{\nu}[N_i(\tau)] \ge \frac{2}{\frac{a_i^2 \mu_i^4}{ \theta_i^2}  ( \frac{\theta_{\ist}}{\mu_{\ist}} - \frac{\theta_i}{\mu_i} )^2} \cdot \log \frac{1}{2.4 \delta} .
\end{align*}

\paragraph{Full Lower Bound.} Putting these three lower bound together, it follows that
\begin{align*}
\Exp_{\nu}[N_i(\tau)] & \ge \max \left \{ \frac{2}{a_i^2 \mu_i^2 ( \frac{\theta_i}{\mu_i} - \frac{\theta_{\ist}}{\mu_{\ist}} )^2} , \frac{2}{a_i^2  ( \frac{\mu_i \mu_{\ist}}{\theta_{\ist}} )^2 ( \frac{\theta_{\ist}}{\mu_{\ist}} - \frac{\theta_i}{\mu_i}  )^2},  \frac{2}{\frac{a_i^2 \mu_i^4}{ \theta_i^2}  ( \frac{\theta_{\ist}}{\mu_{\ist}} - \frac{\theta_i}{\mu_i} )^2} \right \} \cdot \log \frac{1}{2.4 \delta} \\
& \ge \frac{ \frac{2}{3 a_i^2} (\frac{1}{\mu_i^2} + (\frac{\theta_{\ist}}{\mu_i \mu_{\ist}})^2 + \frac{\theta_i^2}{\mu_i^4})}{( \frac{\theta_{\ist}}{\mu_{\ist}} - \frac{\theta_i}{\mu_i}  )^2} \cdot \log \frac{1}{2.4\delta} .
\end{align*}
Repeating this argument for each $i \neq \ist$ and noting that $\tau = \sum_{i=1}^d N_i(\tau)$ gives the result. 
\end{proof}

\begin{proof}[Proof of \Cref{thm:linear_lb}]
For any coordinate $i$, note that any $(\delta,\{ \calX_i \}_{i \in [d]})$-PAC algorithm can pull all arms $a_i \le \gamma /\mu_i$. Assume that we pull $a_i = \gamma / \mu_i$ for all time, then the observation model and hypothesis test are identical to that used in \Cref{lem:lb_change_measure}. It therefore follows from \Cref{lem:lb_change_measure} that
\begin{align*}
\Exp_{\nu}[\tau] \ge \sum_{i \neq \ist}  \frac{ \frac{2}{3 a_i^2} (\frac{1}{\mu_i^2} + (\frac{\theta_{\ist}}{\mu_i \mu_{\ist}})^2 + \frac{\theta_i^2}{\mu_i^4})}{( \frac{\theta_{\ist}}{\mu_{\ist}} - \frac{\theta_i}{\mu_i}  )^2} \cdot \log \frac{1}{2.4\delta} .
\end{align*}
Plugging in $a_i = \gamma/\mu_i$ (note that this value minimizes the above, subject to the safety constraint) gives
\begin{align*}
\frac{ \frac{2}{3 a_i^2} (\frac{1}{\mu_i^2} + (\frac{\theta_{\ist}}{\mu_i \mu_{\ist}})^2 + \frac{\theta_i^2}{\mu_i^4})}{( \frac{\theta_{\ist}}{\mu_{\ist}} - \frac{\theta_i}{\mu_i}  )^2} & = \frac{ \frac{2}{3} \mu_i^2 (\frac{1}{\mu_i^2} + (\frac{\theta_{\ist}}{\mu_i \mu_{\ist}})^2 + \frac{\theta_i^2}{\mu_i^4})}{\gamma^2 ( \frac{\theta_{\ist}}{\mu_{\ist}} - \frac{\theta_i}{\mu_i}  )^2} =  \frac{ \frac{2}{3}  ( 1 + \frac{\theta_{\ist}^2}{\mu_{\ist}^2} + \frac{\theta_i^2}{\mu_i^2})}{\Delta_i^2} 
\end{align*}
which proves the result. 
\end{proof}

\newcommand{\Deltil}{\widetilde{\Delta}}

%!TEX root = ../BAI_Constraint.tex

\section{Monotonic Response Functions Proof}

The following is a useful fact that we will employ in several places.

\begin{proposition}\label{prop:lip_mon_fun}
Assume that $f$ is nondecreasing and is 1-Lipschitz. Then for $x \ge y$, we will have
\begin{align*}
y - f(y) \le x - f(x) . 
\end{align*}
\end{proposition}
\begin{proof}
This trivially follows from the assumptions on $f$. Since the function is 1-Lipschitz, we have $| f(x) - f(y) | \le | x - y | $, and since the function is nondecreasing and $x \ge y$, we have $| f(x) - f(y) | = f(x) - f(y)$ and $|x - y| = x - y$. Rearranging these expressions gives the result.
\end{proof}

\paragraph{Notation.} 
Let $\nhat_{i,\ell}$ denote the first value of $n_i - 1$ in epoch $\ell$ when $  \gamma - \ghat_i(\asafehat^{n_i-1,\ell}(i)) \le 2 \epsilon_\ell$. Similarly, while $\unsafe(i) = 0$, let $\mhat_{i,\ell}$ denote the first value of $m_i - 1$ in epoch $\ell$ when $\gamma + \epssafe - \ghat_i(\aunhat^{m_i - 1,\ell}(i)) \le 2 \epsilon_\ell$, and when $\unsafe(i) = 1$, let $\mhat_{i,\ell}$ denote the first value of $m_i - 1$ for which the condition on Line \ref{line:while_binary_search} or Line \ref{line:if_a_close} is met. Let $\ellhatun(i)$ denote the epoch on which we set $\unsafe(i) = 1$.

\begin{lemma}\label{lem:efun_prob}
Let $\Efun$ be the event on which, for all $\ell$, $i \in \calX_{\ell}$, $n_i$, and $m_i$,
\begin{align*}
& | \ghat_i(\asafehat^{n_i,\ell}(i)) - g_i(\asafehat^{n_i,\ell}(i)) | \le \epsilon_\ell, \quad | \fhat_i(\asafehat^{n_i,\ell}(i)) - f_i(\asafehat^{n_i,\ell}(i)) | \le \epsilon_\ell \\
& | \ghat_i(\aunhat^{m_i,\ell}(i)) - g_i(\aunhat^{m_i,\ell}(i)) | \le \epsilon_\ell, \quad | \fhat_i(\aunhat^{m_i,\ell}(i)) - f_i(\aunhat^{m_i,\ell}(i)) | \le \epsilon_\ell .
\end{align*}
Then $\mathbb{P}[\Efun] \ge 1 - \delta$.
\end{lemma}
\begin{proof}
Our estimate $\ghat_i(\asafehat^{n_i,\ell}(i))$ will have mean $ g_i(\asafehat^{n_i,\ell}(i))$ and variance $1/N_{\ell,t}$, and similarly $\fhat_i(\asafehat^{n_i,\ell}(i))$ will have mean $f_i(\asafehat^{n_i,\ell}(i))$, and variance $1/N_{\ell,t}$. By the concentration of sub-Gaussian random variables, it follows that, with probability at least $1 - \delta/ ( 2 t^2 )$,
\begin{align*}
 | \ghat_i(\asafehat^{n_i,\ell}(i)) - g_i(\asafehat^{n_i,\ell}(i)) | \le \sqrt{\frac{2 \sigma^2 \log \frac{8 t^2}{\delta}}{N_{\ell,t}}} = \epsilon_\ell, \quad | \fhat_i(\asafehat^{n_i,\ell}(i)) - f_i(\asafehat^{n_i,\ell}(i)) | \le \sqrt{\frac{2 \sigma^2 \log \frac{8 t^2}{\delta}}{N_{\ell,t}}} = \epsilon_\ell.
\end{align*}
The same is true of $\ghat_i(\aunhat^{m_i,\ell}(i))$ and $\fhat_i(\aunhat^{m_i,\ell}(i))$. As we increment $t$ each time we collect samples, we can union bound over the total number of times samples are collected, and will get that the above events hold every time with probability at least
\begin{align*}
1 - \sum_{t=1}^{\infty} \frac{\delta}{2t^2} = 1 - \frac{\delta}{2} \cdot \frac{\pi^2}{6} \ge 1 - \delta.
\end{align*}
\end{proof}

\subsection{Controlling Safe and Unsafe Value Estimates}
\begin{lemma}\label{lem:ahat_ge_a0}
For all $k$, $\ell$, and $i$, we will have that $\asafehat^{k,\ell}(i) \ge a_{0,i}$ and $\aunhat^{k,\ell} \ge a_{0,i}$. 
\end{lemma}
\begin{proof}
We initialize both $\asafehat^{0,0}(i) = \aunhat^{0,0}(i) = a_{0,i}$. The only point at which we change $\asafehat^{k,\ell}(i)$ is Line \ref{line:asafehat_increment}, and we update it as
\begin{align*}
\asafehat^{n_i,\ell}(i) \leftarrow \gamma + \asafehat^{n_i-1,\ell}(i) - \ghat_i(\asafehat^{n_i-1,\ell}(i)) - \epsilon_\ell . 
\end{align*}
Note that, by the condition of the preceding while loop, to run this line we must have that $\gamma - \ghat_i(\asafehat^{n_i-1,\ell}(i)) > 2 \epsilon_\ell$. It follows that \begin{align*}
\asafehat^{n_i,\ell}(i) & = \gamma + \asafehat^{n_i-1,\ell}(i) - \ghat_i(\asafehat^{n_i-1,\ell}(i)) - \epsilon_\ell \\
& \ge \asafehat^{n_i-1,\ell}(i) + \epsilon_\ell \\
& \ge \asafehat^{n_i-1,\ell}(i).
\end{align*}
Thus, we only increase the value of $\asafehat^{n_i-1,\ell}(i)$, so since $\asafehat^{0,0}(i) = a_{0,i}$, the result follows. A similar argument shows that while $\unsafe(i) = 0$, we will also have that $\aunhat^{m_i,\ell}(i) \ge a_{0,i}$. Once $\unsafe(i) = 1$, note that we will always have that $\aunhat^{m_i,\ell}(i) \ge \asafehat^{n_i-1,\ell}(i)$, so it follows here that $\aunhat^{m_i,\ell}(i) \ge a_{0,i}$ as well. 
\end{proof}

\begin{lemma}\label{lem:fun_rand_seq_upper_bounds}
For all $k$, $\ell$ and $i$, on $\Efun$ we will have that $\asafehat^{k,\ell}(i) \le g_i^{-1}(\gamma)$ and $\aunhat^{k,\ell}(i) \le g_i^{-1}(\gamma + \epssafe)$. 
\end{lemma}
\begin{proof}
We prove this result by induction. By assumption $\asafehat^{0,0}(i) \le g_i^{-1}(\gamma)$. Assume that $\asafehat^{k-1,\ell}(i) \le g_i^{-1}(\gamma)$, then, on $\Efun$,
\begin{align*}
\asafehat^{k,\ell}(i) & = \gamma + \asafehat^{k-1,\ell}(i) - \ghat_i(\asafehat^{k-1,\ell}(i)) - \epsilon_\ell \\
& \le \gamma + \asafehat^{k-1,\ell}(i) - g_i(\asafehat^{k-1,\ell}(i)) \\
& \le \gamma + g_i^{-1}(\gamma) - g_i(g_i^{-1}(\gamma)) \\
& = g_i^{-1}(\gamma)
\end{align*}
where the last inequality follows from \Cref{prop:lip_mon_fun}. As we always set $\asafehat^{0,\ell}(i) = \asafehat^{\nhat_{i,\ell-1},\ell-1}(i)$, the upper bound on $\asafehat^{k,\ell}(i) $ follows.

While $\unsafe(i) = 0$, the update of $\aunhat^{k,\ell}(i) $ is analogous to the update of $\asafehat^{k,\ell}(i) $ except we replace $\gamma$ with $\gamma + \epssafe$. Thus, in this regime an identical argument to the above shows that $\aunhat^{k,\ell}(i) \le g_i^{-1}(\gamma + \epssafe)$. When $\unsafe(i) = 1$, note that we can only decrease $\aunhat^{k,\ell}(i)$ so the result also holds there.
\end{proof}

\begin{lemma}\label{lem:asafehat_bounds_closedform}
On $\Efun$, $\asafehat^{\nhat_{i,\ell},\ell}(i) \ge \gtil_i^{-1}(\gamma - 3\epsilon_\ell)$.
\end{lemma}
\begin{proof}
By construction $  \gamma - \ghat_i(\asafehat^{\nhat_{i,\ell},\ell}(i)) \le 2 \epsilon_\ell$. This implies, on $\Efun$, 
\begin{align*}
g_i(\asafehat^{\nhat_{i,\ell},\ell}(i)) \ge \gamma - 3 \epsilon_\ell \implies \asafehat^{\nhat_{i,\ell},\ell}(i) \ge \gtil_i^{-1}(\gamma - 3 \epsilon_\ell) 
\end{align*}
where we use $\gtilinv$ since, by \Cref{lem:ahat_ge_a0}, we know that $\asafehat^{\nhat_{i,\ell},\ell}(i) \ge a_{0,i}$.

\end{proof}

\begin{lemma}\label{lem:aunhat_bounds_closedform}
On $\Efun$, while $\unsafe(i) = 0$, $\aunhat^{\mhat_{i,\ell},\ell}(i) \ge  \gtil_i^{-1}(\gamma + \epssafe - 3 \epsilon_\ell)$.
\end{lemma}
\begin{proof}
Note that when $\unsafe(i) = 0$, the sequence $\aunhat^{m_i,\ell}(i)$ is updated analogously to $\asafehat^{n_i,\ell}(i)$, except with $\gamma$ replaced by $\gamma + \epssafe$. We can therefore use the proof from \Cref{lem:asafehat_bounds_closedform} to get the result. 
\end{proof}

\begin{lemma}\label{lem:aunhat_unsafe}
On $\Efun$, for all $i$ and $\ell \ge \ellhatun(i)$, we will have that $\aunhat^{\mhat_{i,\ell},\ell} \ge g_i^{-1}(\gamma)$.
\end{lemma}
\begin{proof}
We prove this by induction. Note that to set $\unsafe(i) = 1$, we must have that $\ghat_i(\aunhat^{m_i,\ell}(i)) - \epsilon_\ell \ge \gamma$. On $\Efun$, this implies that $g_i(\aunhat^{m_i,\ell}(i)) \ge \gamma$, so $g_i(\aunhat^{\mhat_{i,\ellhatun(i)},\ellhatun(i)}(i)) \ge \gamma$. 

Now assume that for some $\ell \ge \ellhatun(i)$, $g_i(\aunhat^{\mhat_{i,\ell},\ell}(i)) \ge \gamma$. At round $\ell +1$, the while loop on Line \ref{line:while_binary_search} either terminates when $\ghat_i(\aunhat^{m_{i}-1,\ell+1}(i)) - \epsilon_{\ell+1} \ge \gamma$ and $\aunhat^{\mhat_{i,\ell+1},\ell+1}(i) = \aunhat^{m_i-1,\ell+1}(i)$, or it terminates when the if statement on Line \ref{line:if_a_close} is true, and we set $\aunhat^{\mhat_{i,\ell+1},\ell+1}(i) = \aunhat^{\mhat_{i,\ell},\ell}(i)$. In the latter case we are done since we have assumed that $g_i(\aunhat^{\mhat_{i,\ell},\ell}(i)) \ge \gamma$. In the former case, on $\Efun$ we will have that $\ghat_i(\aunhat^{m_{i}-1,\ell+1}(i)) - \epsilon_{\ell+1} \ge \gamma$ implies $g_i(\aunhat^{m_{i}-1,\ell+1}(i)) \ge \gamma$, so $g_i(\aunhat^{\mhat_{i,\ell+1},\ell+1}(i)) \ge \gamma$. This proves the inductive hypothesis so the result follows.
\end{proof}

\begin{lemma}\label{lem:aunhat_bound2}
On $\Efun$, once $\unsafe(i) = 1$, we can bound 
$$\aunhat^{0,\ell+1}(i) = \aunhat^{\mhat_{i,\ell},\ell}(i) \le  \sum_{s = 1}^{\ell} \frac{g_i^{-1}(\min \{ \gamma + 2 \epsilon_s, \gamma + \epssafe \})}{2^{\ell  - s + 1}} + \left ( \ell + \frac{4 g_i^{-1}(\gamma + \epssafe)}{ \epssafe } \right ) 2^{-\ell} =: \aunbar^\ell(i) .$$
\end{lemma}
\begin{proof}
Assume that for epoch $\ell$, the if statement on Line \ref{line:if_a_close} is never true. Then, by definition, we will have that $\ghat_i(\aunhat^{\mhat_{i,\ell} - 1,\ell}(i)) - \epsilon_\ell < \gamma$ and $\ghat_i(\aunhat^{\mhat_{i,\ell} ,\ell}(i)) - \epsilon_\ell \ge \gamma$. It follows that we can bound, on $\Efun$,
\begin{align*}
g_i(\aunhat^{\mhat_{i,\ell} - 1,\ell}(i)) < \gamma +  2\epsilon_\ell .
\end{align*}
By \Cref{lem:fun_rand_seq_upper_bounds}, we can also bound $g_i(\aunhat^{\mhat_{i,\ell} - 1,\ell}(i)) \le \gamma + \epssafe$. Putting this together we get that
\begin{align*}
\aunhat^{\mhat_{i,\ell} - 1,\ell}(i) \le g_i^{-1}( \min \{ \gamma + 2 \epsilon_\ell, \gamma + \epssafe \}) . 
\end{align*}
Note that this is well-defined since by assumption $g_i^{-1}(\gamma + \epssafe)$ is well-defined. By definition, $\aunhat^{\mhat_{i,\ell} ,\ell}(i) = \frac{\aunhat^{0 ,\ell}(i)}{2} + \frac{\aunhat^{\mhat_{i,\ell} -1,\ell}(i)}{2}$ so if follows
\begin{align*}
\aunhat^{\mhat_{i,\ell} ,\ell}(i) \le \frac{\aunhat^{0 ,\ell}(i)}{2} + \frac{g_i^{-1}(\min \{ \gamma + 2 \epsilon_\ell, \gamma + \epssafe \})}{2}. 
\end{align*}
If we instead terminate with the if statement on Line \ref{line:if_a_close}, we will have that $\aunhat^{\mhat_{i,\ell} ,\ell}(i) = \aunhat^{0 ,\ell}(i)$. Furthermore, the termination criteria is only met if $\aunhat^{0,\ell}(i) - \aunhat^{m_i,\ell}(i) \le \epsilon_{\ell}$ and $\ghat_i(\aunhat^{m_i,\ell}(i)) - \epsilon_\ell < \gamma$. On $\Efun$, this implies that $\aunhat^{m_i,\ell}(i) < g_i^{-1}(\min \{ \gamma + 2 \epsilon_\ell, \gamma + \epssafe \})$, which implies that
\begin{align*}
\aunhat^{\mhat_{i,\ell} ,\ell}(i) = \aunhat^{0 ,\ell}(i) = \frac{\aunhat^{0 ,\ell}(i)}{2} + \frac{\aunhat^{0 ,\ell}(i)}{2} \le \frac{\aunhat^{0 ,\ell}(i)}{2} +  \frac{g_i^{-1}(\min \{ \gamma + 2 \epsilon_\ell, \gamma + \epssafe \}) + \epsilon_\ell}{2} . 
\end{align*}
By construction $\aunhat^{0 ,\ell+1}(i) = \aunhat^{\mhat_{i,\ell} ,\ell}(i)$, so we have shown that for all $\ell > \ellhatun(i)$, 
\begin{align*}
\aunhat^{0 ,\ell+1}(i) & \le \frac{\aunhat^{0 ,\ell}(i)}{2} + \frac{g_i^{-1}(\min \{ \gamma + 2 \epsilon_\ell, \gamma + \epssafe \})+ \epsilon_\ell}{2} \\
& \le \sum_{s = \ellhatun(i)+1}^{\ell} \frac{g_i^{-1}(\min \{ \gamma + 2 \epsilon_s, \gamma + \epssafe \}) + \epsilon_s}{2^{\ell  - s + 1}} + \frac{\aunhat^{0,\ellhatun(i)}(i)}{2^{\ell - \ellhatun(i)+1}}.
\end{align*}
We would like to obtain a deterministic bound that does not depend on $\ellhatun(i)$. To this end, we bound
\begin{align*}
&  \sum_{s = \ellhatun(i)+1}^{\ell} \frac{\epsilon_s}{2^{\ell - s + 1}}  \le  \sum_{s = 1}^{\ell} \frac{2^{-s}}{2^{\ell - s + 1}} = \ell 2^{-\ell}, \\
& \sum_{s = \ellhatun(i)+1}^{\ell} \frac{g_i^{-1}(\min \{ \gamma + 2 \epsilon_s, \gamma + \epssafe \})}{2^{\ell  - s + 1}} \le \sum_{s = 1}^{\ell} \frac{g_i^{-1}(\min \{ \gamma + 2 \epsilon_s, \gamma + \epssafe \})}{2^{\ell  - s + 1}} .
\end{align*}
Finally, by \Cref{lem:fun_rand_seq_upper_bounds}, we can bound $\aunhat^{0,\ellhatun(i)}(i) \le g_i^{-1}(\gamma + \epssafe)$, and by \Cref{lem:ellhatun_upper} we can bound $\ellhatun(i) \le \log_2 \frac{8}{\epssafe}$, to get that
\begin{align*}
 \frac{\aunhat^{0,\ellhatun(i)}(i)}{2^{\ell - \ellhatun(i)+1}} \le  \frac{g_i^{-1}(\gamma + \epssafe)}{2^{\ell - \log_2 \frac{8}{\epssafe}+1}} = \frac{4 g_i^{-1}(\gamma + \epssafe)}{ \epssafe 2^\ell}. 
\end{align*}
Putting these bounds together gives the result. 
\end{proof}

\subsection{Bounding Number of Epochs}
\begin{lemma}\label{lem:nhat_bound}
On $\Efun$, $\nhat_{i,\ell} \le  \nbar_{i,\ell} $ where
\begin{align*}
\nbar_{i,\ell} := \left \{ \begin{matrix} \frac{ \gtil_i^{-1}(\gamma -  \epsilon_\ell) - \gtil_i^{-1}(\gamma - 6 \epsilon_\ell)}{ \epsilon_\ell}& \ell \ge 2 \\
2(\gtil_i^{-1}(\gamma -  \tfrac{1}{2}) - a_{0,i}) & \ell = 1 \end{matrix} \right .  .
\end{align*}
\end{lemma}
\begin{proof}
Throughout this proof, we make use of \Cref{lem:ahat_ge_a0} to note that $\asafehat^{n_i,\ell}(i) \ge a_{0,i}$ always. This allows us to lower bound terms with a $\gtilinv$.

By definition, $\nhat_{i,\ell}$ is the smallest value of $n_i - 1$ such that $ \gamma - \ghat_i(\asafehat^{n_i-1,\ell}(i)) \le 2 \epsilon_\ell$. Assume we are in the regime where $n_i - 1 < \nhat_{i,\ell}$, so that $ \gamma - \ghat_i(\asafehat^{n_i-1,\ell}(i)) > 2 \epsilon_\ell$. By construction, 
\begin{align*}
\asafehat^{n_i,\ell}(i) & =  \gamma  + \asafehat^{n_i-1,\ell}(i) - \ghat_i(\asafehat^{n_i-1,\ell}(i)) - \epsilon_\ell  \ge \asafehat^{n_i-1,\ell}(i) +  \epsilon_\ell .
\end{align*}
It follows that while $n_i - 1 < \nhat_{i,\ell}$, we can lower bound
\begin{align*}
\asafehat^{n_i,\ell}(i) \ge \asafehat^{0,\ell}(i) + n_i  \epsilon_\ell .
\end{align*}
On $\Efun$, we can guarantee that $ \gamma - \ghat_i(\asafehat^{n_i-1,\ell}(i)) \le 2 \epsilon_\ell$ once
\begin{align*}
\gamma - g_i(\asafehat^{n_i-1,\ell}(i)) \le  \epsilon_\ell \iff \asafehat^{n_i-1,\ell}(i) \ge \gtil_i^{-1}(\gamma - \epsilon_\ell) .
\end{align*}
It follows that we will have reached the exit condition once
\begin{align}\label{eq:ni_upper_bound_cond}
\asafehat^{0,\ell}(i) + n_i  \epsilon_\ell  \ge \gtil_i^{-1}(\gamma -  \epsilon_\ell)  \iff
n_i \ge \frac{ \gtil_i^{-1}(\gamma - \epsilon_\ell) - \asafehat^{0,\ell}(i)}{ \epsilon_\ell}. 
\end{align}
Note also that, for $\ell \ge 2$, on $\Efun$ we must have that 
\begin{align*}
\gamma - g_i(\asafehat^{0,\ell}(i)) \le 6 \epsilon_\ell
\end{align*}
since we know that the termination criteria $\gamma - \ghat_i(\asafehat^{0,\ell}(i)) \le 2 \epsilon_{\ell-1} = 4 \epsilon_\ell$ was reached for $\asafehat^{0,\ell}(i) = \asafehat^{\nhat_{i,\ell-1},\ell-1}(i)$ at epoch $\ell-1$, and since we sample up to tolerance $\epsilon_\ell$ at round $\ell$ (note that this holds even if the termination criteria was reached for $n_i - 1 = 0$ at round $\ell-1$, since we collect data on Line \ref{line:pull_enough} to ensure we have an accurate enough sample, i.e. we obtain an $\epsilon_{\ell-1}$-accurate sample at $\ghat_i(\asafehat^{0,\ell-1}(i))$ instead of relying on our previous $\epsilon_{\ell-2}$-accurate estimate of this). It follows that $\asafehat^{0,\ell}(i) \ge \gtil_i^{-1} ( \gamma - 6 \epsilon_\ell)$, so a sufficient condition to meet \eqref{eq:ni_upper_bound_cond} is that
\begin{align*}
n_i \ge \frac{ \gtil_i^{-1}(\gamma -  \epsilon_\ell) - \gtil_i^{-1}(\gamma - 6 \epsilon_\ell)}{ \epsilon_\ell} .
\end{align*}
In the case where $\ell = 1$, we can simply replace $\asafehat^{0,\ell}(i)$ in \eqref{eq:ni_upper_bound_cond} with $a_{0,i}$. 

Finally, if we have that $\nhat_{i,\ell} = 0$, i.e. we terminate on the first iteration, we note that this is still a valid upper bound since $ \gtil_i^{-1}(\gamma -  \epsilon_\ell) \ge \gtil_i^{-1}(\gamma - 6 \epsilon_\ell)$. 
\end{proof}

\begin{lemma}\label{lem:mhat_bound}
On $\Efun$, $\mhat_{i,\ell} \le \mbar_{i,\ell} $ where
\begin{align*}
\mbar_{i,\ell} :=   \left \{ \begin{matrix} \frac{ \gtil_i^{-1}(\gamma + \epssafe -  \epsilon_\ell) - \gtil_i^{-1}(\gamma + \epssafe - 6 \epsilon_\ell)}{\epsilon_\ell}& \ell \ge 2 \\
2( \gtil_i^{-1}(\gamma + \epssafe -  \tfrac{1}{2}) - a_{0,i}) & \ell = 1 \end{matrix} \right .   .
\end{align*}
\end{lemma}
\begin{proof}
The proof of this follows identically to the proof of \Cref{lem:nhat_bound}, since while $\unsafe(i) = 0$, the sequences $\asafehat^{k,\ell}(i)$ and $\aunhat^{k,\ell}(i)$ are updated in analogous ways, with the $\gamma$ in the $\asafehat^{k,\ell}(i)$ update replaced by $\gamma + \epssafe$ in the $\aunhat^{k,\ell}(i)$ update.
\end{proof}

\begin{lemma}\label{lem:ellhatun_upper}
On $\Efun$, $\ellhatun(i) \le \lceil \log_2 \frac{8}{\epssafe} \rceil$.
\end{lemma}
\begin{proof}
Since we only terminate the while loop on Line \ref{line:while_aun_unsafe0} once $ \gamma + \epssafe - \ghat_i(\aunhat^{m_i - 1,\ell}(i)) \le 2 \epsilon_\ell$, we will have that
\begin{align*}
\aunhat^{m_i,\ell}(i) & = \gamma + \epssafe + \aunhat^{m_i - 1,\ell}(i) - \ghat_i(\aunhat^{m_i-1,\ell}(i)) - \epsilon_\ell \\
& \ge \aunhat^{m_i - 1,\ell}(i) +  \epsilon_\ell \\
& \ge \aunhat^{m_i - 1,\ell}(i).
\end{align*}
It follows that $\aunhat^{m_i - 1,\ell}(i)$ is increasing in $m_i$ if the termination criteria of the while loop has not been met. 

Let $\mhat_{\unsafe}$ denote the $m_i$ value where we set $\unsafe(i) \leftarrow 1$. As the while loop on Line \ref{line:while_aun_unsafe0} has not yet terminated, it follows that $\aunhat^{\mhat_{\unsafe},\ellhatun(i)}(i) \ge \aunhat^{\mhat_{i,\ellhatun(i)-1},\ellhatun(i)-1}(i) \ge \gtil_i^{-1}(\gamma + \epssafe - 3 \epsilon_{\ellhatun(i)-1})$, where the last inequality follows by \Cref{lem:aunhat_bounds_closedform}. 

By definition, we will have that $\ghat_i(\aunhat^{\mhat_{\unsafe},\ellhatun(i)}(i)) - \epsilon_{\ellhatun(i)} \ge \gamma$ which, on $\Efun$, is implied by $g_i(\aunhat^{\mhat_{\unsafe},\ellhatun(i)}(i)) - 2\epsilon_{\ellhatun(i)} \ge \gamma$. The above implies $g_i(\aunhat^{\mhat_{\unsafe},\ellhatun(i)}(i)) \ge g_i(\gtil_i^{-1}(\gamma + \epssafe - 3 \epsilon_{\ellhatun(i)-1}))$, so a sufficient condition is that
\begin{align*}
g_i(\gtil_i^{-1}(\gamma + \epssafe - 3 \epsilon_{\ellhatun(i)-1})) \ge \gamma + 2 \epsilon_{\ellhatun(i)}
\end{align*}
Assume that $\ellhatun(i)$ is large enough that $\gtil_i^{-1}(\gamma + \epssafe - 3 \epsilon_{\ellhatun(i)-1}) = g_i^{-1}(\gamma + \epssafe - 3 \epsilon_{\ellhatun(i)-1})$, then the above is equivalent to
\begin{align*}
\epssafe \ge 3 \epsilon_{\ellhatun(i)-1} + 2 \epsilon_{\ellhatun(i)} = 8 \epsilon_{\ellhatun(i)} .
\end{align*}
Using that $\epsilon_\ell = 2^{-\ell}$ and rearranging this gives that $\ellhatun(i) \ge \log_2 \frac{8}{\epssafe}$. 

As we assume that $g_i^{-1}(\gamma)$ is well-defined, we will have that $\gtil_i^{-1}(\gamma + \epssafe - 3 \epsilon_{\ellhatun(i)-1}) = g_i^{-1}(\gamma + \epssafe - 3 \epsilon_{\ellhatun(i)-1})$ once $\gamma + \epssafe - 3 \epsilon_{\ellhatun(i)-1} \ge \gamma \iff \epssafe \ge 6 \epsilon_{\ellhatun(i)} \iff \ellhatun(i) \ge \log_2 \frac{6}{\epssafe}$. It follows that once $\ell \ge \log_2 \frac{6}{\epssafe}$, we will be in the regime where $\gtil_i^{-1}(\gamma + \epssafe - 3 \epsilon_{\ellhatun(i)-1}) = g_i^{-1}(\gamma + \epssafe - 3 \epsilon_{\ellhatun(i)-1})$, and where a sufficient condition for termination will be $\ellhatun(i) \ge \log_2 \frac{8}{\epssafe}$. Putting these together we conclude that we will have set $\unsafe(i) = 1$ for $\ell \ge \log_2 \frac{8}{\epssafe}$, so $\ellhatun(i) \le \lceil \log_2 \frac{8}{\epssafe} \rceil$.

\end{proof}

\subsection{Sample Complexity of \algmono}

\begin{lemma}\label{lem:fun_ist_active}
On $\Efun$, $\ist \in \calX_\ell$ for all $\ell$. 
\end{lemma}
\begin{proof}
We will only eliminate $\ist$ if $\unsafe(\ist) = 1$ and there exists some $i \neq \ist$ such that
\begin{align*}
\fhat_{\ist}(\aunhat^{\mhat_{\ist,\ell},\ell}(\ist)) + \epsilon_\ell < \fhat_i(\asafehat^{\nhat_{i,\ell},\ell}(i)) - \epsilon_\ell.
\end{align*}
On $\Efun$, this implies that
\begin{align*}
f_{\ist}(\aunhat^{\mhat_{\ist,\ell},\ell}(\ist)) <  f_i(\asafehat^{\nhat_{i,\ell},\ell}(i)) .
\end{align*}
By \Cref{lem:fun_rand_seq_upper_bounds}, we will have that $g_{\ist}(\asafehat^{\nhat_{i,\ell},\ell}(i)) \le \gamma$, so $\asafehat^{\nhat_{i,\ell},\ell}(i)$ is safe. By \Cref{lem:aunhat_unsafe} we will have that $g_{\ist}(\aunhat^{\mhat_{\ist,\ell},\ell}(\ist)) \ge \gamma$ for $\ell \ge \ellhatun(\ist)$, so $\aunhat^{\mhat_{\ist,\ell},\ell}(\ist)$ is unsafe. However, this is a contradiction because $\ist$ by definition satisfies
\begin{align*}
f_{\ist}(g_{\ist}^{-1}(\gamma)) \ge \max_{i \neq \ist} f_{i}(g_{i}^{-1}(\gamma))
\end{align*}
and since $f_i$ and $f_{\ist}$ are monotonic. Thus, we must have that $\ist \in \calX_\ell$ for all $\ell$. 
\end{proof}

\begin{lemma}\label{lem:fun_eliminate_i}
On $\Efun$, we will have eliminated coordinate $i \neq \ist$ once 
\begin{align*}
\ell \ge \max \left \{ \Big \lceil \log_2 \frac{8}{\epssafe} \Big \rceil, \ellbar(i) \right \} 
\end{align*}
where 
\begin{align*}
\ellbar(i) := \argmin_{\ell \in \bbN} \ell \quad \text{s.t.} \quad f_i(\aunbar^\ell(i)) + 4 \epsilon_\ell \le f_{\ist}(\gtil_{\ist}^{-1}(\gamma - 3 \epsilon_\ell))  .
\end{align*}
\end{lemma}
\begin{proof}
By \Cref{lem:fun_ist_active}, we will always have that $\ist \in \calX_\ell$. It follows that a sufficient condition to eliminate coordinate $i$ is that $\unsafe(i) = 1$ and
\begin{align}\label{eq:fun_term_ineq}
\fhat_i(\aunhat^{\mhat_{i,\ell},\ell}(i)) + \epsilon_\ell \le \fhat_{\ist}(\asafehat^{\nhat_{\ist,\ell},\ell}(\ist)) - \epsilon_\ell .
\end{align}
By \Cref{lem:ellhatun_upper}, we will have that $\unsafe(i) = 1$ once $\ell \ge \lceil \log_2 \frac{8}{\epssafe} \rceil$. Assume $\ell$ meets this constraint so that $\unsafe(i)= 1$. On $\Efun$, \eqref{eq:fun_term_ineq} is implied by
\begin{align*}
f_i(\aunhat^{\mhat_{i,\ell},\ell}(i)) + 2\epsilon_\ell \le f_{\ist}(\asafehat^{\nhat_{\ist,\ell},\ell}(\ist)) - 2\epsilon_\ell .
\end{align*}
By \Cref{lem:aunhat_bound2}, since we have that $\unsafe(i) = 1$, we can upper bound $\aunhat^{\mhat_{i,\ell},\ell}(i) \le \aunbar^\ell(i)$, which implies $f_i(\aunhat^{\mhat_{i,\ell},\ell}(i)) \le f_i(\aunbar^\ell(i))$. Furthermore, by \Cref{lem:asafehat_bounds_closedform} we can lower bound $\asafehat^{\nhat_{\ist,\ell},\ell}(\ist) \ge \gtil_{\ist}^{-1}(\gamma - 3 \epsilon_\ell) $ so $f_{\ist}(\asafehat^{\nhat_{\ist,\ell},\ell}(\ist)) \ge f_{\ist}(\gtil_{\ist}^{-1}(\gamma - 3 \epsilon_\ell))$. If follows that the above expression is implied by
\begin{align*}
f_i(\aunbar^\ell(i)) + 2 \epsilon_\ell \le f_{\ist}(\gtil_{\ist}^{-1}(\gamma - 3 \epsilon_\ell)) - 2 \epsilon_\ell . 
\end{align*}
Combining these gives the result. 
\end{proof}

\begin{theorem}[Full version of \Cref{thm:monotonic_complexity}]\label{thm:monotonic_complexity2}
With probability $1-\delta$, \Cref{alg:constrained_bai_monotonic2} will terminate and output $\ist$ after collecting at most
\begin{align*}
& \sum_{i = 1}^d \bigg ( \sum_{\ell = 1}^{\lceil \log_2 \frac{8}{\epssafe} \rceil } ( \mbar_{i,\ell} + \nbar_{i,\ell} + \ell + 2) N_{\ell, \tbar}   + \sum_{\ell = \lceil \log_2 \frac{8}{\epssafe} \rceil + 1}^{\ellbar(i)} ( \nbar_{i,\ell} + \ell + 2) N_{\ell, \tbar} \bigg ) 
\end{align*}
samples, and will only pull safe arms during execution, where we define $\ellbar(\ist) := \max_{i \neq \ist} \ellbar(i)$. 
\end{theorem}
\begin{proof}
First, by \Cref{lem:efun_prob}, we will have that $\bbP[\Efun] \ge 1 - \delta$. We assume that $\Efun$ holds for the remainder of the proof. 

By \Cref{lem:fun_ist_active}, on $\Efun$ $\ist \in \calX_\ell$ for all $\ell$. Since \Cref{alg:constrained_bai_monotonic2} only terminates when $|\calX_\ell| = 1$, it follows that it will output $\ist$.

That \Cref{alg:constrained_bai_monotonic2} only pulls safe arms on $\Efun$ is given by \Cref{lem:fun_rand_seq_upper_bounds}.

For some suboptimal coordinate $i \neq \ist$, \Cref{lem:fun_eliminate_i} gives that on $\Efun$ we will eliminate $i$ after 
\begin{align*}
\ell \ge \max \left \{ \Big \lceil \log_2 \frac{8}{\epssafe} \Big \rceil, \ellbar(i) \right \} .
\end{align*}
Furthermore, by \Cref{lem:nhat_bound} and \Cref{lem:mhat_bound}, we can bound $\nhat_{i,\ell} \le \nbar_{i,\ell}$ and $\mhat_{i,\ell} \le \mbar_{i,\ell}$. Furthermore, once $\unsafe(i) = 1$, we can bound the number of iterations of the while loop on Line \ref{line:while_binary_search} by $\log_2(1/\epsilon_\ell) = \ell$, since after this many iterations the if statement on Line \ref{line:if_a_close} will have been met. Putting this together, we can upper bound
\begin{align*}
t \le \tbar := \sum_{i = 1}^d \sum_{\ell = 1}^{\ellbar(i)} ( \mbar_{i,\ell} + \nbar_{i,\ell} + \ell + 2) 
\end{align*}
where the 2 arises since we always collect at least one set of samples for both safe and unsafe at every epoch. By this same reasoning, we can bound the total number of samples collected by \Cref{alg:constrained_bai_monotonic2} from pulling suboptimal arms by
\begin{align*}
\sum_{i = 1, i \neq \ist}^d \left ( \sum_{\ell = 1}^{\lceil \log_2 \frac{8}{\epssafe} \rceil } ( \mbar_{i,\ell} + \nbar_{i,\ell} + \ell + 2 ) N_{\ell, \tbar}  + \sum_{\ell = \lceil \log_2 \frac{8}{\epssafe} \rceil + 1}^{\ellbar(i)} ( \nbar_{i,\ell} + \ell + 2) N_{\ell, \tbar} \right ) 
\end{align*}
where we have used \Cref{lem:ellhatun_upper} to upper bound the number of iterations where $\unsafe(i) = 0$, and thus where we have to pay $\mbar_{i,\ell}$. Furthermore, we can bound the total number of pulls to $\ist$ by
\begin{align*}
\max_{i \neq \ist} \left ( \sum_{\ell = 1}^{\lceil \log_2 \frac{8}{\epssafe} \rceil } ( \mbar_{\ist,\ell} + \nbar_{\ist,\ell} + \ell + 2) N_{\ell, \tbar}  + \sum_{\ell = \lceil \log_2 \frac{8}{\epssafe} \rceil + 1}^{\ellbar(i)} ( \nbar_{\ist,\ell} + \ell + 2) N_{\ell, \tbar} \right ) 
\end{align*}
where we have used the fact that we will only pull $\ist$ for as many epochs as it takes to eliminate the second-to-last remaining arm.
\end{proof}

\subsection{Interpreting the Complexity}
Throughout this section, we will assume that Assumption \ref{asm:cont_deriv} holds. The following result shows that this implies that $g_i^{-1}( \cdot )$ is $L$-Lipschitz, for each $i$. 

\begin{proposition}\label{prop:cont_inverse}
For any function $g_i$ satisfying Assumption \ref{asm:cont_deriv}, we will have that $g_i^{-1}$ exists and satisfies:
\begin{align*}
| g_i^{-1}(x) - g_i^{-1}(y) | \le L | x - y |
\end{align*}
for all $x,y \in [\gamma - 3/2, \gamma + \epssafe]$. 
\end{proposition}
\begin{proof}
The following is a standard analysis result
\begin{fact}\label{fact:inverse_deriv}
Assume that $f : I \rightarrow \bbR$ is strictly monotone and continuous on the interval $I$, and that its derivative $f'$ is defined at some $c \in I$. Then it's inverse $f^{-1} : f(I) \rightarrow \bbR$ is differentiable at $f(c)$ and
\begin{align*}
(f^{-1})'(f(c)) = \frac{1}{f'(c)} . 
\end{align*}
\end{fact}
Applying Fact \ref{fact:inverse_deriv} to our setting, we have that
\begin{align*}
(g_i^{-1})'(g_i(x)) = \frac{1}{g_i'(x)} \le L
\end{align*}
for any $x \in [g_i^{-1}(\gamma - 3/2), g_i^{-1}(\gamma+\epssafe)]$. As the Lipschitz constant of a function over an interval is bounded by its maximum derivative over that interval, the result follows.
\end{proof}

\begin{lemma}\label{lem:fun_mbar_bound}
Under Assumption \ref{asm:cont_deriv}, we can bound
\begin{align*}
\nbar_{i,\ell} \le \left \{ \begin{matrix}  5L & \ell \ge 2 \\
 2(g_i^{-1}(\gamma -  \tfrac{1}{2}) - a_{0,i}) & \ell = 1 \end{matrix} \right . , \quad \mbar_{i,\ell} \le   \left \{ \begin{matrix}  5L & \ell \ge 2 \\
 2(g_i^{-1}(\gamma + \epssafe -  \tfrac{1}{2}) - a_{0,i} ) & \ell = 1 \end{matrix} \right .  .
\end{align*}
\end{lemma}
\begin{proof}
Recall that:
\begin{align*}
\nbar_{i,\ell} := \left \{ \begin{matrix} \frac{ \gtil_i^{-1}(\gamma -  \epsilon_\ell) - \gtil_i^{-1}(\gamma - 6 \epsilon_\ell)}{ \epsilon_\ell}& \ell \ge 2 \\
2(\gtil_i^{-1}(\gamma -  \tfrac{1}{2}) - a_{0,i}) & \ell = 1 \end{matrix} \right .  .
\end{align*}
As we have assumed that $g_i^{-1}(\gamma - 3/2)$ is well-defined, we can replace all $\gtilinv_i$ in the above with $g_i^{-1}$. Then using that the inverse of a monotonic function is also monotonic, we can bound
\begin{align*}
\frac{ g_i^{-1}(\gamma -  \epsilon_\ell) - g_i^{-1}(\gamma - 6 \epsilon_\ell)}{ \epsilon_\ell} = \frac{ | g_i^{-1}(\gamma -  \epsilon_\ell) - g_i^{-1}(\gamma - 6 \epsilon_\ell)| }{ \epsilon_\ell} \le \frac{5 L \epsilon_\ell}{\epsilon_\ell} = 5L.
\end{align*}
This gives the upper bound on $\nbar_{i,\ell}$, and the identical argument can be used to upper bound $\mbar_{i,\ell}$. 
\end{proof}

\begin{lemma}\label{lem:fun_smooth_inv_aunbar_bound}
We can bound
\begin{align*}
\aunbar^\ell(i) \le g_i^{-1}(\gamma) + \left ( (L+1) \ell + 6L + \frac{6 g_i^{-1}(\gamma)}{\epssafe} \right ) 2^{-\ell} . 
\end{align*}
\end{lemma}
\begin{proof}
By definition,
\begin{align*}
\aunbar^\ell(i) = \sum_{s = 1}^{\ell} \frac{g_i^{-1}(\min \{ \gamma + 2 \epsilon_s, \gamma + \epssafe \})}{2^{\ell  - s + 1}} + \left ( \ell + \frac{4 g_i^{-1}(\gamma + \epssafe)}{ \epssafe } \right ) 2^{-\ell} .
 \end{align*}
Our primary concern will be with upper bounding the first term. Let $\ell_0$ denote the minimum value of $\ell$ such that $ 2\epsilon_s \le \epssafe$: $\ell_0 = \lceil \log_2 \frac{2}{\epssafe} \rceil$. Then,
\begin{align*}
\sum_{s = 1}^{\ell} \frac{g_i^{-1}(\min \{ \gamma + 2 \epsilon_s, \gamma + \epssafe \})}{2^{\ell  - s + 1}} & = \sum_{s=1}^{\ell_0 - 1} \frac{g_i^{-1}(\gamma + \epssafe)}{2^{\ell - s + 1}} + \sum_{s=\ell_0}^{\ell} \frac{g_i^{-1}(\gamma + 2 \epsilon_s)}{2^{\ell - s +1}} \\
& \le g_i^{-1}(\gamma + \epssafe) 2^{\ell_0 - \ell - 1} + \sum_{s = \ell_0}^\ell \frac{g_i^{-1}(\gamma) + 2 L \epsilon_s}{2^{\ell - s +1}} \\
& \le g_i^{-1}(\gamma + \epssafe) 2^{\ell_0 - \ell - 1} + \sum_{s = 1}^\ell \frac{g_i^{-1}(\gamma) + 2 L \epsilon_s}{2^{\ell - s +1}} \\
& \le g_i^{-1}(\gamma + \epssafe) 2^{\ell_0 - \ell - 1} + g_i^{-1}(\gamma) + L \ell 2^{-\ell}.
\end{align*}
Furthermore, since $\ell_0 = \lceil \log_2 \frac{\epssafe}{2} \rceil$, we have that 
\begin{align*}
 g_i^{-1}(\gamma + \epssafe) 2^{\ell_0 - \ell - 1} \le  g_i^{-1}(\gamma + \epssafe) 2^{\log_2 \frac{2}{\epssafe} - \ell } = \frac{2 g_i^{-1}(\gamma + \epssafe)}{\epssafe} 2^{-\ell}
\end{align*}

Finally, we upper bound $g_{i}^{-1}(\gamma + \epssafe) \le g_{i}^{-1}(\gamma) + L \epssafe$. The conclusion then follows by combining these bounds.
\end{proof}

\begin{lemma}\label{lem:fun_ellbar_bound}
We can bound
$$ \ellbar(i) \le \max \left \{ 3 + \log_2 \frac{2 (L+1)}{\Delta_i} + 2 \log_2 \left ( \log_2 \frac{2(L+1)}{\Delta_i} + 3 \right ) , \log_2 \frac{14L + 12 + 12 g_i^{-1}(\gamma)/\epssafe}{\Delta_i},  \log_2 \frac{8}{\epssafe} + 1 \right \} . $$
\end{lemma}
\begin{proof} 
Recall that
\begin{align*}
\ellbar(i) = \argmin_{\ell \in \bbN} \ell \quad \text{s.t.} \quad f_i(\aunbar^\ell(i)) + 4 \epsilon_\ell \le f_{\ist}(\gtil_{\ist}^{-1}(\gamma - 3 \epsilon_\ell))  .
\end{align*}
We first upper bound this by
\begin{align*}
 \argmin_{\ell \in \bbN, \ell \ge \log_2 \frac{8}{\epssafe}} + 1 \ell \quad \text{s.t.} \quad f_i(\aunbar^\ell(i)) + 4 \epsilon_\ell \le f_{\ist}(g_{\ist}^{-1}(\gamma - 3 \epsilon_\ell))  
\end{align*}
which, by \Cref{lem:ellhatun_upper}, implies that we are in the regime where $\ell \ge \ellhatun(i)$. As we have assumed $g_{\ist}^{-1}$ is $L$-Lipschitz, and since the inverse of a monotonic function is monotonic, we have that
\begin{align*}
| g_{\ist}^{-1}(\gamma) - g_{\ist}^{-1}(\gamma - 3 \epsilon_\ell) | = g_{\ist}^{-1}(\gamma) - g_{\ist}^{-1}(\gamma - 3 \epsilon_\ell) \le 3L \epsilon_\ell \implies g_{\ist}^{-1}(\gamma - 3 \epsilon_\ell) \ge g_{\ist}^{-1}(\gamma) - 3L\epsilon_\ell .
\end{align*}
As $f_{\ist}$ is 1-Lipschitz, this then implies that
\begin{align*}
f_{\ist}(g_{\ist}^{-1}(\gamma)) - f_{\ist}(g_{\ist}^{-1}(\gamma - 3 \epsilon_\ell)) \le g_{\ist}^{-1}(\gamma) - g_{\ist}^{-1}(\gamma - 3 \epsilon_\ell) \le 3L \epsilon_\ell \implies f_{\ist}(g_{\ist}^{-1}(\gamma))  - 3L \epsilon_\ell \le  f_{\ist}(g_{\ist}^{-1}(\gamma - 3 \epsilon_\ell)).
\end{align*}
Similarly, using that $f_i$ is 1-Lipschitz, we can bound (using that, by \Cref{lem:aunhat_bound2}, $\aunhat^{\mhat_{i,\ell},\ell}(i) \le \aunbar^\ell(i)$ and by \Cref{lem:aunhat_unsafe} $\aunhat^{\mhat_{i,\ell},\ell}(i) \ge g_i^{-1}(\gamma)$, which implies that $\aunbar^\ell(i) \ge g_i^{-1}(\gamma)$):
\begin{align*}
f_i(\aunbar^\ell(i)) - f_i(g_i^{-1}(\gamma)) \le \aunbar^\ell(i) - g_i^{-1}(\gamma) \implies f_i(\aunbar^\ell(i)) \le f_i(g_i^{-1}(\gamma)) + | \aunbar^\ell(i) - g_i^{-1}(\gamma)| .
\end{align*}
It follows that we can upper bound
\begin{align*}
\ellbar(i) & \le \argmin_{\ell \in \bbN, \ell \ge \log_2 \frac{8}{\epssafe}} \ell \quad \text{s.t.} \quad f_i(g_i^{-1}(\gamma)) + | \aunbar^\ell(i) - g_i^{-1}(\gamma)| + 4 \epsilon_\ell \le f_{\ist}(g_{\ist}^{-1}(\gamma))  - 3L \epsilon_\ell \\
& = \argmin_{\ell \in \bbN, \ell \ge \log_2 \frac{8}{\epssafe}} \ell \quad \text{s.t.} \quad  | \aunbar^\ell(i) - g_i^{-1}(\gamma)| + (4+3L) \epsilon_\ell \le \Delta_i . 
\end{align*}
By \Cref{lem:fun_smooth_inv_aunbar_bound}, we can bound
\begin{align*}
| \aunbar^\ell(i) - g_i^{-1}(\gamma)| \le \left ( (L+1) \ell + 6L + \frac{6 g_i^{-1}(\gamma)}{\epssafe} \right ) 2^{-\ell} , 
\end{align*}
so it follows that an upper bound on $\ellbar(i)$ is any $\ell$ satisfying
\begin{align*}
2^\ell & \ge \max \left \{ \frac{2(L+1) \ell}{\Delta_i}, 2 \frac{7L + 6 + 6 g_i^{-1}(\gamma)/\epssafe}{\Delta_i},  \log_2 \frac{8}{\epssafe} + 1 \right \} \\
& \ge \max \left \{ \frac{(L+1) \ell + 7L + 6 + 6 g_i^{-1}(\gamma)/\epssafe}{\Delta_i},  \log_2 \frac{8}{\epssafe} + 1\right \} .
\end{align*}
To ensure that $2^\ell \ge \frac{2(L+1) \ell}{\Delta_i} \iff \ell \ge \log_2 \ell + \log_2(2(L+1)/\Delta_i)$ it suffices to take
\begin{align*}
\ell \ge \max \{ \log_2(2(L+1)/\Delta_i), 3 \} + 2 \log_2 ( \max \{ \log_2 (2(L+1)/\Delta_i), 3 \})
\end{align*}
since, in this case,
\begin{align*}
\log_2 \ell + \log_2(2(L+1)/\Delta_i) & \le \log_2 \left ( \max \{ \log_2(2(L+1)/\Delta_i), 3 \} + 2 \log_2 ( \max \{ \log_2 (2(L+1)/\Delta_i), 3 \}) \right ) + \log_2(2(L+1)/\Delta_i) \\
& \le \log_2 \left ( 3 \max \{ \log_2(2(L+1)/\Delta_i), 3 \}  \right ) + \log_2(2(L+1)/\Delta_i) \\
& = \log_2 3 + \log_2 \left (  \max \{ \log_2(2(L+1)/\Delta_i), 3 \}  \right ) + \log_2(2(L+1)/\Delta_i) \\
& \le 2 \log_2 \left (  \max \{ \log_2(2(L+1)/\Delta_i), 3 \}  \right ) + \max \{ \log_2(2(L+1)/\Delta_i), 3 \} \\
& = \ell.
\end{align*}
We conclude that
\begin{align*}
\ellbar(i) \le \max \left \{ 3 + \log_2 \frac{2 (L+1)}{\Delta_i} + 2 \log_2 \left ( \log_2 \frac{2(L+1)}{\Delta_i} + 3 \right ) , \log_2 \frac{14L + 12 + 12 g_i^{-1}(\gamma)/\epssafe}{\Delta_i}, \log_2 \frac{8}{\epssafe} + 1 \right \} .
\end{align*}
\end{proof}

\begin{proof}[Proof of \Cref{cor:mon_simplified_complexity}]
This result follows from \Cref{thm:monotonic_complexity} and the upper bounds given in \Cref{lem:fun_mbar_bound} and \Cref{lem:fun_ellbar_bound}. By \Cref{thm:monotonic_complexity}, we can bound the sample complexity as
\begin{align*}
& \sum_{i = 1}^d \left ( \sum_{\ell = 1}^{\lceil \log_2 \frac{8}{\epssafe} \rceil } ( \mbar_{i,\ell} + \nbar_{i,\ell} + \ell + 2) N_{\ell, \tbar}  + \sum_{\ell = \lceil \log_2 \frac{8}{\epssafe} \rceil + 1}^{\ellbar(i)} ( \nbar_{i,\ell} + \ell + 2) N_{\ell, \tbar} \right ) .
\end{align*}
By \Cref{lem:fun_mbar_bound}, we can upper bound $\nbar_{i,\ell} \le 5L $ for $\ell \ge 2$ and $\nbar_{i,1} \le  2(g_i^{-1}(\gamma - 1/2) - a_{0,i})$ and similarly $\mbar_{i,\ell} \le  5L $ and $\mbar_{i,1} \le 2( g_i^{-1}(\gamma + \epssafe - 1/2) - a_{0,i} )$. By \Cref{lem:fun_ellbar_bound}
we can bound
$$ \ellbar(i) \le \max \left \{ 3 + \log_2 \frac{2 (L+1)}{\Delta_i} + 2 \log_2 \left ( \log_2 \frac{2(L+1)}{\Delta_i} + 3 \right ) , \log_2 \frac{14L + 12 + 12 g_i^{-1}(\gamma)/\epssafe}{\Delta_i}, \log_2 \frac{8}{\epssafe} + 1\right \} . $$
Recall that $N_{\tbar,\ell}  = 2 \log \frac{8 \tbar^2}{\delta} \cdot 2^{2 \ell}$. For $i \neq \ist$, letting $\lesssim$ denote inequality 
\begin{align*}
& \sum_{\ell = 1}^{\lceil \log_2 \frac{8}{\epssafe} \rceil } ( \mbar_{i,\ell} + \nbar_{i,\ell} + \ell ) N_{\ell, \tbar}  + \sum_{\ell = \lceil \log_2 \frac{8}{\epssafe} \rceil + 1}^{\ellbar(i)} ( \nbar_{i,\ell} + \ell) N_{\ell, \tbar}  \\
& \le \Big ( 3 + g_i^{-1}(\gamma + \epssafe - 1/2) + g_i^{-1}(\gamma - 1/2) - 2a_{0,i} \Big ) N_{1,\tbar} \\
& \qquad \qquad + \sum_{\ell = 2}^{\lceil \log_2 \frac{8}{\epssafe} \rceil } ( 2 + 10 L + \ellbar(i) ) N_{\ell, \tbar}  + \sum_{\ell = \lceil \log_2 \frac{8}{\epssafe} \rceil + 1}^{\ellbar(i)} ( 1 + 5L + \ellbar(i) ) N_{\ell, \tbar} \\
& \le 8  \Big ( 3 + g_i^{-1}(\gamma + \epssafe - 1/2) + g_i^{-1}(\gamma - 1/2) - 2a_{0,i} \Big ) \log \frac{8 \tbar^2}{\delta} \\
& \qquad \qquad + 4 (2 + 10L + \ellbar(i))  \log \frac{8 \tbar^2}{\delta} \cdot 2^{2 \ellbar(i)} \\
& \le 8  \Big ( 3 + g_i^{-1}(\gamma + \epssafe - 1/2) + g_i^{-1}(\gamma - 1/2) - 2a_{0,i} \Big ) \log \frac{8 \tbar^2}{\delta} \\
& \qquad \qquad \cOtil \left ( (2 + 10L + \ellbar(i))  \log \frac{8 \tbar^2}{\delta} \cdot \frac{1 + L^2 + g_i^{-1}(\gamma)^2/\epssafe^2}{\Delta_i^2} \right ) 
\end{align*}
where $\cOtil(\cdot)$ hides logarithmic terms and absolute constants. Noting that $\ellbar(i)$ is only logarithmic in problem parameters, we can bound all of this by
\begin{align*}
\cOtil \left ( \Big ( 1 + g_i^{-1}(\gamma + \epssafe - 1/2) + g_i^{-1}(\gamma - 1/2) - 2a_{0,i} \Big ) \log \frac{1}{\delta} + \frac{(1+L) (1 + L^2 + g_i^{-1}(\gamma)^2/\epssafe^2)}{\Delta_i^2} \cdot \log \frac{1}{\delta} \right ) .
\end{align*}
A similar upper bound can be used to bound the contribution of $\ist$. Combining these gives the result. 
\end{proof}

\section{Additional Experiments and details}\label{app:experiment}
In the following, we elaborate some additional details of experiments in \Cref{sec:experiment}. First we note that in our instances, we need to choose the initial points $a_{0,i}$ and boundary values $M_i$. In the linear response model, we choose $a_{0,i} = 0.1$ and $M_i = 1.5$ for all $i$, whereas in the nonlinear drug response model, we set initial safe dosage value $a_{0,i} = -3$ and $M_i = -0.5$ for all $i$. \\

We try to make our implementation of \textsc{SafeOpt} more realistic and efficient. In the linear response model, we choose a linear kernel with a $\beta_t$ value selected from \cite{abbasi2011improved}: $$\beta_t(\delta) = R\sqrt{\log\left(\frac{d(1+\frac{(t-1)L^2}{\lambda})}{\delta}\right)} + \sqrt{\gamma}S$$ with $\gamma=1$ (note that this only scales logarithmically in $d$ since it is assumed the learner knows the coordinates are orthogonal). In the drug response model, we choose a RBF kernel $k_{rbf}(x,x') = \sigma_{rbf}^2\text{exp}\left(-\frac{(x-x')^2}{2l^2}\right)$ with signal variance $\sigma_{rbf}^2 =1$ and length scale $l=1$. 

To simplify the algorithms' executions in the drug response model, we made some minor changes to the algorithms in our experiments. In \algmono we replaced the theoretical expression for $N_{\ell, t}$ from $\lceil 2 \log \frac{8 t^2}{\delta} \cdot \epsilon_\ell^{-2} \rceil$ to $\lceil 2 \log \frac{8^2}{\delta} \cdot \epsilon_\ell^{-2} \rceil$. In \textsc{SafeOpt}, we replaced the theoretical $\beta_t$ value with a fixed constant 3. These modifications slightly reduced the amount of pulls needed. Nonetheless, the empirical performances were satisfactory, as no error in result was found.

\end{document}